\documentclass{article}
\usepackage{PRIMEarxiv}
\usepackage[utf8]{inputenc}
\usepackage[T1]{fontenc}
\usepackage{natbib}
\setcitestyle{authoryear,round,citesep={;},aysep={,},yysep={;}}

\date{}

\usepackage{amsmath,amsfonts,bm}

\def\1{\bm{1}}

\DeclareMathAlphabet{\mathsfit}{\encodingdefault}{\sfdefault}{m}{sl}
\SetMathAlphabet{\mathsfit}{bold}{\encodingdefault}{\sfdefault}{bx}{n}

\def\sP{{\mathbb{P}}}

\DeclareMathOperator*{\argmin}{arg\,min}

\usepackage{subcaption}
\DeclareCaptionLabelFormat{arialpanel}{%
  \includegraphics{figures/panels/labels/#2.pdf}%
}
\usepackage{capt-of}
\usepackage{needspace}
\usepackage{wrapfig}
\usepackage{enumitem}
\usepackage{amsthm,amssymb}
\usepackage{booktabs}
\usepackage[section]{placeins}

\usepackage{amsmath,amssymb,amsthm,array,booktabs,graphicx,microtype,multirow,url,xcolor}
\usepackage{colortbl, tcolorbox}
\usepackage{algorithm,algorithmic}
\usepackage{wrapfig}
\usepackage{tikz}
\usetikzlibrary{arrows.meta,backgrounds,fit,positioning}
\usepackage{caption}
\DeclareCaptionFont{tfmcaption}{\fontsize{9.75pt}{10.75pt}\selectfont}
\usepackage[hidelinks]{hyperref}

\newtheorem{theorem}{Theorem}
\newtheorem{proposition}[theorem]{Proposition}
\newtheorem{lemma}[theorem]{Lemma}

\title{Architecture Alignment With Sparse Priors in Tabular Foundation Models}

\author{%
  Tianqi Zhao\textsuperscript{*}\\
  Renmin University of China
  \And
  Tianyi Zhuang\textsuperscript{*}\\
  National University of Singapore
  \And
  Shuo Duan\textsuperscript{*}\\
  National University of Singapore
  \AND
  Guanyang Wang\\
  Rutgers University
  \And
  Yan Shuo Tan\textsuperscript{\textdagger}\\
  National University of Singapore
  \And
  Qiong Zhang\textsuperscript{\textdagger}\\
  Renmin University of China
}

\definecolor{rowtokenblue}{RGB}{75,99,126}
\definecolor{alternatingteal}{RGB}{0,134,125}
\definecolor{coefficientpurple}{HTML}{6F5A9A}
\definecolor{interceptgray}{HTML}{C5C8CC}
\definecolor{residualgray}{HTML}{7D8289}
\definecolor{scopecontext}{HTML}{0072B2}
\definecolor{scopequery}{HTML}{D55E00}
\DeclareRobustCommand{\contextscopeicon}{%
  \tikz[baseline=-0.5ex]{\draw[scopecontext,line width=1.05pt] (0,0) -- (1.5em,0);}%
}
\DeclareRobustCommand{\queryscopeicon}{%
  \tikz[baseline=-0.5ex]{\draw[scopequery,line width=1.05pt] (0,0) -- (1.5em,0);}%
}
\DeclareRobustCommand{\activesourceicon}{%
  \tikz[baseline=-0.5ex]{\draw[black!65,line width=1.05pt] (0,0) -- (1.5em,0);}%
}
\DeclareRobustCommand{\inactivesourceicon}{%
  \tikz[baseline=-0.5ex]{\draw[black!65,line width=1.05pt,dash pattern=on 3.885pt off 1.68pt] (0,0) -- (1.5em,0);}%
}
\DeclareRobustCommand{\dtenicon}{%
  \tikz[baseline=-0.5ex]{%
    \draw[rowtokenblue,dash pattern=on 1.5pt off 1pt,line width=0.5pt] (0,0.18em) -- (1.2em,0.18em);
    \node[circle,draw=rowtokenblue,fill=white,line width=0.5pt,inner sep=1.1pt] at (0.6em,0.18em) {};
    \draw[alternatingteal,dash pattern=on 1.5pt off 1pt,line width=0.5pt] (0,-0.18em) -- (1.2em,-0.18em);
    \node[circle,draw=alternatingteal,fill=white,line width=0.5pt,inner sep=1.1pt] at (0.6em,-0.18em) {};
  }%
}
\DeclareRobustCommand{\dtwentyicon}{%
  \tikz[baseline=-0.5ex]{%
    \draw[rowtokenblue,solid,line width=0.5pt] (0,0.18em) -- (1.2em,0.18em);
    \node[rectangle,draw=rowtokenblue,fill=rowtokenblue,line width=0.5pt,minimum size=3.2pt,inner sep=0pt] at (0.6em,0.18em) {};
    \draw[alternatingteal,solid,line width=0.5pt] (0,-0.18em) -- (1.2em,-0.18em);
    \node[rectangle,draw=alternatingteal,fill=alternatingteal,line width=0.5pt,minimum size=3.2pt,inner sep=0pt] at (0.6em,-0.18em) {};
  }%
}

\begin{document}

\maketitle
\begingroup
\renewcommand{\thefootnote}{\fnsymbol{footnote}}
\footnotetext[1]{Equal contribution.}
\footnotetext[2]{\raggedright Equal supervision. Correspondence to: Yan Shuo Tan (\href{mailto:yanshuo@nus.edu.sg}{\texttt{yanshuo@nus.edu.sg}}) and Qiong Zhang (\href{mailto:qiong.zhang@ruc.edu.cn}{\texttt{qiong.zhang@ruc.edu.cn}}).}
\endgroup

\begin{abstract}
Tabular foundation models (TFMs) are increasingly popular because they deliver strong predictions on new datasets through in-context learning, without task-specific training or extensive tuning.
Yet released TFMs differ simultaneously in their pretraining priors, architectures, and objectives, obscuring their respective inductive biases.
We therefore examine one concrete capability: irrelevant-feature suppression.
Across synthetic tasks and real-world datasets, adding null features causes substantially greater predictive degradation in the row-token model TabDPT, whereas the cell-token alternating-axis model TabPFN v2 and other TFMs remain comparatively stable. 
This gap motivates us to ask whether architecture contributes to irrelevant-feature suppression. 
Because released TFMs remain confounded by other design choices, we train streamlined row-token and alternating-axis transformers under identical sparse-to-dense linear priors.
Exact Bayes analysis shows that sparse prediction requires context-dependent feature gating, whereas the dense endpoint requires only uniform feature weighting.
Consistent with this distinction, the alternating-axis model is substantially closer to the Bayesian optimal predictor on sparse tasks, while the architecture gap becomes negligible on dense tasks; almost all of the sparse gap arises from linear coefficient-estimation error.
Finally, in both the controlled model and frozen TabPFN~v2, we examine the effect of interventions on the feature-attention outputs on the linear coefficients, finding evidence of task-dependent selective routing of computation through feature-indexed pathways.
Together, these results support architecture--prior alignment: preserving an addressable feature axis provides an inductive bias for task-adaptive relevance inference.
Code is available \href{https://github.com/Tianqi-Zhao/ArchitecturePriorTFMs}{here}.
\end{abstract}

\section{Introduction}
\label{sec:introduction}

Tabular foundation models (TFMs) are pretrained predictors that condition on a labeled table and predict labels for new rows in context without updating their parameters.
On small- and medium-sized tabular tasks, recent TFMs achieve state-of-the-art or competitive accuracy, often without task-specific hyperparameter tuning \citep{hollmann2025tabpfn,erickson2025tabarena}.
By replacing dataset-specific fitting and tuning with in-context adaptation, TFMs greatly reduce deployment costs on new datasets, driving rapid development of increasingly capable systems.

Each TFM couples two central design choices: a distribution of pretraining tasks, often called its \emph{data prior}, and an architecture for learning from the sampled context.
TabPFN v2 marked a major change in both choices \citep{hollmann2025tabpfn}.
Whereas TabPFN v1---and later TabDPT---principally represented entire observations as tokens \citep{mueller2022tabpfn,ma2025tabdpt}, TabPFN v2 adopted a cell-based design with alternating attention along the observation and feature axes.
It paired this alternating-axis architecture with a richer pretraining pipeline and delivered a large improvement in benchmark performance.
Subsequent systems have continued to modify both architecture and the data prior \citep{qu2025tabicl,qu2026tabiclv2,grinsztajn2025tabpfn25,priorlabs2026tabpfn3,google2026tabfm,ma2025tabdpt,hosseinzadeh2026tabdptturbo}.

This joint evolution makes architectural progress difficult to interpret.
First, released models typically change the pretraining prior and architecture together, so performance differences cannot be cleanly attributed to either ingredient.
Second, aggregate benchmark scores reveal which model predicts better on average but not which predictive capabilities or inductive biases have changed.
Recent work has begun to isolate the role of the prior by comparing several priors under a fixed architecture and training protocol \citep{zhang2025mitra,turkmen2026priors}.
By contrast, studies of TFM architecture have primarily analyzed released checkpoints through layerwise probing, representation analysis, attention interventions, and ablations \citep{rezaeibalef2026onelayer,bilos2026mechanistic,gupta2026wherecomputation}.
These studies illuminate how trained TFMs compute, but they do not isolate architecture as an experimental variable.
The resulting open question is not simply which architecture performs best, but whether architecture itself affects a specific predictive capability under a matched task prior.
In this paper, we study that question through the following hypothesis:
\begin{center}
\begin{tcolorbox}[colframe=black!50!white, 
    colback=gray!5!white, 
    boxrule=0.5mm, 
    arc=5pt,width=14cm]
\centering
\textit{Cell-token architectures with alternating feature- and observation-axis attention provide a stronger inductive bias toward irrelevant-feature suppression than row-token architectures.}
\end{tcolorbox}
\end{center}
Two observations motivate this hypothesis.
TabPFN v2 is robust to injected irrelevant features \citep{hu2026noiseimmunity} and can sometimes outperform LASSO on sparse linear prediction tasks \citep{zhang2025onemodel}.
We first broaden this evidence by comparing irrelevant-feature robustness across released TFMs.
Across the tested systems, TabPFN v2+ and TabICL v2 remain close to their clean-task performance as null features are added, whereas TabDPT is substantially more sensitive.
The sensitivity decreases across TabDPT releases v1.0--v1.3 but remains well above that of TabPFN v2 in every release (Section~\ref{sec:part-i-production-models} and App.~\ref{app:part-i-tabdpt-versions}).
Because these systems also differ in their pretraining priors and procedures, this comparison establishes a capability gap but cannot attribute it to architecture.

Motivated by this gap, we isolate architecture in a controlled setting.
We construct a matched family of orthogonal linear-regression priors that varies only the number of active features and admits exact posterior predictive means.
At the dense endpoint, the Bayes optimal predictor reduces to a fixed isotropic linear kernel over context observations.
Under every sparse prior, however, each feature--response statistic is modulated by a context-dependent posterior inclusion probability, producing a diagonal feature gate that depends on the full labeled context.
We show that predictors affine in the context responses attain the dense Bayes optimal predictor but incur a strict approximation gap under sparsity, whereas an idealized alternating-axis network can directly approximate the sparse computation.
This contrast identifies a computational alignment rather than an impossibility result for general deep row-token transformers: preserving a feature axis provides an explicit route for inferring and applying task-specific relevance.

We test the resulting prediction---that the architectural advantage should track the need for context-dependent feature gating---by training streamlined row-token and cell-token alternating-axis transformers under identical priors and training protocols.
Against the exact Bayes targets, the alternating-axis model has substantially lower excess prediction error than the row-token model on sparse tasks, while the architecture gap becomes negligible at the dense endpoint.
An exact functional decomposition attributes almost all of the sparse gap to errors in the estimated linear coefficients rather than to intercept or nonlinear query effects, tying the advantage to context-dependent allocation of predictive mass across coordinates.

Finally, towards explaining this sparse gap, we investigate how the alternating-axis architecture may enable computation that is more suited to sparse priors.
Specifically, we use interventions on the feature-attention outputs, stratified by layer, feature, and row-type, and examine its effect on individual linear coefficients.
Across the controlled model and frozen TabPFN v2, we observe evidence of selective computational routing through the feature-indexed pathways provided by feature-attention, that this routing is task dependent, and that computation is heterogeneous across layers.

Together, the observational, controlled, and interventional evidence supports architecture--prior alignment rather than a universal ranking of tabular transformers.
Preserving an explicit feature axis offers a direct route to task-adaptive coordinate gating when the task determines which features matter, while conferring little advantage when the optimal rule weights all coordinates uniformly.

\section{Released TFMs respond differently to irrelevant features}
\label{sec:part-i-production-models}

We first ask how strongly irrelevant features affect released TFMs.
We distinguish two questions: whether adding irrelevant features reduces predictive performance, and whether the fitted predictor remains functionally sensitive to those features.
Our primary contrast is between the row-token model TabDPT and the classical alternating-axis model TabPFN v2; TabICL v2 and TabPFN v3 broaden the comparison beyond these two models.
We consider models released before August 1, 2026.
The appendix adds TabPFN v2.5 (App.~\ref{app:part-i1-null-dose}) and repeats the analysis across TabDPT releases (App.~\ref{app:part-i-tabdpt-versions}).

\phantomsection
\textbf{Paired design and metrics.}
We evaluate each TFM on paired versions of synthetic and real-world \emph{regression} datasets.
Each pair contains a clean dataset with only the raw features and an augmented dataset with additional features constructed to carry no information about the response; we call these added features \emph{null features}.
We vary the null-to-original feature ratio $\rho$ from zero to seven.
For the synthetic datasets, we add independent null features to randomly generated Friedman datasets~\citep{friedman1991mars}, whose nonlinear response depends on five known relevant features and therefore provides ground-truth relevance.
For the real-world evaluation, we add row-permuted copies of raw columns to 42 OpenML regression datasets~\citep{vanschoren2014openml}, preserving their empirical distributions while breaking their association with the target.
We evaluate TabDPT v1.2 without retrieval~\citep{hosseinzadeh2026tabdptturbo}, TabPFN v2/v3~\citep{hollmann2025tabpfn,priorlabs2026tabpfn3}, and TabICL v2~\citep{qu2026tabiclv2}.
Across both settings, the \emph{$R^2$ drop} relative to the paired clean task measures predictive degradation.
To measure functional dependence, we permute one null query feature while holding the context and all other query features fixed.
We normalize the resulting prediction change by the standard deviation of the query targets and call this measure the \emph{null-feature shift}.
An \emph{ideal relevance-adaptive predictor} should incur little $R^2$ drop when null features are added and little prediction shift when a null coordinate is permuted.
App.~\ref{app:part-i-production-models} gives the data construction, evaluated values of $\rho$, metrics, and inference settings.

\begin{figure}[!tbp]
\centering
    \centering
    \includegraphics[width=0.40\linewidth]{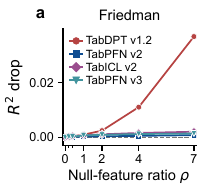}
    \includegraphics[width=0.40\linewidth]{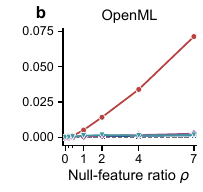}
    \caption{\textbf{Adding null features degrades TabDPT v1.2 more than the other released TFMs}.
    Median query-set $R^2$ drop on (a) synthetic and (b) OpenML datasets as the ratio $\rho$ increases.}
    \label{fig:part-i-performance}
\end{figure}
\phantomsection
\label{sec:part-i1-null-dose}
\textbf{Predictive degradation under null features.}
All models perform similarly on the clean tasks; their robustness separates only as null features are added (App. Fig.~\ref{fig:part-i1-raw-r2}).
At $\rho=7$, TabDPT's median $R^2$ drop is $0.0370$ on Friedman and $0.0713$ on OpenML, compared with $0.0008$ and $0.0014$, respectively, for TabPFN v2.
Thus, TabDPT's median degradation exceeds TabPFN v2's by factors of more than 46 on Friedman and more than 50 on OpenML, while TabICL v2 and TabPFN v3 also remain close to their clean-task performance (Fig.~\ref{fig:part-i-performance}).
A performance gap alone does not show that the fitted predictor depends on the added coordinates.
We therefore turn to the null-feature shift, which directly measures the prediction change caused by perturbing a null query feature.

\begin{figure}[!tbp]
\centering
\centering
\includegraphics[width=0.40\linewidth]{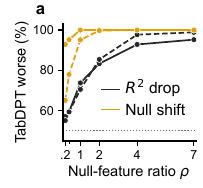}
\includegraphics[width=0.40\linewidth]{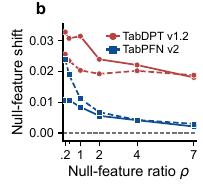}
\caption{\textbf{Null features affect TabDPT more often and more strongly than TabPFN.}
(a) Fraction of paired tasks on which TabDPT v1.2 has a larger $R^2$ drop or null-feature shift than TabPFN v2.
(b) Median null-feature shift for both models across values of $\rho$.
Solid: OpenML; dashed: Friedman.}
\label{fig:part-i-functional}
\end{figure}
\phantomsection
\label{sec:part-i2-functional-sensitivity}
\textbf{Functional dependence on null features.}
The paired comparison shows that the performance gap is widespread rather than driven by a few tasks.
At $\rho=7$, TabDPT's $R^2$ drop exceeds TabPFN v2's on $98.9\%$ of the 350 Friedman task settings and $95.2\%$ of the 42 OpenML datasets (Fig.~\ref{fig:part-i-functional}(a)).
Both models respond much more strongly to known relevant Friedman features or raw OpenML covariates than to the added nulls, providing positive controls for the sensitivity measure (App. Fig.~\ref{fig:part-i2-permutation}).
Nevertheless, TabDPT is consistently more sensitive to the null coordinates (Fig.~\ref{fig:part-i-functional}(a,~b)).
At $\rho=7$ on OpenML, its median null-feature shift is more than 8.5 times that of TabPFN v2, and it exceeds TabPFN v2's on all 42 datasets.
An additional partial-dependence analysis likewise shows that TabDPT depends more strongly on the added null features than TabPFN v2 (App. Fig.~\ref{fig:part-i2-pdp}).
Together, the predictive and functional results show that this difference is large and widespread.

\section{Isolating the architectural contrast with controlled priors}
\label{sec:controlled-priors}

The preceding released-model comparison shows that TabDPT is much more sensitive to irrelevant features than TabPFN v2, but differences in their pretraining data and procedures prevent us from attributing this gap to architecture.
To isolate the architectural contrast, we train streamlined row-token and alternating-axis models under \emph{identical} task priors and training protocols and evaluate them against exact Bayes targets.

\subsection{Sparse and dense priors require different computations}
To make the sparse--dense comparison precise, we seek a task family spanning sparse settings and a dense endpoint for which the Bayes-optimal predictor is available in closed form.
We construct such a family by varying the number of active features in an orthogonal fixed design regression model.

\textbf{Matched priors.}
Fix $n\ge d\ge2$, $v^2,\sigma^2>0$, and consider fixed design regression tasks of the form
\begin{equation}
    y=X\beta+\varepsilon,
    \qquad \varepsilon\sim\mathcal{N}(0,\sigma^2I_n),
    \qquad X^\top X=nI_d.
    \label{eq:orthogonal-context}
\end{equation}
Write $[d]=\{1,\ldots,d\}$.
For $k \in [d]$, let $\mathcal{S}_k=\{A\subseteq[d]:|A|=k\}$ and define a prior $P_k$ by drawing
\begin{equation}
    S\sim\operatorname{Unif}(\mathcal{S}_k),
    \qquad \beta_S| S\sim\mathcal{N}(0,(v^2/k)I_k),
    \qquad \beta_{S^c}=0.
    \label{eq:sparse-prior}
\end{equation}
Here $k$ is the number of active features.
Throughout the controlled analysis, \emph{active} and \emph{inactive} refer specifically to membership in the realized support $S$, whereas \emph{relevant} and \emph{irrelevant} are reserved for broader statements about predictive dependence.
The prior is sparse when $k<d$ and dense when $k=d$.
The scaling by $1/k$ gives $\mathbb{E}[\beta]=0$, $\mathbb{E}\|\beta\|_2^2=v^2$, and $\mathbb{E}[\beta\beta^\top]=(v^2/d)I_d$ for every $k$.
Thus, changing $k$ changes the support structure but not the coefficient covariance or expected signal energy.
We therefore call $\{P_k\}$ a matched prior family.
Together with~\eqref{eq:orthogonal-context}, each $P_k$ specifies a hierarchical Bayesian linear regression model.

\textbf{Posterior predictive means.}
Given observed data $D=(X,y)$ from~\eqref{eq:orthogonal-context}, consider prediction at a fixed covariate vector $x\in\mathbb{R}^d$.
For $y_q=x^\top\beta+\varepsilon_q$, where $\varepsilon_q\sim\mathcal N(0,\sigma^2)$ is independent of the observed data, the Bayes-optimal predictor under squared loss is the posterior predictive mean $f_k^*(D,x)=\mathbb{E}[y_q| D,x]=x^\top\mathbb{E}[\beta| D]$.
Orthogonality makes $c=c(D):=X^\top y/n$ sufficient for $\beta$ and yields the closed-form posterior coefficient mean
\begin{equation}
   \eta_k^*(c):=\mathbb{E}[\beta| D]=\mathbb{E}[\beta| c]
   = \lambda_kq_k(c)\odot c,
    \label{eq:adaptive-bayes-kernel}
\end{equation}
where $q_{k,j}(c)=\sP(j\in S| c)$ is the $j$th element of the posterior inclusion vector $q_k(c)$, $\lambda_k=nv^2/(nv^2+k\sigma^2)$ is a shrinkage factor, and $\odot$ denotes coordinatewise multiplication.
See App.~\ref{app:bayes-posterior} for details.

\textbf{Computational difference between dense and sparse priors.}
The posterior predictive mean in~(\ref{eq:adaptive-bayes-kernel}) exposes a computational difference between the dense and sparse priors.
\begin{itemize}[leftmargin=*,topsep=0pt,parsep=0pt,itemsep=0pt]
    \item Under the dense prior $P_d$, $q_d(c)=\mathbf 1_d$ and $\eta_d^*(c)=\lambda_dc$, so the Bayes rule applies the same shrinkage factor to every feature.
    Its prediction $f_d^*(D,x)$ can be implemented as the linear-kernel regression $(\lambda_d/n)\sum_i\langle x,x_i\rangle y_i$.
    \item Under a sparse prior $P_k$ with $k<d$, the posterior predictive mean instead uses the context-dependent feature gate $q_k(c)$.
    Although $c=X^\top y/n$ is linear in $y$, the exponential normalization in~\eqref{eq:support-posterior}--\eqref{eq:posterior-inclusion} makes $q_k(c)$ nonlinear in $c$.
\end{itemize}

\subsection{Attention architectures and alignment with posterior mean computation}
\label{subsec:architectures}

To test whether keeping features separate provides a more direct route to the feature weights required by sparse prediction, we compare controlled row-token and alternating-axis models.
Fig.~\ref{fig:arch-and-construction}(a,b) summarizes the streamlined row-token and cell-token alternating-axis architectures used in the controlled comparison.
\begin{figure}[!htbp]
    \centering
    \includegraphics[width=\linewidth]{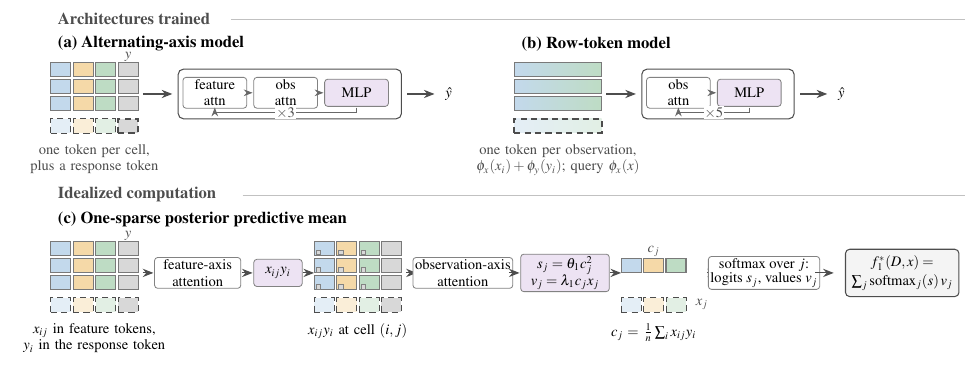}
    \caption{\textbf{Architectures and an idealized one-sparse Bayes posterior predictive mean.}
    \textbf{(a--b)} The alternating-axis and row-token models used in the controlled comparison.
    \textbf{(c)} For $k=1$, response routing and within-feature aggregation form $c_j$, and feature-axis softmax combines $\lambda_1c_jx_j$ using logits proportional to $c_j^2$.
    Panel (c) is a mathematical construction, not an identified circuit in the trained models (App.~\ref{app:feature-construction}).
    Colors denote features; dashed and heavy outlines mark query and readout tokens, respectively; shading denotes pointwise maps.}
    \label{fig:arch-and-construction}
    \vspace{-\baselineskip}
\end{figure}
Details about the architectures are given in App.~\ref{app:network-architecture}.
Both models use the same embedding width and similar parameter counts, providing representative implementations of the two attention layouts at a comparable model scale.
Differences in depth, token count, and computational cost remain inherent to the comparison.
App.~\ref{app:iv-protocol} gives the full configurations and parameter counts.

\textbf{Why row-token architecture can handle dense priors but may struggle with sparse ones.}
We show below that any fixed-kernel regression cannot express the posterior predictive mean under a sparse prior.
\begin{theorem}[Response-affine barrier]
\label{thm:fixed-kernel-gap}
Fix $X^\top X=nI_d$ and a nonzero prediction point $x$. Let
\begin{equation}
    \mathcal F_{\mathrm{aff}}(x)
    =\{f(D,x)=a(X,x)+b(X,x)^\top y:
       a(X,x)\in\mathbb R,\ b(X,x)\in\mathbb R^n\},
    \label{eq:fixed-kernel-class}
\end{equation}
where $a$ and $b$ may depend arbitrarily on $X$, $x$, and known prior
parameters, but not on $y$. Under every $P_k$, the unique minimizer of squared
prediction risk in $\mathcal F_{\mathrm{aff}}(x)$ is
$f_{\mathrm{aff}}^*(D,x)=\lambda_dx^\top c$. It equals $f_d^*(D,x)$, whereas
for every $k<d$,
\begin{equation}
    \inf_{f\in\mathcal F_{\mathrm{aff}}(x)}
    \mathbb E_D\!\left[\{f(D,x)-f_k^*(D,x)\}^2\right]>0.
    \label{eq:fixed-kernel-gap}
\end{equation}
\end{theorem}
See proof details in App.~\ref{app:fixed-kernel-gap}.
The theorem shows that when the number of active features is strictly smaller than the feature dimension, no function in the affine class can approximate $f_k^*(D,x)$ arbitrarily well.
To connect this barrier to row-token architectures, note that linear-kernel regression corresponds to an idealized one-layer row-token attention computation without softmax~\citep{vonoswald2023icl}. 
Since linear-kernel regression is an instance of the fixed-kernel class covered by Thm.~\ref{thm:fixed-kernel-gap}, this correspondence suggests that the same barrier applies to that idealized row-token model. Although nonlinearities and multiple layers can improve expressivity, the expressivity of linear-kernel regression remains a useful baseline heuristic for row-token architectures.

\phantomsection
\label{sec:feature-indexed-construction}
\textbf{Why alternating-axis architecture is more aligned with sparse priors.}
For the sparse prior with $k=1$, an alternating-axis circuit implements the Bayes predictor as shown in Fig.~\ref{fig:arch-and-construction}(c).

\textit{Step 1.} Response routing makes each context response $y_i$ available to each feature. 

\textit{Step 2.} Observation-axis aggregation forms $c_j$. 

\textit{Step 3.} Feature-axis softmax combines values $\lambda_1c_jx_j$ using logits proportional to $c_j^2$.
Because feed-forward networks can approximate the required continuous pointwise maps on compact domains, an idealized alternating-axis network can approximate this one-sparse construction arbitrarily well (Prop.~\ref{prop:one-sparse-alternating-realization}, App.~\ref{app:feature-construction}).
The appendix extends the construction to every $k$ using three routing and aggregation stages and $O(k)$ pooled statistics.



\subsection{Controlled experiments to test the architectural contrast}
\label{sec:controlled-experiments}
\textbf{Setup.}
For each $(d,k)$ pair, with $d\in\{10,20\}$ and $k\in\{1,d/2,d\}$, we train separate row-token and alternating-axis models on tasks drawn from $P_k$.
We set $v^2=1$ and $\sigma=1.5$.
This choice keeps the prior-averaged signal-to-noise ratio fixed across values of $k$ and $d$.
We evaluate both architectures on the same $T=32{,}000$ tasks per condition.
Each task has an orthogonal context with 50--100 observations and independent queries $x\sim\mathcal N(0,I_d)$.
See App.~\ref{app:iv-protocol} for pretraining details.

\begin{figure}[!tbp]
\centering
    \centering
    \includegraphics[width=.45\linewidth]{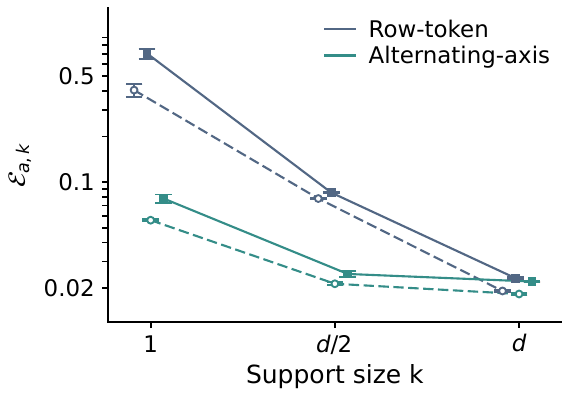}
    \caption{\textbf{The sparse advantage contracts at the dense endpoint.}
    Task-normalized Bayes approximation error (log scale) for the row-token and alternating-axis models.
    \dtenicon\ denotes $d=10$; \dtwentyicon\ denotes $d=20$.
    Markers and error bars show the mean $\pm$ SD across three training seeds, each evaluated on 32,000 shared test tasks.}
    \label{fig:iv-theory-performance}
\end{figure}
\textbf{Error relative to Bayes target.}
\label{sec:iv1-architecture-prior}
Let $\widehat f_{a,k}$ denote the predictor with architecture $a$ trained under $P_k$.
We now describe how to define an aggregate measure of excess prediction error of $\widehat f_{a,k}$ relative to the posterior predictive mean $f_k^*$.
For each test context $D$, we evaluate its mean squared error to $f_k^*$ (where the mean is taken with respect to an independent query $x\sim\mathcal N(0,I_d)$), normalized by the energy of $f_k^*$, and then take a further expectation over contexts:
\begin{equation}
\label{eq:iv-relative-error}
\mathcal E_{a,k}
:=\mathbb E_D\!\left[ \frac{\mathbb E_x\!\left[\|\widehat f_{a,k}(D,x)-f_k^*(D,x)\|^2\right]}{\mathbb E_x\!\left[\|f_k^*(D,x)\|^2\right]} \right]
.
\end{equation}
App.~\ref{app:iv-metrics} gives the empirical estimator of this quantity.

The empirical error values, averaged over 32,000 test tasks, are shown in Fig.~\ref{fig:iv-theory-performance}.
We see that the alternating-axis model has substantially lower excess error under sparse priors, with a gap of $\mathcal E_{\operatorname{row},k} -  \mathcal E_{\operatorname{alternating},k} \approx 0.34$ (respectively $\approx 0.627)$ at $k=1$ and $d=10$ (respectively $d=20$).
The gap is positive for all three training seeds, and also holds conditioned on $D$ across nearly all test tasks (App.~\ref{app:iv-performance}).
However, for dense priors, this gap becomes essentially negligible.

\textbf{Error gap is due to linear coefficient estimation.}
\label{sec:iv2-functional-accounting}
Since $f_k^*(D, x)$ is a linear function of $x$, it is insightful to decompose the difference $\widehat f_{a,k}(D,x)-f_k^*(D,x)$ into a linear term and a nonlinear residual.
This allows us to examine how much inaccuracy is due to failing to estimate the coefficients correctly, and how much is due to lying outside the linear model class.
Aggregating this analysis across architectures then gives further insight into what contributes to their performance gap over sparse priors.
We first state a general decomposition.

\begin{lemma}[Prediction error decomposition]
\label{lem:coefficient-accounting}
Let $f$ be square-integrable, and let $x\sim\mathcal N(0,I_d)$ be independent of $D$.
For any context $D$, define the best linear predictor coefficients
\begin{equation}
\label{eq:blp_coefs}
    (\zeta_f(D),\eta_f(D))
    :=\argmin_{\zeta\in\mathbb{R},\,\eta\in\mathbb{R}^d}
    \mathbb{E}_x\!\left[(f(D,x)-\zeta-x^\top\eta)^2\right],
\end{equation}
and the nonlinear residual $r_f(D,x)=f(D,x)-\zeta_f(D)-x^\top\eta_f(D)$.
Then we have
\begin{equation}
\label{eq:bayes-error-decomposition}
\mathbb E_x[(f(D,x)-f_k^*(D,x))^2] =
        \underbrace{\|\eta_f(D)-\eta_k^*(c)\|_2^2}_{\text{coefficient error}}
        +\underbrace{\zeta_f(D)^2}_{\text{intercept error}}
        +\underbrace{\mathbb E_x[r_f(D,x)^2]}_{\text{nonlinearity error}}.
\end{equation}
\end{lemma}
The proof and its connection to excess prediction risk appear in App.~\ref{app:coefficient-accounting}.
We estimate the three terms for each context and normalize them as in~\eqref{eq:iv-relative-error}.
App.~\ref{app:iv-projection} gives the estimation details.

\begin{figure}[!htbp]
\centering
\includegraphics[width=0.25\linewidth]{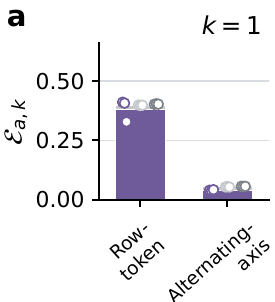}
\includegraphics[width=0.25\linewidth]{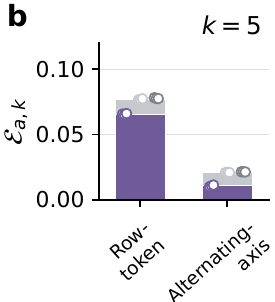}
\includegraphics[width=0.25\linewidth]{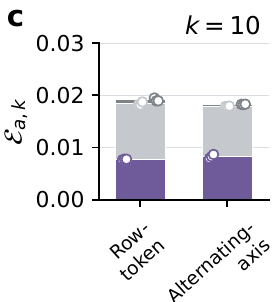}
\vspace{-0.3\baselineskip}
\caption{\textbf{Task-normalized error components at $\boldsymbol{d=10}$.}
\textcolor{coefficientpurple}{\rule{0.8em}{0.8em}} coefficient;
\textcolor{interceptgray}{\rule{0.8em}{0.8em}} squared intercept;
\textcolor{residualgray}{\rule{0.8em}{0.8em}} residual.
Bars stack these components for one architecture.
Each component is averaged equally over 32,000 shared test datasets per seed.
Bars show three-seed means, and points show seed averages.
The total bar height equals the sum of the three components.}
\label{fig:iv-affine-levels-d10}
\end{figure}

\textbf{Sparse priors.}
At $d=10$, the between-architecture coefficient-error difference accounts for 98.4\% of the mean prediction-error gap at $k=1$ and 97.5\% at $k=5$ (Fig.~\ref{fig:iv-affine-levels-d10}; App. Fig.~\ref{fig:app-iv-affine-accounting}).
App.~\ref{app:iv-support} further separates coefficient error over active and inactive features.
The sparse architecture gap therefore comes almost entirely from the row-token architecture's much larger coefficient error.


\textbf{Dense prior.}
At $k=d$, the Bayes coefficient map is the uniform rule $\eta_d^*(c)=\lambda_dc$.
The row-token model has slightly lower mean coefficient error than the alternating-axis model at both $d=10$ and $d=20$.
Its larger intercept and residual errors offset this coefficient advantage.
The alternating-axis model therefore retains a small overall prediction advantage (App.~\ref{app:iv-component-results}).

In summary, the alternating-axis advantage under sparse priors lies primarily in estimating the context-dependent coefficients, rather than in representing the prediction’s linear dependence on the query. At the dense endpoint, where the Bayesian coefficient map reduces to uniform shrinkage, the coefficient advantage disappears. We next try to understand why there is such an accuracy gap in coefficient estimation.

\section{Tracing attention-mediated computation of linear coefficients}
\label{sec:cross-scale-gating}
Section~\ref{sec:controlled-experiments} shows that alternating-axis models approximate the sparse Bayesian predictor more accurately, primarily through better estimation of its context-dependent coefficients.
We now investigate how feature attention contributes to this computation.
We hypothesize that \emph{feature-axis attention supports task-dependent, selectively routed computation of linear coefficients}: attenuating messages associated with a feature should preferentially affect its own coefficient, with stronger effects for active features.
We further examine where this control is concentrated across layers and between context and query messages.
We test this in the simplified alternating-axis model and then apply the same analysis to pretrained TabPFN~v2.

\subsection{Feature-message interventions and coefficient responses}
To examine the hypothesis, we define interventions on internal quantities of the computation graph and observe their effect on the fitted linear coefficients, as defined in \eqref{eq:blp_coefs}.
We perform this analysis for both the simplified alternating-axis architecture from Section~\ref{subsec:architectures} as well as frozen TabPFN~v2.

\textbf{Intervention.}
For one feature-attention head within a row, let $\alpha_{pg}$ be the attention weight from source token $g$ to receiver $p$, and let $v_g$ be its value vector.
Attenuating source feature $j$ replaces the weighted sum by
\[
 \sum_g\alpha_{pg}v_g
 \quad\longmapsto\quad
 \sum_{g\ne j}\alpha_{pg}v_g+(1-\delta)\alpha_{pj}v_j.
\]
We apply this change to every receiver and head at one selected layer $\ell$, using $\delta=0.1$ unless stated otherwise.
The \emph{context-only} intervention changes context rows only; the \emph{query-only} intervention changes query rows only.
Attention weights are unchanged and not renormalized, residual connections remain intact, and downstream computation is rerun.
App.~\ref{app:cross-scale-protocol} gives the equivalent implementation after the output projection.
The operation is illustrated in Fig.~\ref{fig:message-intervention}.

\begin{figure}[!htb]
\centering
\includegraphics[width=\linewidth]{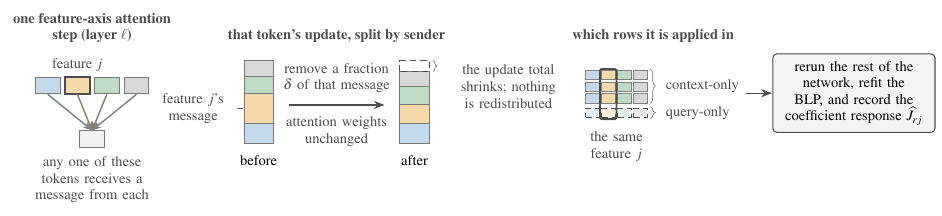}
\caption{\textbf{Feature-message attenuation in the simplified alternating-axis model.}
Source feature $j$'s contribution is scaled by $1-\delta$ at every receiver
and head, either in context rows or in query rows.
Other sources' contributions at that operation and the residual pathway are unchanged.
Bar sizes are schematic; App.~\ref{app:cross-scale-protocol} gives implementation details.}
\label{fig:message-intervention}
\end{figure}


\textbf{Coefficient response.}
Because the fitted model is almost perfectly linear, as shown in Fig.~\ref{fig:iv-affine-levels-d10}, 
we investigate the effect of the intervention on the estimated linear coefficients.
Let $a\in\{\mathrm{context},\mathrm{query}\}$ denote the intervened row type.
Let $\widehat\eta_{0,r}(D)$ and $\widehat\eta_{\delta,\ell, j,r}^{\,a}(D)$ be the baseline and intervened coefficients for query feature $r$, estimated on the same queries, and
define the coefficient response per unit attenuation by
\begin{equation}
 \widehat J_{rj}^{(\ell,a)}
 =\{\widehat\eta_{0,r}(D)-\widehat\eta_{\delta,\ell, j,r}^{\,a}(D)\}/\delta.
 \label{eq:nano-route-matrix-estimate}
\end{equation}
We call $j$ the source feature index and $r$ the target feature index.


\textbf{An illustrative coefficient-response matrix.}
Fig.~\ref{fig:message-response-matrices} shows query-only interventions for one task, selected by proximity to the median Bayes coefficient norm.
Column $j$ corresponds to the attenuated feature and row $r$ to the affected coefficient: diagonal entries indicate own-feature effects, while off-diagonal entries indicate effects on other features.
In this example, the largest final-layer response lies on the active source's own coefficient.
The separate color scales show localization within each layer, not relative response magnitudes across layers.
Context-only matrices are reported in App.~\ref{app:cross-scale-reconstruction}.
\begin{figure}[!ht]
\centering
\includegraphics[width=0.32\linewidth]{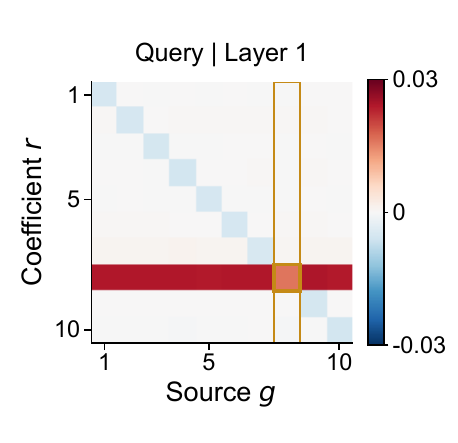}
\includegraphics[width=0.32\linewidth]{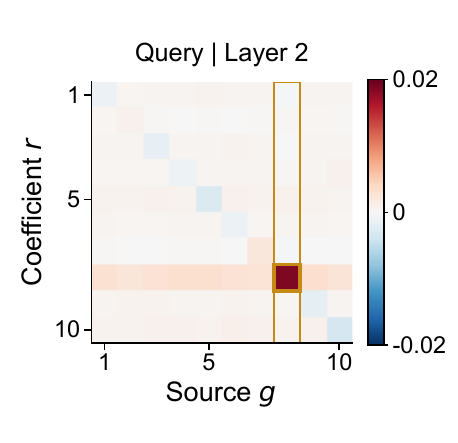}
\includegraphics[width=0.32\linewidth]{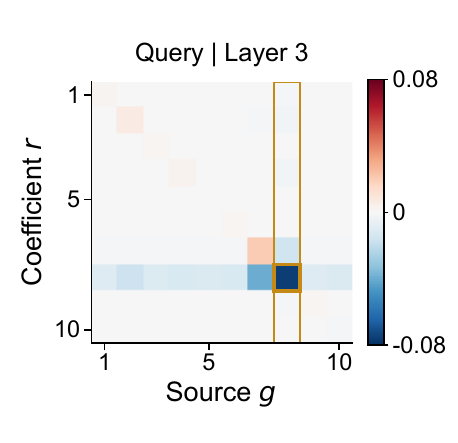}
\vspace{-0.5\baselineskip}
\caption{\textbf{Coefficient responses to query-message attenuation.}
Alternating-axis model, $d=10$, $k=1$, $\delta=0.1$.
The task is nearest the median Bayes coefficient norm among the 128 tasks.
Each matrix column attenuates one source on query rows; matrix rows index affected coefficients.
Color shows signed $\widehat J_{rj}^{(\ell)}$ with a symmetric scale per layer.
Gold outlines mark the active source and its own coefficient.}
\vspace{-\baselineskip}
\label{fig:message-response-matrices}
\end{figure}


\subsection{Evaluating and interpreting selective routing}

To determine whether these patterns persist across tasks, we quantify where the coefficient response occurs and how strong it is, computing each metric separately for context and query interventions and suppressing $a$ below:
\begin{enumerate}[leftmargin=*,topsep=0pt,parsep=0pt,itemsep=0pt]
    \item \textbf{Specificity}, $p_j^{(\ell)}=|\widehat J_{jj}^{(\ell)}|/\sum_r|\widehat J_{rj}^{(\ell)}|$, is the fraction of the total absolute coefficient response occurring on feature $j$ itself. A high value indicates that attenuating source $j$ mainly changes the prediction's linear dependence on that same feature, rather than other features.
    \item \textbf{Sensitivity}, $w_j^{(\ell)}=|\widehat J_{jj}^{(\ell)}|/|c_j|$, is the magnitude of the own-feature coefficient response per unit attenuation, normalized by the observed feature--response statistic $|c_j|$. A high value indicates stronger control over that coefficient relative to this statistic.
\end{enumerate}
For $d\in\{10,20\}$, we compute specificity and sensitivity values on a collection same 128 datasets with sparsity $k=1$.
We summarize effects separately over active and inactive features.

App.~\ref{app:cross-scale-protocol} gives the estimators, denominator handling, and aggregation.

\textbf{Extension to TabPFN v2.}
Each source token in TabPFN~v2 represents a native group of two features.
We therefore attenuate a source group and aggregate its own-feature responses over both constituent coefficients.
A group is active if it contains an active coordinate.
The grouped definitions are given in App.~\ref{app:cross-scale-protocol}; below, a source denotes an individual feature in the simplified model or a feature group in TabPFN~v2.

\begin{figure}[!ht]
 \centering
 \input{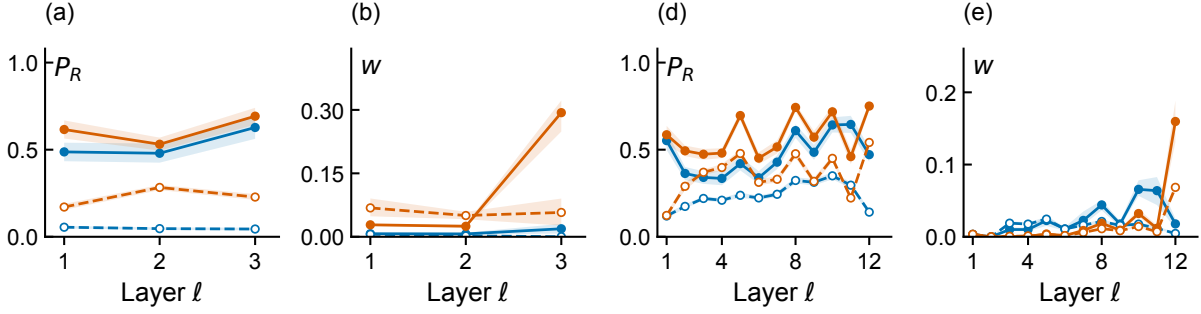}
 \caption{\textbf{Support-selective message effects in context and query rows.}
 Panels (a,b): alternating-axis model; (d,e): frozen TabPFN~v2.
 $d=10$, $k=1$, $\delta=0.1$; the same 128 tasks are used throughout.
 \contextscopeicon~denotes context-row interventions; \queryscopeicon~denotes query-row interventions.
 \activesourceicon\ and \inactivesourceicon\ denote active and inactive sources, respectively.
 For each model, panels show mean specificity $P_R$ and median sensitivity $w$, after averaging sources within each role and task.
 Prediction-error panels (c,f) appear in Fig.~\ref{fig:scope-prediction-effects-d10}.
 Bands are pointwise 95\% intervals from 2,000 generator-batch bootstrap resamples.}
 \label{fig:cross-scale-signed}
 \label{fig:cross-scale-allocation}
 \label{fig:cross-scale-v2}
 \label{fig:message-specificity}
\end{figure}

The results in Fig.~\ref{fig:cross-scale-signed} show similar patterns in the simplified alternating-axis model and TabPFN~v2:
\begin{itemize}[leftmargin=*,topsep=0pt,parsep=0pt,itemsep=0pt]
    \item \textbf{Selective computational routing.} Active sources have high specificity, particularly for final-layer query messages (panels~\subref{fig:section4_scope_effects:a},~\subref{fig:section4_scope_effects:d}): attenuating a feature's messages predominantly changes its own coefficient. This supports selective computational routing through feature-indexed pathways.
    \item \textbf{Routing is task dependent.} Active sources have higher specificity than inactive sources at every tested layer under both interventions. Their final-layer query sensitivities are also $5.1\times$ and $2.3\times$ the inactive-source values in the simplified model and TabPFN~v2, respectively (panels~\subref{fig:section4_scope_effects:b},~\subref{fig:section4_scope_effects:e}); both contrasts persist at $d=20$ (App.~\ref{app:cross-scale-row-scope}).
    \item \textbf{Computation is heterogeneous across layers.} Own-feature sensitivity peaks in final-layer query messages; for active sources, the query median is $15.5\times$ and $8.9\times$ the context median in the two models, respectively. Earlier messages exert weaker own-feature control, although this does not imply that those layers are dispensable.
\end{itemize}

The signs of $\widehat J_{jj}$ and $c_j$ do not always agree (App.~\ref{app:cross-scale-direction}), suggesting that the computation is more complex than simply transmitting a positively weighted feature contribution.
Coefficient changes nevertheless closely reconstruct the centered prediction response for final-layer active query interventions (median fidelity $0.976$ and $0.969$; App.~\ref{app:cross-scale-fidelity}).
Predictive consequences and robustness across attenuation strengths and support sizes are reported in Apps.~\ref{app:cross-scale-row-scope}--\ref{app:cross-scale-support-size}.
Selective computational routing is a candidate mechanism underlying the alternating-axis advantage, but these experiments do not establish where relevance is inferred or whether reduced inactive-source influence reflects suppression within attention itself.


\section{Conclusion}
\label{sec:conclusion}

Irrelevant features expose differences among tabular foundation models that are less apparent on clean tasks. Controlled pretraining under matched sparse and dense linear priors shows that architecture contributes to this contrast: the alternating-axis model more accurately approximates the sparse Bayesian predictor, while its advantage contracts sharply at the dense endpoint. Almost all of the sparse prediction gap is attributable to coefficient error rather than intercept or nonlinear query effects. The analytical targets motivate this prior-dependent advantage: sparse prediction requires context-dependent feature weighting, whereas dense prediction requires only uniform shrinkage.

Feature-message interventions provide evidence of selective computational routing in both the simplified alternating-axis model and pretrained TabPFN v2. Messages associated with active features exert more feature-specific coefficient control, with strong own-feature effects concentrated in final-layer query messages. These findings identify pathways that regulate predictive dependence on individual features, without establishing where relevance is inferred or fully explaining the architecture gap. Together, the results support studying architectural inductive biases through controlled task families and explicit computational targets, rather than seeking a universal ranking of tabular architectures.

\section*{Acknowledgments}
Tianqi Zhao and Qiong Zhang are supported by the National Key R\&D Program of China Grant 2024YFA1015800.
Y.S. Tan was supported by NUS Startup Grant A-8000448-00-00 and the Singapore Ministry of Education (MOE) AcRF Tier 1 Grants A-8002498-00-00 and A-8004458-00-00.




\bibliography{references}
\bibliographystyle{references}
\appendix
\renewcommand{\floatpagefraction}{0.75}
\renewcommand{\bottomfraction}{0.9}
\setcounter{topnumber}{4}
\setcounter{bottomnumber}{3}
\setcounter{totalnumber}{6}
\raggedbottom
\makeatletter
\setlength{\@fptop}{0pt}
\setlength{\@fpsep}{18pt plus 2pt minus 2pt}
\makeatother
\section{Related work}
\label{sec:related-work}

\textbf{Tabular foundation models and architectural axes.}
Prior-data fitted networks learn to approximate Bayesian prediction under a
pretraining distribution and apply the learned procedure to new datasets in
context \citep{mueller2022tabpfn,hollmann2025tabpfn}. Modern TFMs differ not
only in scale and pretraining data but also in how they represent a table.
TabPFN v2 combines attention over observations with attention over randomized
feature groups \citep{hollmann2025tabpfn}; TabICL and TabICL v2 use dedicated
column embedding and compression stages before in-context processing
\citep{qu2025tabicl,qu2026tabiclv2}. TabDPT emphasizes row tokens, learned
dataset preprocessing, real-data pretraining, and retrieval for scaling the
context \citep{ma2025tabdpt}. These systems confound architecture with prior,
scale, and preprocessing, motivating our paired production stress test followed
by common-prior training of streamlined architectures.

Feature processing also determines how models scale to wide tables.
TabPFN-Wide uses continued prior-informed pretraining to extend TabPFN to
extreme feature counts and reports improved robustness to feature noise
\citep{kolberg2026tabpfnwide}. This demonstrates that robustness depends on
pretraining exposure as well as architecture. Our objective is complementary:
we hold the prior fixed when comparing architectures, and then vary only the
support structure of that common prior.

\textbf{In-context learning as kernel aggregation or learned optimization.}
Theory for transformer in-context learning has connected attention to linear
regression and gradient-based learning algorithms
\citep{vonoswald2023icl}. For nonlinear targets, transformers can implement
functional gradient descent, yielding similarity-weighted or kernel-like
updates over context examples \citep{cheng2024functional}. Such accounts make
row attention a natural mechanism for aggregating labels from similar context
observations. They do not by themselves explain how the similarity should
adapt when most coordinates are irrelevant. In our controlled family, the
dense Bayes rule is exactly a fixed row kernel, while the sparse Bayes rule
requires its coordinate metric to depend on support evidence from the current
task. This places feature suppression inside the kernel rather than treating a
generic kernel interpretation as either sufficient or deficient.

\textbf{Robustness to noisy and irrelevant features.}
The closest empirical study evaluates TabPFN under injected random and
correlated features, label noise, and varying sample size, finding stable
accuracy and increasingly concentrated feature attention
\citep{hu2026noiseimmunity}. Our production experiments broaden the comparison
across five released TFMs and use paired regression tasks, while the controlled
experiments separate architecture from pretraining and evaluate against known
Bayes functions. Most importantly, attention concentration is not treated as
mechanistic evidence: our source-route intervention tests whether a
route has a feature-aligned effect on the fitted predictor. Work on feature
shift benchmarks studies robustness to changing feature spaces more generally
\citep{cheng2025tabfsbench}; we focus specifically on task-irrelevant
coordinates whose addition should leave the oracle prediction rule unchanged.

\textbf{Mechanistic interpretation of TFMs.}
Recent studies show that similar benchmark accuracy can conceal different
in-context algorithms. \citet{bilos2026mechanistic} find evidence for
attention-weighted row voting in TabPFN v2 and a prototype-like representation
in TabICL v2. Layerwise probing and structural interventions further suggest
that TFM inference proceeds through iterative refinement, with substantial but
non-interchangeable redundancy across depth
\citep{rezaeibalef2026onelayer}. At a finer scale,
\citet{gupta2026wherecomputation} use activation patching, head ablation, and
attention entropy to localize task-dependent computation in TabPFN v2.5.
These works ask where predictions or readout algorithms emerge. We instead
start from a statistically defined computation---context-dependent coordinate
gating---and measure how interventions change the fitted coefficient map.
Because later representations may repair or redistribute an intervention, we
use small dose responses and propagate every perturbation through the complete
remaining network rather than equating a large one-shot ablation with a
specific mechanism.

\textbf{Feature attribution.}
Attribution and mechanism are related but distinct. ExplainerPFN predicts
Shapley-style feature importance from the data distribution without access to
the target model \citep{fonseca2026explainerpfn}, while conventional post-hoc
methods explain a fixed model through perturbations, gradients, or surrogate
functions. Raw attention weights describe routing mass but omit the routed
values, residual pathways, and downstream transformations; they therefore need
not equal feature contributions \citep{jain2019attention,wiegreffe2019attention}.
Our route-to-coefficient matrix is an intervention diagnostic, not a proposed
replacement for SHAP. In full TabPFN v2 it is reported at token-group resolution
because the tokenizer does not expose a one-to-one raw-feature route.

\section{Discussion and limitations}
\label{sec:discussion}

\textbf{Architecture--prior alignment, not a universal winner.}
The paper's three stages support a conditional architectural claim.
The released-model comparison identifies a practically relevant difference in irrelevant-feature suppression, the exact Bayes rules explain why sparse and dense priors demand different computations, and the controlled experiments test whether the two architectures learn those computations equally well.
The alternating-axis advantage under sparsity, together with its contraction at the dense endpoint, is consistent with feature-indexed computation being aligned with context-dependent support inference.
The coefficient decomposition further shows that the sparse prediction gap arises primarily from the learned coefficient map rather than from intercept or nonlinear query effects.

This evidence does not make two-dimensional attention either necessary or sufficient for irrelevant-feature suppression.
Deep row-token models can construct context-dependent metrics, and the fixed-kernel theorem does not cover that unrestricted class.
Conversely, TabPFN v2, TabPFN v2.5, TabICL v2, and TabPFN v3 differ substantially in their tokenization and feature-processing pathways even though all are robust in the null-column experiment.
The supported conclusion is narrower: a coordinate-indexed computational route directly represents the sufficient statistics and posterior gates required for sparse task adaptation, and this alignment can improve learning under a matched training protocol.

\textbf{Functional suppression and internal mechanism.}
We define suppression by the fitted predictor's behavior rather than by exact support recovery.
A null feature is functionally suppressed when changing its query value has little effect on the prediction and when adding many such features produces little loss in predictive performance.
Under the sparse prior, the realized support is hidden, so even a generated inactive coordinate can have nonzero posterior inclusion probability and a nonzero posterior-mean coefficient.
For this reason, the primary coefficient analysis compares the learned map with the posterior mean, while the active--inactive split serves as a secondary diagnostic for locating approximation error.

The input--output diagnostics establish functional dependence but do not locate its internal implementation.
Raw attention weights are also insufficient because a large weight may carry a small value vector, be canceled downstream, or be bypassed by a residual stream.
The route-to-coefficient intervention instead asks how attenuating a specific feature message changes the fitted coefficient map after the perturbation propagates through the remaining network.
Specificity locates the coefficient response, normalized sensitivity measures its magnitude, and $\Delta E$ tests whether the message improves Bayes approximation.
Under the one-sparse prior, active sources have higher specificity throughout the network and greater final-layer sensitivity in both context and query processing. At intermediate sparsity, the active-source advantage persists in specificity and prediction-error changes.
Across support sizes, feature-specific responses persist even at the dense endpoint, whereas active--inactive contrasts identify support selectivity only in sparse settings (App.~\ref{app:cross-scale-support-size}).
These are local effects in fixed networks rather than a complete circuit identification, and norm-matched perturbation controls and replication across alternating-axis training seeds remain outside the present experiments.

\textbf{Tokenization limits raw-feature conclusions.}
The controlled alternating-axis model assigns one address to each raw feature, whereas TabPFN v2 may place multiple raw features in the same token.
A token-level intervention can therefore identify the causal role of a group route but cannot determine which member generated that role.
Padding, singleton passes, and post-hoc splitting could create apparent member-level scores, but each option introduces an additional assumption or changes the model input.
We consequently retain group-level claims for the frozen checkpoint under the evaluated single-estimator configuration and do not extend the intervention result to its full production ensemble.
Any use of these interventions for raw-feature attribution would require a separate identification and validation argument.

\textbf{Statistical and experimental scope.}
The exact theory assumes Gaussian queries, orthogonal context columns, a known support size, and a linear response.
These assumptions remove feature correlation, support-size uncertainty, and nonlinear approximation so that the posterior inclusion calculation is explicit and learned predictors can be compared with the full Bayes function.
The Friedman and OpenML experiments show that the motivating phenomenon extends beyond the orthogonal model, but they do not validate every step of the proposed mechanism outside that model.
With correlated features, relevance can be shared among substitutes and the diagonal metric is no longer the complete statistical object.

The released-model comparison is also observational with respect to architecture because the systems differ in pretraining data, preprocessing, scale, feature limits, and ensemble procedures.
We disable TabDPT retrieval to avoid a separate full-space retrieval failure, while TabPFN v2 uses its native feature subsampling beyond the declared width limit.
The controlled protocol removes many of these confounders by fixing the task prior and evaluation distribution, but differences in depth, token count, and computation remain inherent to the architectures.
A fuller optimization-level comparison would additionally require explicit control of model capacity, training compute, seed variability, feature order, and checkpoint selection.
Sensitivity to these choices and to intermediate support sizes determines how broadly the observed architectural advantage generalizes.

\textbf{Implications.}
The sparse--dense comparison suggests that TFM evaluation should vary the structure of the task prior rather than report only average benchmark accuracy.
It also suggests a design principle: when relevance changes across tasks, preserve an addressable feature axis long enough to accumulate evidence and modulate feature contributions.
This primitive need not be implemented through dense feature attention, because shared feature gates, structured sparse attention, or other permutation-equivariant coordinate modules may provide cheaper routes to the same computation.
The theory identifies the required statistical map, not a single mandatory architecture.

\section{Details for the pretrained production-model experiments}
\label{app:part-i-production-models}

\subsection{Shared paired data construction}
\label{app:part-i-shared-data}

\textbf{Synthetic tasks.}
We use seven unique Friedman configurations.
A sample-size sweep sets $N\in\{250,500,1000,2000\}$ at noise standard deviation $\sigma=1$, while a noise sweep sets $\sigma\in\{0.5,1,2,4\}$ at $N=1000$; the shared $(N,\sigma)=(1000,1)$ configuration is evaluated once.
Each configuration contains 50 independently generated tasks with an 80/20 context/query split and five relevant features sampled from $\operatorname{Unif}(0,1)$.
Both benchmarks use the null-to-original feature ratios $\rho\in\{0,0.2,0.4,1,2,4,7\}$.
At the largest dose, we generate 35 null features drawn i.i.d.\ from $\operatorname{Unif}(0,1)$, the same marginal distribution as the relevant features, so a null column differs from a relevant one only in being independent of the response; context and query rows use separate random streams.
Smaller doses take nested prefixes of the same feature bank, so the relevant-feature values, targets, response noise, split, and previously introduced null columns are identical within a paired comparison.

\textbf{Real tasks and preprocessing.}
We use the 42 OpenML regression datasets listed in Table~\ref{tab:part-i1-openml}.  A deterministic permutation selects at most 1,000 rows, followed by an approximately 80/20 context/query split; smaller tables use all available rows without resampling.  Feature preprocessing is fit only on the context partition.  Constant columns are removed, numeric missing values are median-imputed, and categorical features are ordinal-encoded, with a finite extra code for categories observed only in the query partition.  This experiment-level preprocessing retains at most one numeric column per original feature.  After the null-feature construction below, the resulting matrix is passed directly to each model, whose native inference pipeline may apply additional data-dependent preprocessing or feature expansion internally.

For a real table with $d$ retained features, we add $r=\lceil\rho d\rceil$ null columns.  At fractional doses, a seeded ordering selects $r$ original columns; at integer doses, we concatenate complete copies of the table.  Each copy is permuted by whole rows, preserving dependence among its columns, with independent permutations for context and query.  Smaller doses are nested within larger ones, and a final seeded column permutation removes positional distinctions between original and null features.  Because $r$ is integer-valued, the realized $r/d$ can slightly exceed the requested $\rho$ on narrow tables; cross-dataset summaries use the requested $\rho$.

\textbf{Construction repeats.}
Every Friedman task and every OpenML dataset is augmented five times independently, and these five construction repeats are used in every analysis in this section.  A repeat redraws only the augmentation: the synthetic null columns, and, on real tables, the column selection and the whole-row permutations.  The underlying task is held fixed, so the relevant features, targets, response noise, and context/query split are identical across the five repeats of a given task or dataset.  All randomness is derived deterministically from a single run seed together with the source, the dataset identifier, and the repeat index, so one repeat presents identical inputs to every model and nests across doses; the five repeats are therefore paired across models and mutually independent.  Unless stated otherwise, every summary below averages the five repeats within a task or dataset before aggregating across tasks or datasets.

\begin{table}[!htbp]
\centering
\caption{OpenML regression datasets used in the irrelevant-feature experiments of Section~\ref{sec:part-i-production-models}.  Numbers are OpenML dataset IDs.}
\label{tab:part-i1-openml}
\scriptsize
\resizebox{\linewidth}{!}{%
\begin{tabular}{r l r l r l}
\hline
41021 & Moneyball & 44971 & White Wine & 44989 & King County \\
44956 & Abalone & 44972 & Red Wine & 44990 & Brazilian Houses \\
44957 & Airfoil Self Noise & 44973 & Grid Stability & 44992 & FPS Benchmark \\
44958 & Auction Verification & 44974 & Video Transcoding & 44993 & Health Insurance \\
44959 & Concrete Strength & 44975 & Wave Energy & 44994 & Cars \\
44960 & Energy Efficiency & 44976 & SARCOs & 45012 & FIFA \\
44962 & Forest Fires & 44977 & California Housing & 45402 & Space GA \\
44963 & Physicochemical Protein & 44978 & CPU Activity & 196 & Auto MPG \\
44964 & Superconductivity & 44979 & Diamonds & 230 & Machine CPU \\
44965 & Geographical Origin of Music & 44980 & Kin8nm & 560 & Body Fat \\
44966 & Solar Flare & 44981 & PumaDyn32NH & 42370 & Yacht Hydrodynamics \\
44967 & Student Performance & 44983 & Miami Housing & 44146 & Medical Charges \\
44969 & Naval Propulsion & 44984 & CPS88 Wages & 46286 & Communities and Crime \\
44970 & QSAR Fish Toxicity & 44987 & Socmob & 46292 & Servo \\
\hline
\end{tabular}}
\end{table}

\subsection{Predictive degradation under null features}
\label{app:part-i1-null-dose}

\textbf{Released checkpoints and inference.}
We use the official released regression checkpoints in Table~\ref{tab:part-i1-inference} and the default inference settings of the package versions listed there, including native preprocessing and target handling.  We make two exceptions.  First, we disable TabDPT retrieval while retaining the full context.  Second, we set \texttt{ignore\_pretraining\_limits=true} for TabPFN v2 because its declared 500-feature limit is below the maximum experimental width of 976; its native pipeline then performs estimator-level feature subsampling in these wide conditions.  TabPFN v2.5 and v3 have 2,000-feature limits and require no override.  This subsampling therefore affects only TabPFN v2 on the widest OpenML conditions: the largest Friedman width is 40 features, and TabPFN v2.5 and v3 reach comparable OpenML robustness without any override.  TabPFN v2's stability is thus not an artifact of discarding the added columns.

\begin{table}[!htbp]
\centering
\caption{Official released regression checkpoints and departures from their package-default inference settings.}
\label{tab:part-i1-inference}
\scriptsize
\resizebox{\linewidth}{!}{%
\begin{tabular}{l l l l}
\hline
System & Package & Official checkpoint & Departure from default \\
\hline
TabDPT (retrieval off) & \texttt{tabdpt 1.2.0} &
\texttt{tabdpt1\_2.safetensors} & retrieval disabled \\
TabPFN v2 & \texttt{tabpfn 8.0.8} &
\texttt{tabpfn-v2-regressor.ckpt} & width-limit override \\
TabPFN v2.5 & \texttt{tabpfn 8.0.8} &
\texttt{tabpfn-v2.5-regressor-v2.5\_default.ckpt} & none \\
TabICL v2 & \texttt{tabicl 2.1.1} &
\texttt{tabicl-regressor-v2-20260212.ckpt} & none \\
TabPFN v3 & \texttt{tabpfn 8.0.8} &
\texttt{tabpfn-v3-regressor-v3\_default.ckpt} & none \\
\hline
\end{tabular}}
\end{table}

All reported inference used PyTorch 2.7.1 (CUDA 12.8 build) and NVIDIA GeForce RTX 5090 GPUs.

\textbf{Metrics and aggregation.}
We use the query-set $R^2$ drop $\mathrm{Drop}_{a,t}(\rho)=R^2_{a,t}(0)-R^2_{a,t}(\rho)$.  We first average the five construction repeats within each task or dataset and then report medians; the appendix figures below add the corresponding IQRs, which the main-text figure omits because they describe the spread of the benchmark collection rather than the model comparison.  For the primary two-model comparison, we subtract these repeat-averaged drops within each task: $\Delta_t(\rho)=\mathrm{Drop}_{\mathrm{TabDPT},t}(\rho)-\mathrm{Drop}_{\mathrm{TabPFN\ v2},t}(\rho)$.  Each Friedman configuration contributes 50 tasks, so the pooled synthetic summary weights the seven configurations equally.

\begin{table}[!htbp]
\centering
\caption{Median paired $R^2$ drop at the maximum dose $\rho=7$.  Friedman pools seven configurations with 50 tasks each; OpenML contains 42 datasets.  Brackets contain the 25th and 75th percentiles.}
\label{tab:part-i1-all-models-rho7}
\small
\begin{tabular}{l c c}
\hline
System & Friedman & OpenML \\
\hline
TabDPT (retrieval off) & 0.0370 [0.0246, 0.0562] & 0.0713 [0.0291, 0.1584] \\
TabPFN v2 & 0.0008 [$-$0.00003, 0.0018] & 0.0014 [0.0001, 0.0096] \\
TabPFN v2.5 & $-$0.0002 [$-$0.0012, 0.0002] & 0.0011 [$-$0.0010, 0.0063] \\
TabICL v2 & 0.0021 [0.0011, 0.0055] & 0.0026 [0.0000, 0.0084] \\
TabPFN v3 & 0.0014 [0.0005, 0.0034] & 0.0017 [$-$0.0001, 0.0057] \\
\hline
\end{tabular}
\end{table}

\begin{table}[!htbp]
\centering
\caption{Paired $R^2$-drop difference $\Delta_t(\rho)$ between TabDPT (retrieval off) and TabPFN v2.  Entries are medians [25th, 75th percentiles] across 350 Friedman task-settings or 42 OpenML datasets, after averaging five construction repeats within each task.  Positive values mean greater degradation for TabDPT.}
\label{tab:part-i1-paired-contrast}
\small
\begin{tabular}{c c c}
\hline
$\rho$ & Friedman & OpenML \\
\hline
0.2 & 0.0001 [$-$0.0004, 0.0007] & 0.0003 [$-$0.0011, 0.0040] \\
0.4 & 0.0002 [$-$0.0005, 0.0011] & 0.0005 [$-$0.0015, 0.0083] \\
1 & 0.0006 [$-$0.0002, 0.0020] & 0.0022 [$-$0.0004, 0.0214] \\
2 & 0.0020 [0.0008, 0.0051] & 0.0117 [0.0020, 0.0476] \\
4 & 0.0107 [0.0049, 0.0213] & 0.0270 [0.0120, 0.1143] \\
7 & 0.0363 [0.0238, 0.0550] & 0.0698 [0.0280, 0.1525] \\
\hline
\end{tabular}
\end{table}

\begin{figure}[!tbp]
    \centering
    \includegraphics[width=\linewidth]{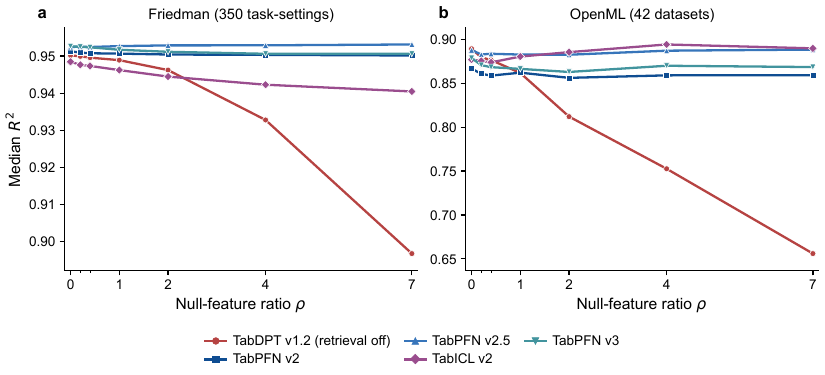}
    \caption{Raw query-set $R^2$ without subtracting the $\rho=0$ baseline.  Points show medians across tasks or datasets after averaging construction repeats.  The five models have comparable clean performance within each benchmark setting.}
    \label{fig:part-i1-raw-r2}
\end{figure}

\textbf{Friedman setting sweeps.}
Figure~\ref{fig:part-i1-dose} resolves the Friedman dose response by observation noise and by total sample size, and repeats the OpenML summary with its IQR.  TabDPT degrades in every setting, most strongly at small $N$.  The other systems stay close to their clean baselines, except that TabICL v2 degrades noticeably at the smallest sample size ($N=250$).

\begin{figure}[!tbp]
    \centering
    \includegraphics[width=\linewidth]{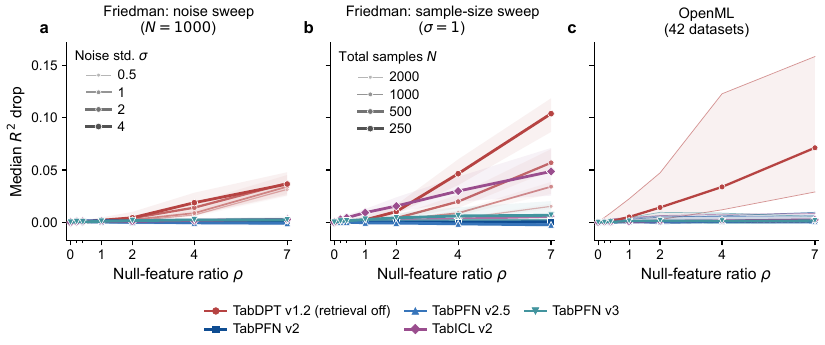}
    \caption{Friedman noise and sample-size sweeps for the $R^2$ drop.  Panels (a) and (b) vary observation noise and total sample size, respectively; shade and line width encode the setting, and ribbons span the IQR across 50 tasks.  Panel (c) reports the median and IQR across 42 OpenML datasets.}
    \label{fig:part-i1-dose}
\end{figure}

\textbf{OpenML width stratification.}
Figure~\ref{fig:part-i1-openml-dimension} separates the OpenML datasets by their pre-augmentation feature count.  The dose-dependent TabDPT degradation is visible in all three strata and is not driven only by the naturally widest tables.

\begin{figure}[!tbp]
    \centering
    \includegraphics[width=\linewidth]{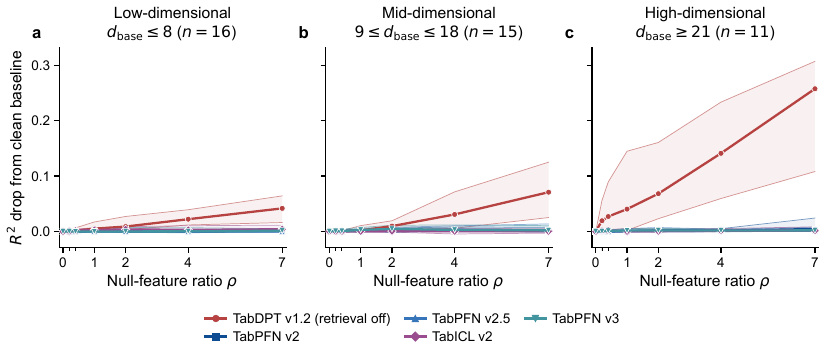}
    \caption{OpenML $R^2$ drop stratified by the number of retained features before augmentation, $d_{\mathrm{base}}$.  Points and ribbons show medians and IQRs across datasets in each stratum.}
    \label{fig:part-i1-openml-dimension}
\end{figure}

\subsection{Functional dependence on null features}
\label{app:part-i2-functional-suppression}

\textbf{Setup.}
This analysis uses the same seven Friedman configurations and 42 OpenML datasets as the dose response above, at all seven ratios and all five construction repeats.  The primary comparison is between TabDPT with retrieval disabled and TabPFN v2, with the checkpoints and inference settings in Table~\ref{tab:part-i1-inference}.

\textbf{Single-feature permutation shift.}
Let $\mathbf y_{\mathrm{query}}$ denote the held-out query targets, $\widehat{\mathbf y}$ the query predictions, and $\widehat{\mathbf y}^{\mathrm{perm}(j)}$ the predictions after permuting query feature $j$ while holding the context and all other query features fixed.
We define the normalized prediction shift as
\begin{equation}
    S_j^{\mathrm{perm}}
    =\frac{\operatorname{RMSE}(\widehat{\mathbf y}^{\mathrm{perm}(j)},\widehat{\mathbf y})}
    {\operatorname{SD}(\mathbf y_{\mathrm{query}})},
    \label{eq:part-i-permutation-shift}
\end{equation}
and call it the null-feature shift when $j$ is an added null coordinate.
Corresponding features use the same permutation across models and nested doses.
The five construction repeats already provide independent draws of the null features, so we do not add a second permutation loop within each repeat.
The known relevant features in the synthetic tasks provide positive controls.
For real data, we compare added null features with original features, whose relevance is unknown.

\textbf{One-way PDP variation.}
We additionally evaluate one-way partial-dependence curves at the empirical quantiles $5\%$, $27.5\%$, $50\%$, $72.5\%$, and $95\%$.
Each synthetic condition includes all five relevant features.
Real conditions include at most five original features, and both benchmarks include at most five null features.
Grids are fixed across models and doses.
Relevant-feature and original-feature grids come from the clean context; a real null feature uses the grid of the original feature from which it was copied.
We summarize feature $j$ by
\begin{equation}
    S_j^{\mathrm{PDP}}
    =\frac{\operatorname{SD}_{u\in\mathcal G_j}
    \left[Q_{\mathrm{eval}}^{-1}\sum_{q=1}^{Q_{\mathrm{eval}}}\widehat f(D,x_q^{j\leftarrow u})\right]}
    {\operatorname{SD}(y_{\mathrm{query}})},
\end{equation}
where $\mathcal G_j$ is its five-point grid, $Q_{\mathrm{eval}}$ is the number of query rows, and $x_q^{j\leftarrow u}$ replaces coordinate $j$ of $x_q$ with $u$ while keeping all other coordinates fixed. The context $D$ and fitted model remain fixed throughout.  We report this summary over two sets of null features.  The first is all null features present at the current dose, so its membership grows with $\rho$.  The second is restricted to the null features introduced at the smallest positive dose $\rho=0.2$; because doses are nested, these same columns are still present at every larger dose, and the restriction is the same for both models.  A change in the first summary can come either from the model's dependence on a given column or from the changing composition of the set being averaged, whereas the second follows one fixed set of columns across the dose grid.

\textbf{Aggregation.}
For each model, dose, and feature role, we average first across features within each task and repeat, and then across the five repeats.  The appendix figures (Figures~\ref{fig:part-i2-permutation}, \ref{fig:part-i2-pdp}, \ref{fig:part-i2-friedman-sweeps}, \ref{fig:part-i2-openml-dimension-permutation}, and~\ref{fig:part-i2-openml-dimension-pdp}) report medians and IQRs across tasks or datasets; the main-text figure (Figure~\ref{fig:part-i-functional}(b)) reports medians only.

\textbf{Relevant-feature controls.}
Figure~\ref{fig:part-i2-permutation} places the main-text result next to its positive controls.
Both models respond an order of magnitude more strongly to relevant Friedman features or original OpenML covariates than to added null features, so the difference in null-feature sensitivity is not a difference in overall responsiveness.
TabDPT's response to the relevant Friedman features also decreases with dose, whereas TabPFN v2 remains stable.

\begin{figure}[!tbp]
    \centering
    \includegraphics[width=\linewidth]{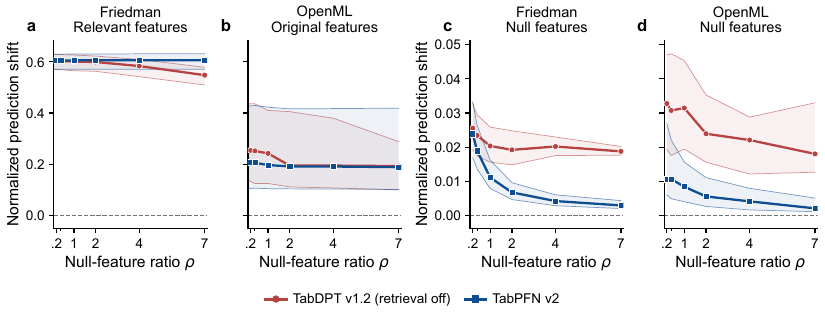}
    \caption{Normalized prediction shift from Eq.~\ref{eq:part-i-permutation-shift}.
    Panels (a,b) show relevant Friedman features and original OpenML features; panels (c,d) show added null features in the same order.
    The two feature-role pairs use separate vertical scales.
    Points show medians and ribbons span the IQR after averaging features and five construction repeats within each task.
    Friedman panels pool the seven configurations equally; OpenML panels summarize 42 datasets.
    Panels (c,d) show the null-feature shift plotted in Figure~\ref{fig:part-i-functional}(b).}
    \label{fig:part-i2-permutation}
\end{figure}

\begin{figure}[!tbp]
    \centering
    \includegraphics[width=\linewidth]{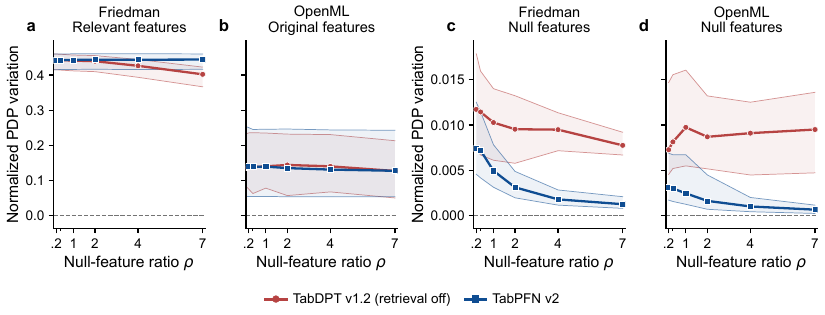}
    \caption{Normalized one-way PDP variation.
    Panels (a,b) show relevant Friedman features and original OpenML features; panels (c,d) show added null features in the same order, with a separate vertical scale for each feature-role pair.
    Points show medians across tasks or datasets and ribbons span the IQR after averaging eligible features and five construction repeats within each task.
    The ordering agrees with the permutation intervention in Figure~\ref{fig:part-i2-permutation}.}
    \label{fig:part-i2-pdp}
\end{figure}

Restricting the PDP summary to the null features introduced at $\rho=0.2$, and following those same columns as the dose grows, reproduces the ordering obtained over all null features: TabDPT keeps depending on them, while TabPFN v2 suppresses them further at larger doses.  The ordering therefore does not depend on which null features enter the average at each dose.

\textbf{Representative PDP curves.}
Figures~\ref{fig:part-i2-pdp-curves-friedman} and~\ref{fig:part-i2-pdp-curves-grid-stability} show column-resolved examples at $\rho=7$ for one Friedman task and Grid Stability, respectively.
Each curve is centered at its $50\%$ grid point.
The Friedman example labels the five relevant columns $x_j$ and the added null columns $z_k$.
For Grid Stability, an added column $z_k$ is labeled by the original column from which it was copied; matching colors identify the same source column.
Both figures show one deterministic construction repeat, whereas the aggregate results in Figure~\ref{fig:part-i2-pdp} use all five repeats.
Relevant-feature or original-feature panels and null-feature panels use separate vertical scales, each shared across the two models.

\begin{figure}[!tbp]
    \centering
    \includegraphics[width=\linewidth]{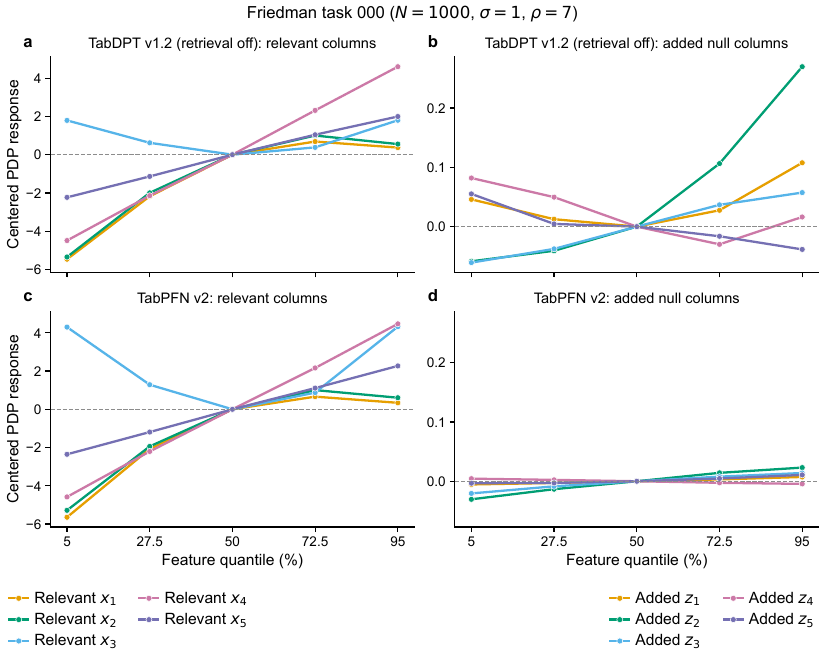}
    \caption{Column-resolved PDPs for a fixed Friedman task with $N=1000$ and $\sigma=1$.
    Both models respond to the relevant features, while TabDPT shows a larger response to the added null features.}
    \label{fig:part-i2-pdp-curves-friedman}
\end{figure}

\begin{figure}[!tbp]
    \centering
    \includegraphics[width=\linewidth]{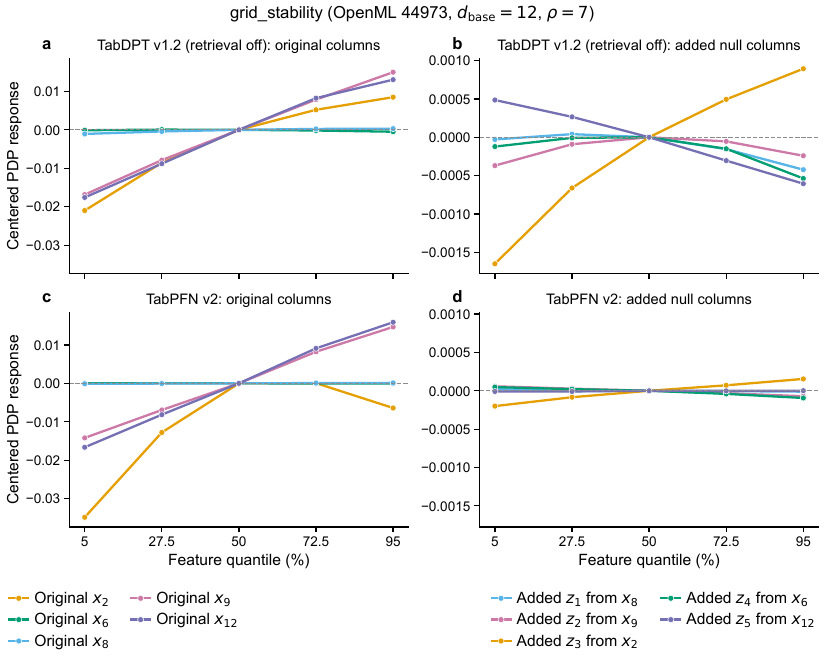}
    \caption{Column-resolved PDPs for Grid Stability, a medium-dimensional OpenML dataset with 12 original features.}
    \label{fig:part-i2-pdp-curves-grid-stability}
\end{figure}

\textbf{Friedman setting dependence.}
Figure~\ref{fig:part-i2-friedman-sweeps} resolves the permutation result by noise and sample size.
The null-feature gap persists in both sweeps from $\rho=1$ onward and is largest at smaller sample sizes; the response to relevant features remains substantially larger than the response to null features.

\begin{figure}[!tbp]
    \centering
    \includegraphics[width=\linewidth]{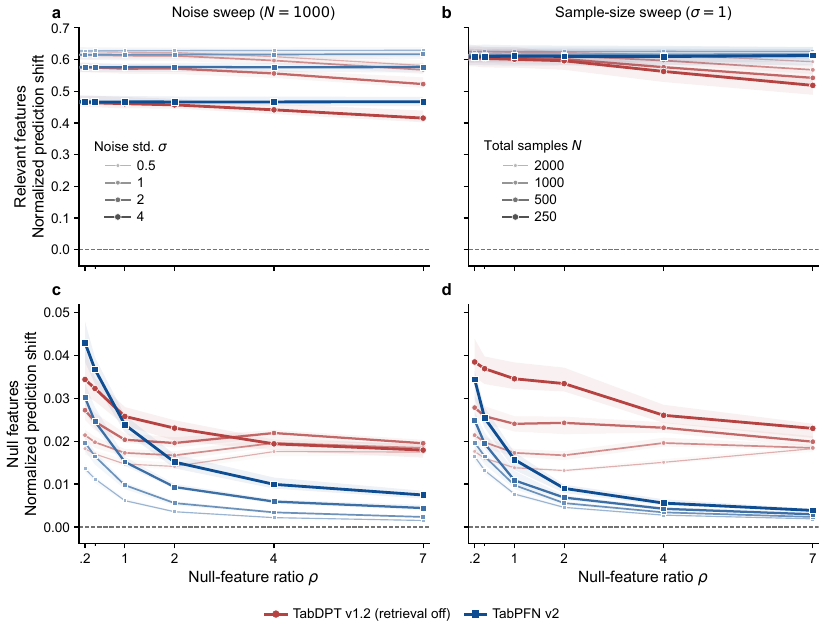}
    \caption{Friedman permutation sensitivity across noise and sample-size settings.
    Shade and line width encode the setting; ribbons show IQRs across 50 tasks.
    Columns vary noise (left) and total sample size (right), while rows show relevant (top) and null (bottom) features.}
    \label{fig:part-i2-friedman-sweeps}
\end{figure}

\textbf{OpenML width stratification.}
Stratifying OpenML datasets by their original width preserves the larger TabDPT sensitivity to null features under both interventions (Figures~\ref{fig:part-i2-openml-dimension-permutation} and~\ref{fig:part-i2-openml-dimension-pdp}).

\begin{figure}[!tbp]
    \centering
    \includegraphics[width=\linewidth]{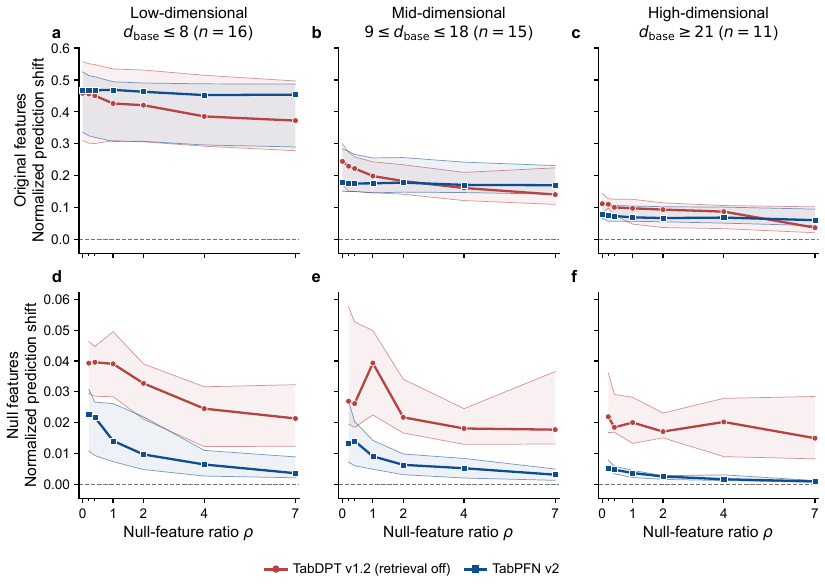}
    \caption{OpenML permutation sensitivity stratified by pre-augmentation feature count.  Rows show original and null features; points and ribbons show medians and IQRs across datasets.}
    \label{fig:part-i2-openml-dimension-permutation}
\end{figure}

\begin{figure}[!tbp]
    \centering
    \includegraphics[width=\linewidth]{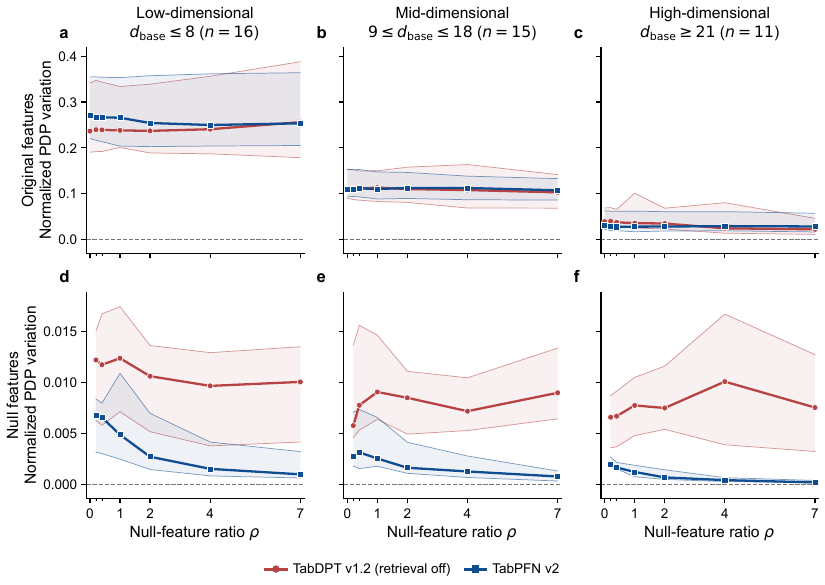}
    \caption{OpenML PDP variation stratified by pre-augmentation feature count, using the same aggregation as Figure~\ref{fig:part-i2-pdp}.}
    \label{fig:part-i2-openml-dimension-pdp}
\end{figure}

\subsection{Replication across TabDPT releases}
\label{app:part-i-tabdpt-versions}

\textbf{Motivation and protocol.}
The two analyses above evaluate the TabDPT release that was current when we ran them, v1.2.  To test whether the irrelevant-feature sensitivity is specific to that release, we repeat both experiments for the official v1.0, v1.1 and v1.3 checkpoints on exactly the same 350 Friedman task-settings, the same 42 OpenML datasets, the same five construction repeats and the same dose grid.  Each release runs with the default inference settings of its own package version (Table~\ref{tab:part-i-tabdpt-versions-inference}), with the single departure used throughout this section: the full training context, that is retrieval disabled.  Ensembling, preprocessing, outlier clipping, feature reduction and target handling therefore differ between releases exactly as their defaults differ, and the comparison is observational with respect to architecture in the same sense as the cross-system comparison above.  Table~\ref{tab:part-i-tabdpt-versions-inference} lists the artifacts and the resulting departures.

\begin{table}[!htbp]
\centering
\caption{TabDPT releases used in the replication.  The default context of \texttt{tabdpt 1.1.14} already exceeds every context in our tasks, and \texttt{tabdpt 1.2.0} and \texttt{1.3.0} use the full context by default, so only v1.0 departs from its package default.}
\label{tab:part-i-tabdpt-versions-inference}
\scriptsize
\resizebox{\linewidth}{!}{%
\begin{tabular}{l l l l l}
\hline
Release & Package & Official checkpoint & Default ensemble & Departure from default \\
\hline
TabDPT v1.0 & \texttt{tabdpt 0.1.0} & \texttt{tabdpt\_76M.ckpt} & single estimator & full context instead of 128-row retrieval \\
TabDPT v1.1 & \texttt{tabdpt 1.1.14} & \texttt{tabdpt1\_1.safetensors} & 8 & none in effect \\
TabDPT v1.2 & \texttt{tabdpt 1.2.0} & \texttt{tabdpt1\_2.safetensors} & 8 & none \\
TabDPT v1.3 & \texttt{tabdpt 1.3.0} & \texttt{tabdpt1\_3.safetensors} & 8 & none \\
\hline
\end{tabular}}
\end{table}

\textbf{Dose response.}
Figure~\ref{fig:part-i-tabdpt-versions-r2}(a,~b) shows that every release degrades as null columns accumulate, and that the degradation shrinks monotonically from v1.0 to v1.3 on both benchmarks.  At the largest dose the median $R^2$ drop on Friedman falls from $0.1466$ for v1.0 to $0.0089$ for v1.3, and on OpenML from $0.1205$ to $0.0619$; the TabPFN v2 reference stays at $0.0008$ and $0.0014$ (Table~\ref{tab:part-i-tabdpt-versions-rho7}).  Even the newest release therefore remains more than an order of magnitude more dose-sensitive than TabPFN v2 on OpenML, and it still has a larger $R^2$ drop than TabPFN v2 on $94.0\%$ of Friedman task-settings and $92.9\%$ of OpenML datasets.  Panels (c,~d) give the absolute medians behind these drops: the releases differ in clean-task accuracy as well, most visibly on OpenML, which is why the paired drop rather than the raw $R^2$ is the comparable quantity.

\begin{figure}[!tbp]
    \centering
    \includegraphics[width=\linewidth]{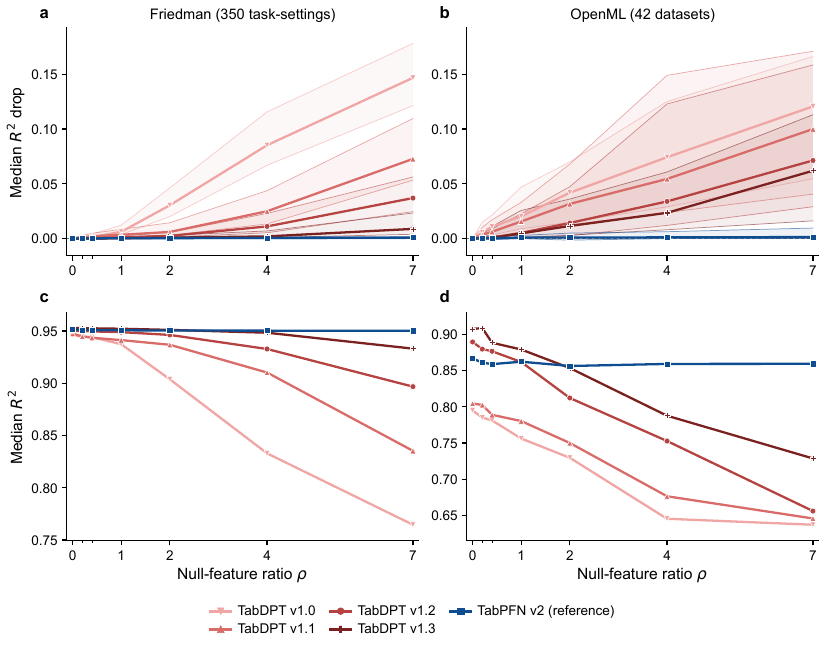}
    \caption{Null-feature dose response of four TabDPT releases with TabPFN v2 as reference.  (a,~b) Median $R^2$ drop with interquartile bands.  (c,~d) Median absolute $R^2$ at each dose.  Aggregation matches Figure~\ref{fig:part-i1-dose}: five construction repeats are averaged within each task before taking medians across 350 Friedman task-settings or 42 OpenML datasets.}
    \label{fig:part-i-tabdpt-versions-r2}
\end{figure}

\textbf{Functional sensitivity.}
Figure~\ref{fig:part-i2-tabdpt-versions-permutation} repeats the functional analysis for the same releases.
Panels (a,~b) are the positive controls: all releases respond to relevant Friedman features or original OpenML features about an order of magnitude more strongly than to the added nulls, so the differences in panels (c,~d) are not differences in overall responsiveness.
On the added null features the ordering matches the dose response, with the median null-feature shift at $\rho=7$ falling from $0.0344$ (v1.0) to $0.0115$ (v1.3) on Friedman and from $0.0276$ to $0.0156$ on OpenML, against $0.0029$ and $0.0021$ for TabPFN v2.
Newer releases suppress null coordinates better, and none of them matches the suppression achieved by TabPFN v2.

\begin{figure}[!tbp]
    \centering
    \includegraphics[width=\linewidth]{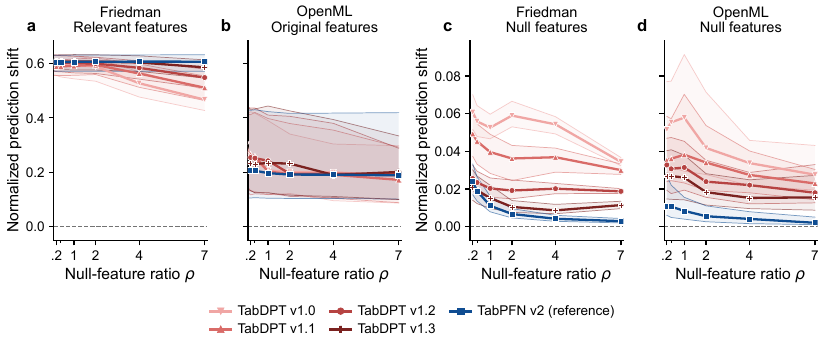}
    \caption{Normalized prediction shift under single-feature permutation for four TabDPT releases, with TabPFN v2 as reference.
    (a,~b) Relevant Friedman features and original OpenML features as positive controls.
    (c,~d) Added null features.
    Aggregation matches Figure~\ref{fig:part-i2-permutation}.}
    \label{fig:part-i2-tabdpt-versions-permutation}
\end{figure}

\begin{table}[!htbp]
\centering
\caption{TabDPT releases at the maximum dose $\rho=7$.  Entries are medians [25th, 75th percentiles] across 350 Friedman task-settings or 42 OpenML datasets after averaging five construction repeats within each task.  The last column is the percentage of paired tasks on which the release exceeds TabPFN v2 in $R^2$ drop, Friedman / OpenML.}
\label{tab:part-i-tabdpt-versions-rho7}
\scriptsize
\resizebox{\linewidth}{!}{%
\begin{tabular}{l c c c c c}
\hline
& \multicolumn{2}{c}{$R^2$ drop} & \multicolumn{2}{c}{Null-feature shift} & Exceeds TabPFN v2 \\
Release & Friedman & OpenML & Friedman & OpenML & (\% of tasks) \\
\hline
TabDPT v1.0 & 0.1466 [0.1213, 0.1780] & 0.1205 [0.0547, 0.1661] & 0.0344 [0.0324, 0.0367] & 0.0276 [0.0194, 0.0431] & 100.0 / 97.6 \\
TabDPT v1.1 & 0.0727 [0.0533, 0.1095] & 0.0999 [0.0406, 0.1709] & 0.0300 [0.0278, 0.0333] & 0.0229 [0.0143, 0.0304] & 100.0 / 97.6 \\
TabDPT v1.2 & 0.0370 [0.0246, 0.0562] & 0.0713 [0.0291, 0.1584] & 0.0188 [0.0177, 0.0202] & 0.0180 [0.0127, 0.0330] & 98.9 / 95.2 \\
TabDPT v1.3 & 0.0089 [0.0039, 0.0233] & 0.0619 [0.0162, 0.1130] & 0.0115 [0.0096, 0.0135] & 0.0156 [0.0088, 0.0272] & 94.0 / 92.9 \\
TabPFN v2 (reference) & 0.0008 [$-$0.00003, 0.0018] & 0.0014 [0.0001, 0.0096] & 0.0029 [0.0020, 0.0043] & 0.0021 [0.0011, 0.0051] & --- \\
\hline
\end{tabular}}
\end{table}

\textbf{Scope.}
These runs show that the phenomenon is a property of every released TabDPT version rather than of one checkpoint, and that the releases have been reducing it.  They do not attribute the trend to any particular architectural change: consecutive releases differ simultaneously in weights and training data, default ensembling, clipping thresholds, preprocessing and output parameterization (see the release notes of each package version).

\section{Derivations and proofs for the sparse and dense priors}
\label{app:theory-details}

We use the model and notation of Section~\ref{sec:controlled-priors}.
Throughout, $n\ge d\ge2$, $v^2,\sigma^2>0$, and
$k\in\{1,\ldots,d\}$ are fixed. Expectations are taken under $P_k$ and
the fixed-design regression model.

\subsection{Sufficiency and matched prior moments}
\label{app:orthogonal-sufficiency}

The design $X$ can be fixed or drawn independently of the coefficient prior
subject to $X^\top X=nI_d$. One random construction is $X=\sqrt{n}Q$,
where $Q$ contains the orthonormal columns from a thin QR factorization of an
$n\times d$ matrix with independent standard Gaussian entries. Given $X$,
draw $S$ uniformly from $\mathcal{S}_k$, draw its $k$ active coefficients
independently with variance $v^2/k$, and draw observation noise independently.

Let $P_X=XX^\top/n$. Orthogonality gives
\begin{equation}
    \|y-X\beta\|_2^2
    =\|(I_n-P_X)y\|_2^2+n\|c-\beta\|_2^2,
    \qquad c=\frac{X^\top y}{n}.
    \label{eq:app-likelihood-factorization}
\end{equation}
The first term is independent of $\beta$, so the posterior depends on $D$
only through $c$. Moreover,
\begin{equation}
    c=\frac{1}{n}X^\top y=\beta+\xi,
    \qquad \xi=\frac{X^\top\varepsilon}{n}
    \sim\mathcal{N}(0,\tau^2I_d),
    \qquad \tau^2=\frac{\sigma^2}{n}.
    \label{eq:context-statistic}
\end{equation}
Indeed,
$\operatorname{Cov}(X^\top\varepsilon/n\mid X)
=\sigma^2X^\top X/n^2=\tau^2I_d$. Its noise law is therefore the same for
every admissible $X$.

For each coordinate,
\begin{equation}
    \Pr(j\in S)=\frac{k}{d},
    \qquad
    \mathbb{E}[\beta_j]=0,
    \qquad
    \mathbb{E}[\beta_j^2]=\frac{k}{d}\frac{v^2}{k}=\frac{v^2}{d}.
\end{equation}
For distinct $j,r$, conditional independence and zero conditional means give
$\mathbb{E}[\beta_j\beta_r\mid S]=0$. Averaging over $S$ and summing the
coordinate variances proves the matched-moment identities stated in
Section~\ref{sec:controlled-priors}, including the dense endpoint.

\subsection{Support posterior and posterior-mean coefficient map}
\label{app:bayes-posterior}

Write $a_k=v^2/k$. Conditional on a candidate support $A\in\mathcal{S}_k$,
integrating the Gaussian coefficient prior against the likelihood gives
independent coordinates with
\begin{equation}
    c_j\mid S=A\sim
    \begin{cases}
        \mathcal{N}(0,a_k+\tau^2), & j\in A,\\
        \mathcal{N}(0,\tau^2), & j\notin A.
    \end{cases}
\end{equation}
Their joint density is
\begin{equation}
    p(c\mid S=A)
    =\frac{\exp\!\left(-\frac{\|c\|_2^2}{2\tau^2}
    +\theta_k\sum_{j\in A}c_j^2\right)}
    {(2\pi)^{d/2}(a_k+\tau^2)^{k/2}(\tau^2)^{(d-k)/2}},
    \qquad
    \theta_k=\frac{a_k}{2\tau^2(a_k+\tau^2)}.
    \label{eq:app-support-likelihood}
\end{equation}
Since all candidate supports have size $k$ and equal prior probability,
all factors except the support score cancel in Bayes' rule, giving
\begin{equation}
    \pi_k(A\mid c)
    =\frac{\exp\!\left(\theta_k\sum_{j\in A}c_j^2\right)}
    {\sum_{B\in\mathcal{S}_k}
    \exp\!\left(\theta_k\sum_{r\in B}c_r^2\right)},
    \qquad A\in\mathcal{S}_k.
    \label{eq:support-posterior}
\end{equation}
Marginalizing the support indicators gives
\begin{equation}
    q_{k,j}(c)=\Pr(j\in S\mid c)
    =\sum_{\substack{A\in\mathcal{S}_k\\j\in A}}\pi_k(A\mid c),
    \qquad
    \sum_{j=1}^d q_{k,j}(c)=\mathbb{E}[|S|\mid c]=k.
    \label{eq:posterior-inclusion}
\end{equation}

For an active coordinate, Gaussian conjugacy gives
\begin{equation}
    \beta_j\mid c,S=A
    \sim\mathcal{N}(\lambda_kc_j,\lambda_k\tau^2),
    \qquad
    \lambda_k=\frac{a_k}{a_k+\tau^2}.
\end{equation}
Inactive coordinates remain zero. Hence
\begin{equation}
    \mathbb{E}[\beta_j\mid c]
    =\sum_{A\in\mathcal{S}_k}\pi_k(A\mid c)
    \lambda_kc_j\mathbf{1}\{j\in A\}
    =\lambda_kq_{k,j}(c)c_j.
\end{equation}
Sufficiency and independence of the prediction noise then imply
\begin{equation}
    \mathbb{E}[y_q\mid D,x]
    =x^\top\mathbb{E}[\beta\mid D]
    =x^\top\eta_k^*(c),
\end{equation}
which is the unique squared-risk minimizer up to almost-sure equality.

\subsection{Data dependence of the Bayes feature metric}
\label{app:bayes-kernels}

We verify the claim in Section~\ref{sec:controlled-priors} that the Bayes
predictor is a row kernel with diagonal feature metric
$\operatorname{diag}(q_k(c))$, which is the identity for $k=d$ and varies
with the context for every $k<d$.

\begin{proof}
Substituting $c=n^{-1}\sum_i x_i y_i$ into the posterior mean gives
\begin{equation}
    f_k^*(D,x)
    =\lambda_k x^\top\operatorname{diag}(q_k(c))c
    =\frac{\lambda_k}{n}\sum_{i=1}^n
    y_i x^\top\operatorname{diag}(q_k(c))x_i.
\end{equation}
For $k=d$, the only support is $[d]$, so $q_d(c)=\mathbf{1}$ and the metric
is the identity.

For $k<d$, fix a coordinate $j$ and fix $c_{-j}$. Define the positive finite
quantities
\begin{equation}
    U_j=\sum_{\substack{A\subseteq[d]\setminus\{j\}\\|A|=k-1}}
    \exp\!\left(\theta_k\sum_{r\in A}c_r^2\right),
    \qquad
    V_j=\sum_{\substack{A\subseteq[d]\setminus\{j\}\\|A|=k}}
    \exp\!\left(\theta_k\sum_{r\in A}c_r^2\right).
\end{equation}
The empty-support term is one. Separating supports according to whether
they contain $j$ yields
\begin{equation}
    q_{k,j}(c)
    =\frac{e^{\theta_kc_j^2}U_j}
    {e^{\theta_kc_j^2}U_j+V_j}.
    \label{eq:app-gate-limit}
\end{equation}
Because $\theta_k,U_j,V_j>0$, this quantity is strictly increasing in
$c_j^2$ and converges to one as $|c_j|\to\infty$. The diagonal metric
therefore varies with $c$ for every $k<d$.
\end{proof}

\subsection{Response-affine barrier}
\label{app:fixed-kernel-gap}

\begin{proof}[Proof of Theorem~\ref{thm:fixed-kernel-gap}]
Fix $X$ and $x\ne0$, and write $\alpha=v^2/d$. The matched first two moments
and independent Gaussian noise give
\begin{equation}
    \mathbb E[y]=0,
    \qquad
    \operatorname{Cov}(y)=\alpha XX^\top+\sigma^2I_n,
    \qquad
    \operatorname{Cov}(y,y_q)=\alpha Xx.
    \label{eq:app-affine-moments}
\end{equation}
The response covariance is positive definite. Because
$X^\top X=nI_d$, the vector $Xx$ is its eigenvector with eigenvalue
$n\alpha+\sigma^2$. Hence the unique least-squares coefficients for predicting
$y_q$ affinely from $y$ are
\begin{equation}
    a^*=0,
    \qquad
    b^*=\operatorname{Cov}(y)^{-1}\operatorname{Cov}(y,y_q)
    =\frac{\alpha}{n\alpha+\sigma^2}Xx
    =\frac{\lambda_d}{n}Xx.
    \label{eq:optimal-fixed-kernel}
\end{equation}
Therefore $(b^*)^\top y=\lambda_dx^\top c$. Uniqueness also follows from
\begin{equation}
    \mathbb E[(a+b^\top y-y_q)^2]
    -\mathbb E[((b^*)^\top y-y_q)^2]
    =a^2+(b-b^*)^\top\operatorname{Cov}(y)(b-b^*).
\end{equation}
The calculation uses only the matched first two moments and consequently has
the same solution for every $P_k$.

Since $f_k^*(D,x)=\mathbb E[y_q\mid D,x]$, the conditional-mean identity gives,
for every $D$-measurable predictor $f$,
\begin{equation}
    \mathbb E[(f(D,x)-y_q)^2]
    =\mathbb E[(f(D,x)-f_k^*(D,x))^2]
    +\mathbb E[(f_k^*(D,x)-y_q)^2].
\end{equation}
The second term does not depend on $f$. Thus the same affine predictor minimizes
approximation error to $f_k^*$, and
\begin{equation}
    \inf_{f\in\mathcal F_{\mathrm{aff}}(x)}
    \mathbb E_D[(f(D,x)-f_k^*(D,x))^2]
    =\mathbb E_D\!\left[
       \{x^\top(\lambda_dc-\eta_k^*(c))\}^2
     \right].
    \label{eq:app-fixed-affine-gap}
\end{equation}

For strict positivity, take $k<d$ and choose $j$ with $x_j\ne0$. At
$c=te_j$, all coordinates of $\lambda_dc-\eta_k^*(c)$ except $j$ vanish.
Moreover, $q_{k,j}(0)=k/d$ and
$q_{k,j}(te_j)\to1$ as $|t|\to\infty$ by
Eq.~\ref{eq:app-gate-limit}, while
\begin{equation}
    \frac{k}{d}
    <\frac{\lambda_d}{\lambda_k}
    =\frac{v^2+k\tau^2}{v^2+d\tau^2}<1.
\end{equation}
For some finite nonzero $t$, therefore,
$x_jt\{\lambda_d-\lambda_kq_{k,j}(te_j)\}\ne0$. The difference is continuous,
so it remains nonzero on an open neighborhood. The marginal law of $c$ is a
mixture of nonsingular Gaussian densities and is strictly positive everywhere;
that neighborhood has positive probability. This proves the strict gap. For
$k=d$, $\eta_d^*(c)=\lambda_dc$ identically, so the gap is zero.
\end{proof}

The result holds separately at each known $n$. If $n$ varies, condition on it
and use $\lambda_d(n)=v^2/(v^2+d\sigma^2/n)$.

\subsection{Alternating-axis realization and exact-\texorpdfstring{$k$}{k} normalization}
\label{app:feature-construction}

\textbf{Architecture class.}
Fix $n$ observed rows and one prediction row, indexed by $i\in[n+1]$, and
augment the $d$ feature tokens in each row with a response/readout token
indexed by $j=0$. A width-$m$ cell-token state is therefore
$h=(h_{ij})_{i\in[n+1],\,j\in\{0,\ldots,d\}}$, with
$h_{ij}\in\mathbb R^m$. At initialization, feature token $(i,j)$, $j\geq1$,
contains $x_{ij}$ (with $x_{n+1,j}=x_j$), observed response token $(i,0)$
contains $y_i$, and $(n+1,0)$ is the prediction readout. Every token also
contains fixed indicators of its feature-versus-response type and its
observed-versus-prediction-row role.

For a token $u$ and an allowed attention group $G(u)$, one softmax-attention
sublayer has the form
\begin{equation}
 \widetilde h_u=h_u+W_O\sum_{v\in G(u)}\alpha_{uv}W_Vh_v,
 \qquad
 \alpha_{uv}=
 \frac{\exp\{(W_Qh_u)^\top(W_Kh_v)+M_{uv}\}}
      {\sum_{w\in G(u)}\exp\{(W_Qh_u)^\top(W_Kh_w)+M_{uw}\}},
 \label{eq:app-axis-attention}
\end{equation}
where $M_{uv}\in\{0,-\infty\}$ is a fixed mask depending only on token type
and observed-versus-prediction status. For feature-axis attention,
$G_F(i,j)=\{(i,\ell):0\leq\ell\leq d\}$; for observation-axis attention at a
feature token, $G_O(i,j)=\{(\ell,j):1\leq\ell\leq n+1\}$. Parameters are
shared over all groups of the same sublayer. Each attention sublayer may be
followed by a residual pointwise map
\begin{equation}
 h_u\longmapsto h_u+W_2\rho(W_1h_u+b_1)+b_2,
 \label{eq:app-pointwise-layer}
\end{equation}
shared over tokens, where $\rho$ is continuous and nonpolynomial. Layer
parameters may differ across depth, and a sublayer may be made inactive by
setting its output projection to zero. We call any finite composition whose
active attention sublayers alternate between the two axes an
\emph{alternating-axis network}; its scalar output is a linear readout of
$h_{n+1,0}$. Denote this class by $\mathcal A_{n,d}$.

\begin{proposition}[Alternating-axis neural approximation]
\label{prop:one-sparse-alternating-realization}
Fix $n,d$, $1\leq k\leq d$, and positive prior and noise parameters.
Let $\mathcal K$ be a compact set of observed datasets and prediction
points satisfying $X^\top X=nI_d$. For every $\epsilon>0$, there exists
$F_\epsilon\in\mathcal A_{n,d}$ such that
\begin{equation}
    \sup_{(D,x)\in\mathcal K}
    |F_\epsilon(D,x)-f_k^*(D,x)|<\epsilon.
    \label{eq:app-one-sparse-approximation}
\end{equation}
The construction uses three active attention sublayers in
feature--observation--feature order and $O(k)$ residual-state channels.
The hidden widths of the pointwise feed-forward maps may depend on
$n,d,k$, the parameters, $\mathcal K$, and $\epsilon$; no corresponding
total-parameter or optimization bound is asserted.
\end{proposition}

\begin{proof}
First, we give an exact computation with continuous pointwise maps, then
approximate these maps by feed-forward networks. Feature-axis
attention makes the response $y_i$ available at every feature cell $(i,j)$ of
an observed row while the residual path preserves $x_{ij}$. A pointwise
map forms the product $x_{ij}y_i$. Second, set the
observation-axis attention logits equal and mask them to the observed rows. At
fixed feature $j$, its output is the exact average
$c_j=n^{-1}\sum_i x_{ij}y_i$; the prediction-cell residual preserves $x_j$. A
pointwise map can now operate on $(c_j,x_j)$.

For general $k$, put $w_j=\exp(\theta_kc_j^2)$ and
$a_j=\lambda_kc_jx_j$, and form the $2k$ channels
$(w_j^r,a_jw_j^r)_{r=1}^k$. Third, feature-axis attention at the prediction
readout is uniform over the $d$ feature tokens, excluding the response
token. Scaling the averages by the known $d$ yields
\begin{equation}
    p_r=\sum_{j=1}^d w_j^r,\qquad
    t_r=\sum_{j=1}^d a_jw_j^r,\qquad 1\leq r\leq k.
    \label{eq:app-pooled-power-sums}
\end{equation}
Let $e_m(w)$ be the elementary symmetric polynomial of degree $m$,
with $e_0=1$. Newton's identities recover these polynomials from the
pooled power sums:
\begin{equation}
    e_m=\frac{1}{m}\sum_{r=1}^m(-1)^{r-1}e_{m-r}p_r,
    \qquad 1\leq m\leq k.
    \label{eq:app-newton-recursion}
\end{equation}
For completeness, this follows by equating coefficients in
$E'(z)=E(z)\sum_{r\geq1}(-1)^{r-1}p_rz^{r-1}$, where
$E(z)=\prod_j(1+w_jz)=\sum_m e_mz^m$, interpreted as a formal power series.
Similarly, expanding $E(z)/(1+w_jz)$ gives
$e_{k-1}(w_{-j})=\sum_{r=0}^{k-1}(-1)^rw_j^re_{k-1-r}(w)$.
Using the inclusion probabilities in
Eq.~\ref{eq:app-symmetric-inclusion}, the final pointwise map evaluates
\begin{equation}
    f_k^*(D,x)
    =\frac{\sum_{r=1}^k(-1)^{r-1}e_{k-r}(w)t_r}{e_k(w)}.
    \label{eq:app-pooled-predictor}
\end{equation}
Thus the prediction is a continuous function of the $2k$ pooled channels.
Since $w_j\geq1$, its denominator obeys $e_k(w)\geq\binom{d}{k}>0$.

All exact intermediate states range over compact sets. Products,
exponentials, and powers can therefore be approximated uniformly by the
pointwise networks. The final rational map is continuous on a compact
neighborhood of the attainable summaries where its denominator stays
positive, so it too admits uniform approximation. Choosing successive
approximation errors sufficiently small proves
Eq.~\ref{eq:app-one-sparse-approximation}. Residual channels retain the
inputs needed at each stage; token-type and row-role indicators allow
shared pointwise maps to implement the required role-specific operations.
The first routing operation can select the response token exactly with the
allowed type mask. The final pointwise output is stored in one channel for
the linear readout. Only $O(k)$ state channels are needed; the pointwise
hidden widths are unrestricted.

For $k=1$, a more direct construction identifies the final attention
weights with the posterior gate. After the observation-axis stage, form
$s_j=\theta_1c_j^2$ and $v_j=\lambda_1c_jx_j$. Third, a prediction readout state
applies feature-axis attention with logits $s_j$ and values $v_j$, producing
\begin{equation}
    \sum_{j=1}^d
    \frac{e^{\theta_1c_j^2}}{\sum_{r=1}^d e^{\theta_1c_r^2}}
    \lambda_1c_jx_j
    =f_1^*(D,x).
\end{equation}
Approximating the products and squares gives the same uniform-approximation
conclusion. This special construction is illustrated in
Figure~\ref{fig:arch-and-construction}c.
\end{proof}

\textbf{Scope of the construction.}
The three attention stages belong to the idealized class $\mathcal A_{n,d}$,
not necessarily to three blocks of the trained model. The construction
uses type masks and flexible pointwise maps and does not establish
realization by its particular normalization, widths, or block layout.
For general $k$, the final attention pools uniformly; the nonlinear
readout, rather than its attention weights, implements the inclusion-weighted
prediction. Consequently, the result gives a feature-indexed realization
but neither identifies attention weights with relevance for all $k$ nor
proves a separation from row-token networks. Newton's recursion takes
$O(k^2)$ arithmetic operations after pooling, but its alternating sums can
be numerically ill-conditioned. The proof is an approximation result on
compact domains, not a numerical-stability or learning-efficiency guarantee.

\textbf{General support sizes.}
For $w_j=\exp(\theta_kc_j^2)$, define the elementary symmetric polynomials
\begin{equation}
    e_m(w)=\sum_{\substack{A\subseteq[d]\\|A|=m}}\prod_{r\in A}w_r,
    \qquad e_0(w)=1.
\end{equation}
Set $e_m=0$ when $m<0$ or $m$ exceeds the number of available coordinates.
The support normalizer is $e_k(w)>0$. Summing over supports containing $j$
gives
\begin{equation}
    q_{k,j}(c)=\frac{w_j e_{k-1}(w_{-j})}{e_k(w)}.
    \label{eq:app-symmetric-inclusion}
\end{equation}

Define prefix and suffix tables
\begin{equation}
    F_{j,m}=e_m(w_1,\ldots,w_j),
    \qquad G_{j,m}=e_m(w_j,\ldots,w_d),
    \qquad 0\le m\le k.
\end{equation}
Initialize $F_{0,0}=G_{d+1,0}=1$ and
$F_{0,m}=G_{d+1,m}=0$ for $m>0$, with all negative-degree entries zero.
Partitioning subsets by inclusion of their boundary coordinate yields
\begin{equation}
    F_{j,m}=F_{j-1,m}+w_jF_{j-1,m-1},
    \qquad
    G_{j,m}=G_{j+1,m}+w_jG_{j+1,m-1}.
    \label{eq:app-subset-recursion}
\end{equation}
Compute $F$ in increasing $j$ and $G$ in decreasing $j$. Then
\begin{equation}
    e_k(w)=F_{d,k},
    \qquad
    e_{k-1}(w_{-j})
    =\sum_{m=0}^{k-1}F_{j-1,m}G_{j+1,k-1-m}.
    \label{eq:app-subset-exclusion}
\end{equation}
Each table takes $O(dk)$ arithmetic operations and storage. Evaluating the
$k$-term sum for every $j$ takes another $O(dk)$ operations. Together with
Eq.~\ref{eq:app-symmetric-inclusion}, this computes all inclusion
probabilities. Forming $c$ takes $O(nd)$ operations, and forming and summing
$\lambda_kq_{k,j}c_jx_j$ takes $O(d)$.

\textbf{Endpoints.}
For $k=1$, $e_1(w)=\sum_jw_j$ and $e_0(w_{-j})=1$, recovering the feature
softmax. For $k=d$, $e_d(w)=\prod_jw_j$, so all gates equal one and the
subset computation can be omitted.

\textbf{Log-domain evaluation.}
Store $\ell_j=\theta_kc_j^2$, $L^F_{j,m}=\log F_{j,m}$, and
$L^G_{j,m}=\log G_{j,m}$. Represent zero entries by $-\infty$ and entries
equal to one by zero. With
$\operatorname{LSE}(a,b)=\log(e^a+e^b)$ evaluated by the usual maximum
subtraction, the recurrences become
\begin{equation}
    \begin{aligned}
        L^F_{j,m}&=\operatorname{LSE}
        (L^F_{j-1,m},\ell_j+L^F_{j-1,m-1}),\\
        L^G_{j,m}&=\operatorname{LSE}
        (L^G_{j+1,m},\ell_j+L^G_{j+1,m-1}).
    \end{aligned}
\end{equation}
Define $\operatorname{LSE}$ of all $-\infty$ inputs to be $-\infty$.
The final log inclusion probability is
\begin{equation}
    \log q_{k,j}
    =\ell_j+\operatorname{LSE}_{0\le m<k}
    \big(L^F_{j-1,m}+L^G_{j+1,k-1-m}\big)-L^F_{d,k}.
\end{equation}
Only the normalized log probabilities need to be exponentiated. A common
shift of all $\ell_j$ also leaves the probabilities unchanged, since every
support has exactly $k$ entries. This implementation retains the $O(dk)$
arithmetic and storage bounds.

\subsection{Proof of Lemma~\ref{lem:coefficient-accounting} and excess prediction risk}
\label{app:coefficient-accounting}

\begin{proof}
Fix a context $D$ for which $f(D,\cdot)$ is square-integrable.
Since $\mathbb{E}_x[x]=0$ and $\mathbb{E}_x[xx^\top]=I_d$, the
constant function $1$ and query-coordinate functions $x\mapsto x_j$,
$j\in[d]$, form an orthonormal family in $L^2$.
Thus the affine projection in Lemma~\ref{lem:coefficient-accounting}
has the unique coefficients
\begin{equation}
    \zeta_f(D)=\mathbb{E}_x[f(D,x)],
    \qquad \eta_f(D)=\mathbb{E}_x[xf(D,x)],
    \label{eq:query-linear-projection}
\end{equation}
and its residual satisfies
\begin{equation}
    \mathbb{E}_x[r_f(D,x)]=0,
    \qquad \mathbb{E}_x[xr_f(D,x)]=0.
    \label{eq:app-projection-orthogonality}
\end{equation}
Using $f_k^*(D,x)=x^\top\eta_k^*(c)$, write
\[
    f(D,x)-f_k^*(D,x)
    =\zeta_f(D)+x^\top(\eta_f(D)-\eta_k^*(c))+r_f(D,x).
\]
The three summands are mutually orthogonal by query centering and
Eq.~\ref{eq:app-projection-orthogonality}. Squaring and averaging over
$x$, with $\mathbb{E}_x[xx^\top]=I_d$, gives
Eq.~\ref{eq:bayes-error-decomposition}.
\end{proof}

\textbf{Connection to excess prediction risk.}
Define the population squared risk under $P_k$ by
\begin{equation}
    R_k(f):=\mathbb{E}_{D,x,y_q}[(f(D,x)-y_q)^2],
    \label{eq:iv-risk-definitions}
\end{equation}
with $y_q$ as in Section~\ref{sec:controlled-priors}.
Since $f_k^*(D,x)=\mathbb{E}[y_q\mid D,x]$, conditional-expectation
orthogonality and Eq.~\ref{eq:bayes-error-decomposition} give
\begin{equation}
    R_k(f)-R_k(f_k^*)
    =\mathbb{E}_{D,x}[(f(D,x)-f_k^*(D,x))^2]
    =\mathbb{E}_D[C_f(D)+\zeta_f(D)^2+N_f(D)],
\end{equation}
where $C_f(D)=\|\eta_f(D)-\eta_k^*(c)\|_2^2$ and
$N_f(D)=\mathbb{E}_x[r_f(D,x)^2]$.

\begingroup
\section{Controlled experiments: setup and extended results}
\label{app:iv-details}

This appendix gives the experimental settings and complete results for
Section~\ref{sec:controlled-experiments}. We first describe training and
evaluation, then report the architecture comparison and its coefficient-space
decomposition. Population identities and proofs appear in
Appendix~\ref{app:coefficient-accounting}.

\subsection{Controlled architectures}
\label{app:network-architecture}

\begin{itemize}[leftmargin=*,topsep=0pt,parsep=0pt,itemsep=0pt]
    \item The \underbar{alternating-axis model} follows nanoTabPFN \citep{pfefferle2025nanotabpfn}, a compact reimplementation of TabPFN v2 \citep{hollmann2025tabpfn}.
    It embeds each feature cell separately, places the response in an additional token, and alternates feature- and observation-axis attention within each block.
    We adapt it to regression by giving its two-layer MLP decoder a scalar output and training with squared loss, and we read predictions from the query response tokens.
    \item The \underbar{row-token model} replaces the shared scalar cell encoder with a linear projection of the entire covariate vector, $\mathbb R^d\to\mathbb R^E$, adds a linear response embedding to each context row token, and removes feature-axis attention and its associated normalization.
    Query rows receive no response embedding, and predictions are read from their row tokens.
    We retain the observation-axis attention mask, post-norm feed-forward blocks, and scalar MLP decoder.
    This construction follows TabPFN v1's linear-encoder, post-norm configuration \citep{mueller2022tabpfn} and shares the linear row and response encoders and additive label injection of the original TabDPT architecture \citep{ma2025tabdpt}, although TabDPT uses different normalization.
\end{itemize}

\subsection{Experimental setup and evaluation protocol}
\label{app:iv-protocol}

\subsubsection{Tasks and model configurations}
\label{app:iv-configurations}

\textbf{Tasks.}
We use the priors in Section~\ref{sec:controlled-priors} with
$d\in\{10,20\}$, $k\in\{1,d/2,d\}$, $v^2=1$, and $\sigma=1.5$.
Each condition contains $T=32{,}000$ test tasks (data-generation seed 1), generated
in 1,000 batches of 32. Context length $n_t$ is shared within a batch
and ranges from 50 to 100, with $X_t^\top X_t=n_tI_d$.
Each task has $Q_{\mathrm{eval}}^{(t)}=150-n_t$ independent queries
$x_{tq}\sim\mathcal N(0,I_d)$.
All architectures and training seeds use
the same test contexts and queries. The Bayes target
$\eta_{k,n_t}^*(c_t)$ uses the task's actual context length.
Changing $d$ also changes $d/n_t$ and the per-coordinate prior variance;
the two dimensions test recurrence of the pattern, not an isolated dimension effect.

\textbf{Models and inference.}
We use the architectures defined in Section~\ref{sec:controlled-priors} and Appendix~\ref{app:network-architecture}.
Both have embedding width 96, feed-forward width 192, four attention heads,
GELU activations, post-normalization, and a two-layer scalar MLP decoder.
The row-token model has five blocks and 393,985/394,945 parameters at
$d=10/20$; the alternating-axis model has three blocks and 355,873 parameters
at either dimension. Counts include the encoders and decoder. This
approximately matches parameter scale; token counts and computation differ.
Features are standardized using context statistics and clipped to
$[-100,100]$. The alternating-axis model uses one token per feature plus
a response token, with unknown query responses initialized to the mean
context response. Observation-axis attention permits context self-attention
and query-to-context attention only. Inference preserves feature order and response units,
without feature augmentation, target scaling, or prediction ensembling.

\subsubsection{Hyperparameter selection and final training}
\label{app:iv-training}

\textbf{Shared search protocol.}
For each $(d,k)$, both architectures use schedule-free AdamW and the same
learning-rate grid
$\{5\times10^{-4},10^{-3},2\times10^{-3},4\times10^{-3}\}$.
Both architectures share one tuning dataset per $(d,k)$, with separate
training, validation, and hyperparameter-selection sets. Data-generation
seed 0 fixes these task sets. For each candidate learning rate, we train
three models with training seeds $\{0,1,2\}$, which vary model
initialization while keeping the task sets fixed. Each candidate's best
validation checkpoint is scored on the same 32,000 selection tasks using MSE pooled over noisy
query targets. We average this score over the three training runs and
select the learning rate with the lowest mean, separately for each
architecture and $(d,k)$.
Table~\ref{tab:iv-affine-training} reports the selected learning rates.

\textbf{Training and early stopping.}
The same training and stopping rules apply during hyperparameter search
and final training. All runs use batch size 32, zero weight decay, and
gradient clipping at norm 1. Maximum budgets are 10,000 steps for $d=10$
and 20,000 for $d=20$. Every 500 steps, we evaluate MSE against noisy
query labels on all 3,200 validation tasks. Training stops after five
consecutive checks without a new lowest validation MSE, or when the step
budget is exhausted. In either case, we restore the checkpoint with the
lowest validation MSE.

\textbf{Final training.}
With these hyperparameters fixed, we retrain each model from random
initialization on a newly generated training set from the same prior
(data-generation seed 1), using the same three training seeds and training
protocol. The training sets contain 320,000 tasks for $d=10$ and 640,000
for $d=20$. A separate 3,200-task validation set selects the final
checkpoint; its step is reported in Table~\ref{tab:iv-affine-training}.
The final 32,000-task test set is used only for evaluation.

\begin{table}[!htbp]
\centering
\small
\setlength{\tabcolsep}{3.5pt}
\caption{Selected learning rates for schedule-free AdamW and checkpoints from final training. Maximum step budgets are fixed by dimension; checkpoint steps are chosen by validation MSE and listed in training-seed order 0/1/2.}
\label{tab:iv-affine-training}
\begin{tabular}{@{}rrlrrl@{}}
\toprule
$d$ & $k$ & Model & Selected LR & Max. steps & Checkpoint step (0/1/2) \\
\midrule
10 & 1 & Row-token & 0.0010 & 10000 & 10000 / 10000 / 10000 \\
10 & 1 & Alternating-axis & 0.0010 & 10000 & 9500 / 9500 / 9500 \\
10 & 5 & Row-token & 0.0010 & 10000 & 9500 / 10000 / 10000 \\
10 & 5 & Alternating-axis & 0.0010 & 10000 & 9500 / 10000 / 10000 \\
10 & 10 & Row-token & 0.0010 & 10000 & 10000 / 10000 / 7000 \\
10 & 10 & Alternating-axis & 0.0010 & 10000 & 10000 / 10000 / 6000 \\
20 & 1 & Row-token & 0.0020 & 20000 & 19000 / 20000 / 20000 \\
20 & 1 & Alternating-axis & 0.0005 & 20000 & 20000 / 20000 / 20000 \\
20 & 10 & Row-token & 0.0010 & 20000 & 18500 / 12500 / 19000 \\
20 & 10 & Alternating-axis & 0.0010 & 20000 & 20000 / 10500 / 20000 \\
20 & 20 & Row-token & 0.0010 & 20000 & 12500 / 12500 / 12500 \\
20 & 20 & Alternating-axis & 0.0005 & 20000 & 20000 / 18500 / 19500 \\
\bottomrule
\end{tabular}
\end{table}

\subsubsection{Evaluation}
\label{app:iv-metrics}

All quantities below are defined for fixed $d,k$ and one training seed;
the seed index is suppressed. For test dataset $D_t$, we estimate the
normalized Bayes approximation error in Eq.~\ref{eq:iv-relative-error} by
\begin{equation}
    \widehat{\mathcal E}_{a,k}^{(t)}
    :=\frac{\displaystyle\frac{1}{Q_{\mathrm{eval}}^{(t)}}
    \sum_{q=1}^{Q_{\mathrm{eval}}^{(t)}}
    [\widehat f_{a,k}(D_t,x_{tq})-f_k^*(D_t,x_{tq})]^2}
    {s_k^{(t)}}.
    \label{eq:iv-error-estimate}
\end{equation}
The denominator $s_k^{(t)}$ is the dataset's Bayes prediction energy.
All evaluated datasets have $s_k^{(t)}>10^{-12}$ and are retained without a denominator offset. 

\subsection{Extended architecture--prior comparison}
\label{app:iv-performance}

\subsubsection{Dense-prior reference}
\label{app:iv-dense-reference}

For a dataset $D$ drawn under $P_k$, we define the dense-prior prediction
gap as
\begin{equation}
    \begin{aligned}
        \Delta_k^{\mathrm{dense}}(D)
        &:=\mathbb E_x\!\left[(f_d^*(D,x)-f_k^*(D,x))^2\right]\\
        &=\|\lambda_dc-\eta_k^*(c)\|_2^2,
        \qquad c=c(D).
    \end{aligned}
    \label{eq:iv-dense-reference-error}
\end{equation}
The equality follows from query isotropy. This gap measures the prediction
cost of applying uniform dense-prior shrinkage instead of the Bayes rule
under $P_k$. Its expectation is strictly positive for $k<d$ and zero for
$k=d$ (Theorem~\ref{thm:fixed-kernel-gap}).

Figure~\ref{fig:app-iv-reference-and-architecture-gaps}(a) shows that exploiting the sparse
prior reduces expected squared prediction error relative to the dense-prior
rule. For both dimensions, this gain is largest at $k=1$, decreases at
$k=d/2$, and vanishes at $k=d$, where the two predictors coincide.

\subsubsection{Learned predictors}

For each test dataset, the paired architecture gap is
\begin{equation}
    \widehat G_k^{(t)}
    :=\widehat{\mathcal E}_{\text{row-token},k}^{(t)}
    -\widehat{\mathcal E}_{\text{alternating-axis},k}^{(t)}.
    \label{eq:iv-architecture-gap}
\end{equation}
Positive values favor the alternating-axis model.
Figure~\ref{fig:app-iv-reference-and-architecture-gaps}(b) shows a large
mean advantage under sparse priors and a small positive mean advantage
at the dense endpoints for all three seeds.

\begin{figure}[!tbp]
    \centering
    \input{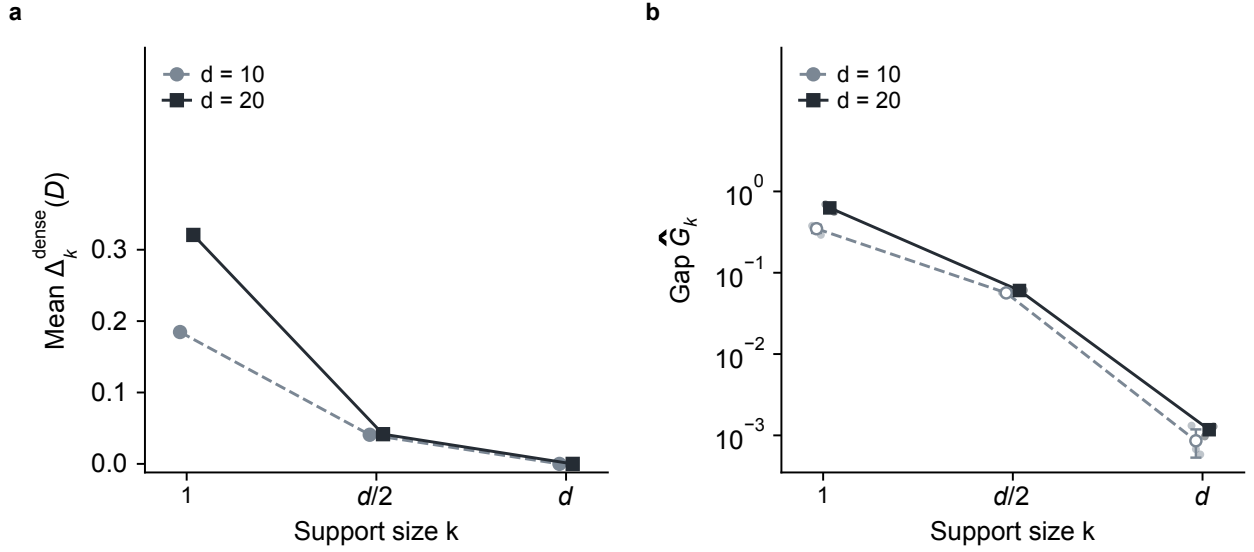}
    \caption{\textbf{Dense-prior and architecture gaps across support sizes.}
    (a) Mean dense-prior prediction gap $\Delta_k^{\mathrm{dense}}(D)$
    (Eq.~\ref{eq:iv-dense-reference-error}) over 32,000 test datasets per
    condition. The dense endpoint is exactly zero.
    (b) Per-dataset architecture gaps $\widehat G_k^{(t)}$
    (Eq.~\ref{eq:iv-architecture-gap}), averaged equally over the same
    32,000 test datasets for each seed. Positive values favor the
    alternating-axis model. Large markers show means over three seeds,
    small points show individual seed averages, and error bars show seed
    SDs (ddof=0). The axis label omits the dataset index.
    Circles/dashed lines denote $d=10$; squares/solid lines denote $d=20$.
    Panel (a) uses unnormalized errors on a linear axis; panel (b) uses
    task-normalized errors on a log axis. Support sizes are categorical.}
    \label{fig:app-iv-reference-and-architecture-gaps}
\end{figure}

\textbf{Task-level architecture gaps.}
Figure~\ref{fig:app-iv-task-gap-distributions} shows the distribution of
per-dataset architecture gaps, both after averaging over training seeds
and for individual seeds. The alternating-axis model outperforms
the row-token model on nearly all sparse datasets. Under dense priors,
its average advantage is small, and the better-performing architecture
varies across datasets.

\begin{figure}[!tbp]
    \centering
    \input{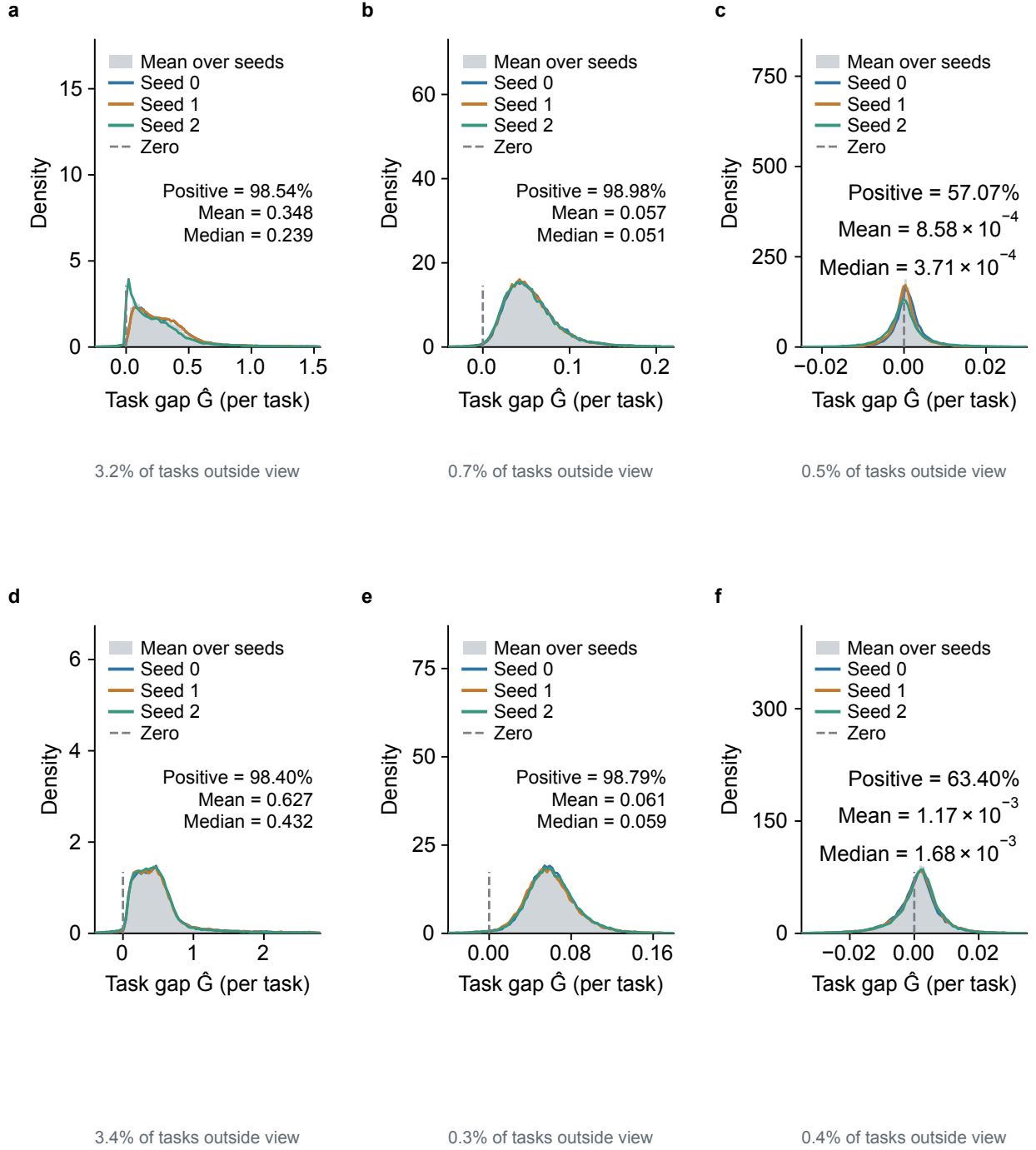}
    \caption{\textbf{Task-level distributions of the architecture gap.}
    Rows correspond to $d=10$ and $d=20$; columns to
    $k=1$, $k=d/2$, and $k=d$.
    Positive gaps favor the alternating-axis model.
    Gray shading shows the distribution of per-dataset gaps
    $\widehat G_k^{(t)}$ after averaging over three training seeds;
    colored curves show the distributions for individual seeds.
    All distributions use the same histogram bins within each panel;
    the dashed line marks zero.
    Annotations report the positive fraction, mean, and median of the
    seed-averaged gaps, computed over all 32,000 tasks per condition.
    Axis ranges focus on the distribution bodies; the fraction of
    seed-averaged gaps outside each view is noted below the panel.
    Densities are normalized by the full task count.}
    \label{fig:app-iv-task-gap-distributions}
\end{figure}

\subsection{Coefficient-space accounting of the architecture gap}
\label{app:iv-accounting}

\subsubsection{Affine projection and error estimation}
\label{app:iv-projection}

\textbf{Estimating the affine component.}
For each context $D$, we estimate the affine projection of $f$
(Lemma~\ref{lem:coefficient-accounting}) using
$Q_{\mathrm{proj}}=1{,}000$ independent Gaussian queries
$x_m^{\mathrm{proj}}\sim\mathcal N(0,I_d)$:
\begin{equation}
    (\widehat\zeta_f(D),\widehat\eta_f(D))
    :=\arg\min_{\zeta\in\mathbb R,\,\eta\in\mathbb R^d}
    \frac{1}{Q_{\mathrm{proj}}}\sum_{m=1}^{Q_{\mathrm{proj}}}
    [f(D,x_m^{\mathrm{proj}})-\zeta-(x_m^{\mathrm{proj}})^\top\eta]^2.
    \label{eq:iv-affine-blp}
\end{equation}
Projection queries are independent of the context and evaluation queries,
and shared across models and intervention conditions. Fit targets are
model predictions in original response units, with preprocessing and model
randomness fixed; query labels are not used.

We solve the unregularized least-squares problem, including an intercept,
in float64 with \texttt{numpy.linalg.lstsq} (\texttt{rcond=None}).

\textbf{Error components.}
Using the fitted coefficients and independent evaluation queries, define
\begin{equation}
    \begin{aligned}
        \widehat C_f^{(t)}
        &:=\|\widehat\eta_f(D_t)-\eta_{k,n_t}^*(c_t)\|_2^2,\\
        \widehat N_f^{(t)}
        &:=\frac{1}{Q_{\mathrm{eval}}^{(t)}}
        \sum_{q=1}^{Q_{\mathrm{eval}}^{(t)}}\widehat r_{f,tq}^2,
        \qquad
        \widehat r_{f,tq}:=f(D_t,x_{tq})-\widehat\zeta_f(D_t)
        -x_{tq}^\top\widehat\eta_f(D_t).
    \end{aligned}
    \label{eq:iv-estimated-components}
\end{equation}
These give coefficient error $\widehat C_f^{(t)}$, squared-intercept
error $\widehat\zeta_f(D_t)^2$, and fitted residual error
$\widehat N_f^{(t)}$. For $f=\widehat f_{a,k}$, define the normalized
components on dataset $D_t$ as
\begin{equation}
    \widehat C_{a,k}^{(t)}:=\frac{\widehat C_f^{(t)}}{s_k^{(t)}},\qquad
    \widehat Z_{a,k}^{(t)}:=\frac{\widehat\zeta_f(D_t)^2}{s_k^{(t)}},\qquad
    \widehat N_{a,k}^{(t)}:=\frac{\widehat N_f^{(t)}}{s_k^{(t)}}.
    \label{eq:iv-normalized-components}
\end{equation}
All three use the same denominator as $\widehat{\mathcal E}_{a,k}^{(t)}$.

\subsubsection{Full component results across dimensions}
\label{app:iv-component-results}

\textbf{Error components.}
Figures~\ref{fig:iv-affine-levels-d10}
and~\ref{fig:app-iv-affine-levels-d20} show the component errors at
$d=10$ and $d=20$, respectively. Under sparse priors, the alternating-axis
model has substantially lower coefficient error, whereas intercept and
residual errors are closer between architectures. At $d=20,k=1$, coefficient
error is $0.685$ for the row-token model and $0.0603$ for the alternating-axis
model. At the dense endpoint, the row-token model has slightly lower
coefficient error but higher intercept and residual errors in both dimensions.

\begin{figure}[!tbp]
    \centering
    \input{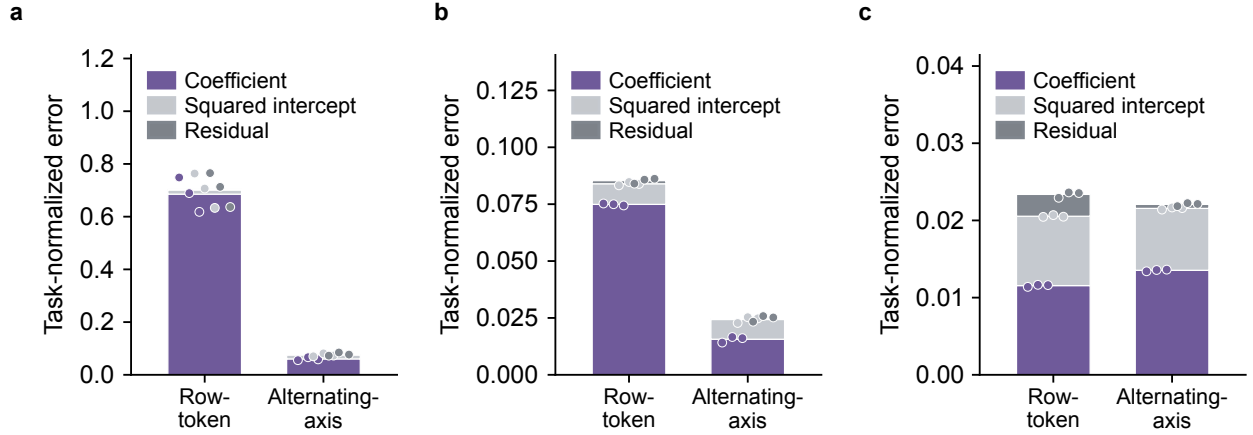}
    \caption{\textbf{Task-normalized error components at $\boldsymbol{d=20}$.}
    Bars stack coefficient, squared-intercept, and residual errors for
    row-token and alternating-axis predictors at $k=1,10,20$.
    For each seed, components are averaged equally over the same 32,000
    test datasets. Bars show means over three seeds; points mark each
    seed's cumulative component means at the stack boundaries. Colors match
    Figure~\ref{fig:iv-affine-levels-d10}.}
    \label{fig:app-iv-affine-levels-d20}
\end{figure}

\textbf{Architecture gaps.}
To locate the prediction gap $\widehat G_k^{(t)}$
(Eq.~\ref{eq:iv-architecture-gap}), we compare the normalized components
on the same dataset:
\begin{equation}
    \begin{aligned}
        \widehat G_k^{C,(t)}&:=\widehat C_{\text{row-token},k}^{(t)}
        -\widehat C_{\text{alternating-axis},k}^{(t)},\\
        \widehat G_k^{\zeta,(t)}&:=\widehat Z_{\text{row-token},k}^{(t)}
        -\widehat Z_{\text{alternating-axis},k}^{(t)},\\
        \widehat G_k^{N,(t)}&:=\widehat N_{\text{row-token},k}^{(t)}
        -\widehat N_{\text{alternating-axis},k}^{(t)}.
    \end{aligned}
    \label{eq:iv-component-gap-definitions}
\end{equation}
Positive gaps favor the alternating-axis model within that component.
The superscript $\zeta$ denotes a difference in \emph{squared} intercepts.

Figure~\ref{fig:app-iv-affine-accounting} shows that coefficient error
dominates the sparse architecture gap. At $d=10$, the between-architecture coefficient-error difference
accounts for 98.4\% and 97.5\% of the mean prediction-error gap at $k=1$ and $k=5$.
At $d=20$, the corresponding shares are 99.7\% and 98.0\% at $k=1$ and
$k=10$. At both dense endpoints, the negative coefficient gap is offset
by positive intercept and residual gaps, leaving a small positive
prediction gap.

\begin{figure}[!tbp]
    \centering
    \input{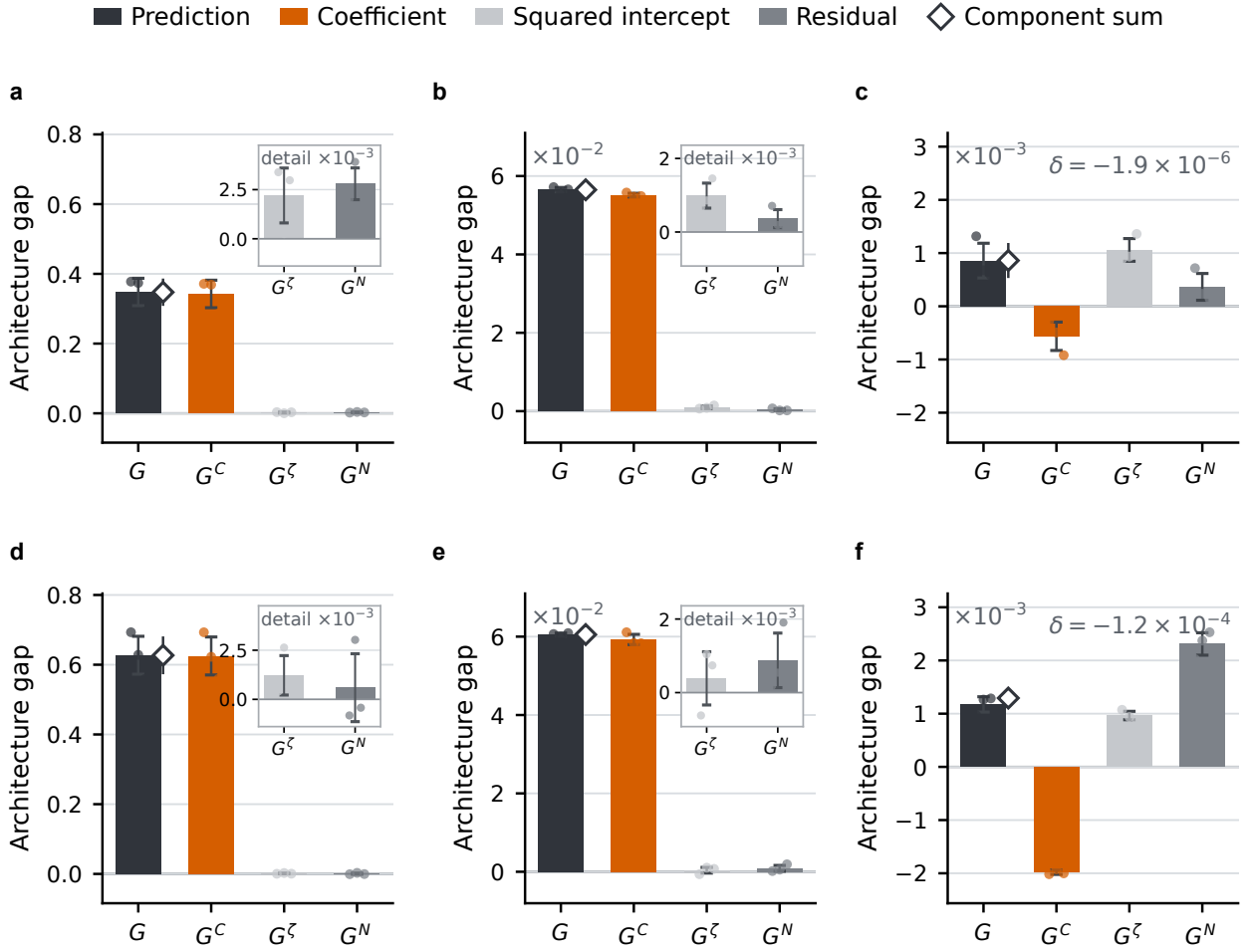}
    \caption{\textbf{The sparse architecture gap lies mainly in coefficients.}
    Rows correspond to $d=10,20$; columns to $k=1,d/2,d$.
    Bars show row-token minus alternating-axis gaps in
    task-normalized prediction, coefficient, squared-intercept, and residual
    error. For each seed, gaps are averaged equally over the same 32,000
    test datasets. Bars show means over three seeds; points show individual
    seed averages; error bars show seed SDs (ddof=0). Positive values favor
    the alternating-axis model. Hollow diamonds show component
    sums; insets magnify the two small non-coefficient gaps.
    Main panels share vertical ranges within each column. Labels
    $G$, $G^C$, $G^\zeta$, and $G^N$ suppress hats and the indices $k,t$
    (Eqs.~\ref{eq:iv-architecture-gap} and~\ref{eq:iv-component-gap-definitions}).}
    \label{fig:app-iv-affine-accounting}
\end{figure}

\subsubsection{Active and inactive coordinate contributions}
\label{app:iv-support}

\textbf{Coefficient errors by support group.}
For the realized support $S_t$, we split each model's coefficient error
into active and inactive contributions:
\begin{equation}
    \begin{aligned}
        \widehat C_{f,\mathrm{active}}^{(t)}
        &:=\sum_{j\in S_t}
        (\widehat\eta_{f,j}(D_t)-\eta_{k,n_t,j}^*(c_t))^2,\\
        \widehat C_{f,\mathrm{inactive}}^{(t)}
        &:=\sum_{j\notin S_t}
        (\widehat\eta_{f,j}(D_t)-\eta_{k,n_t,j}^*(c_t))^2.
    \end{aligned}
    \label{eq:iv-support-components}
\end{equation}
Both terms measure squared deviations from the Bayes coefficient vector
and sum to $\widehat C_f^{(t)}$. For architecture $a$ trained under $P_k$,
let $f=\widehat f_{a,k}$ and define the task-normalized group errors as
\begin{equation}
    \widehat C_{a,k,b}^{(t)}:=\frac{\widehat C_{f,b}^{(t)}}{s_k^{(t)}},
    \qquad b\in\{\mathrm{active},\mathrm{inactive}\}.
    \label{eq:iv-support-levels}
\end{equation}
Figure~\ref{fig:app-iv-support-levels} decomposes each model's coefficient
error using the same normalization as Figure~\ref{fig:iv-affine-levels-d10}.
Under sparse priors, the alternating-axis model has lower error in both
groups. At $k=d$, every coordinate is active and the inactive error is zero.

\begin{figure}[!tbp]
    \centering
    \input{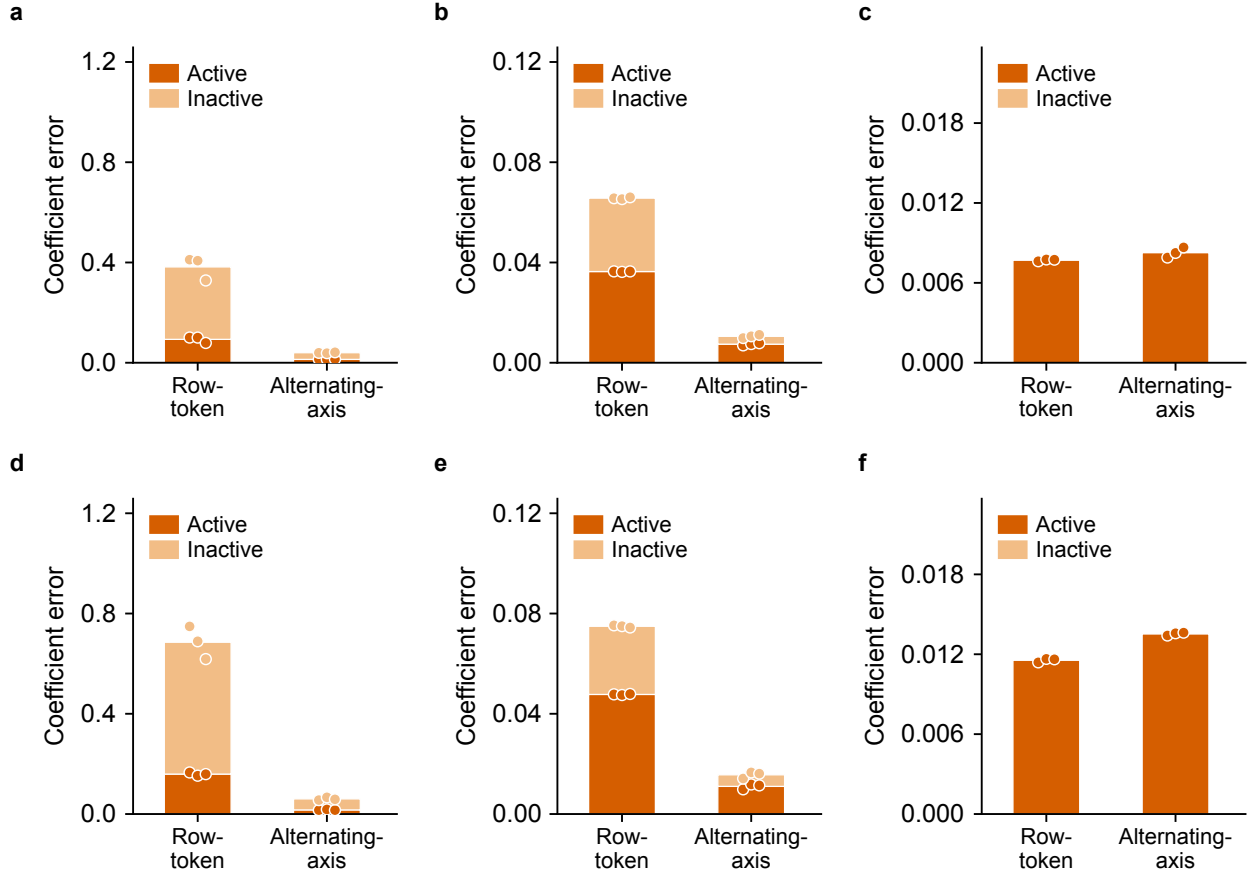}
    \caption{\textbf{Active and inactive coefficient errors by architecture.}
    Stacked bars show task-normalized group errors
    (Eq.~\ref{eq:iv-support-levels}); each bar sums to the model's
    coefficient error. Rows correspond to $d=10,20$ and columns to
    $k=1,d/2,d$. For each seed, group errors are averaged equally over
    the same 32,000 test datasets. Bars show means over three seeds;
    points mark each seed's mean active and total coefficient errors
    at the corresponding stack boundaries.
    At $k=d$, the inactive contribution
    is zero. Vertical ranges are shared across dimensions within each column.}
    \label{fig:app-iv-support-levels}
\end{figure}

\textbf{Coefficient errors per coordinate.}
For sparse priors, we divide $\widehat C_{a,k,\mathrm{active}}^{(t)}$ by $k$
and $\widehat C_{a,k,\mathrm{inactive}}^{(t)}$ by $d-k$ to obtain the mean
error per coordinate in each group.
Figure~\ref{fig:app-iv-support-per-coordinate} shows that active coordinates
have higher per-coordinate error for both architectures across the sparse
conditions. At $k=1$, total error is higher in the larger inactive group.
The alternating-axis model has lower per-coordinate
error in both groups.

\begin{figure}[!tbp]
    \centering
    \input{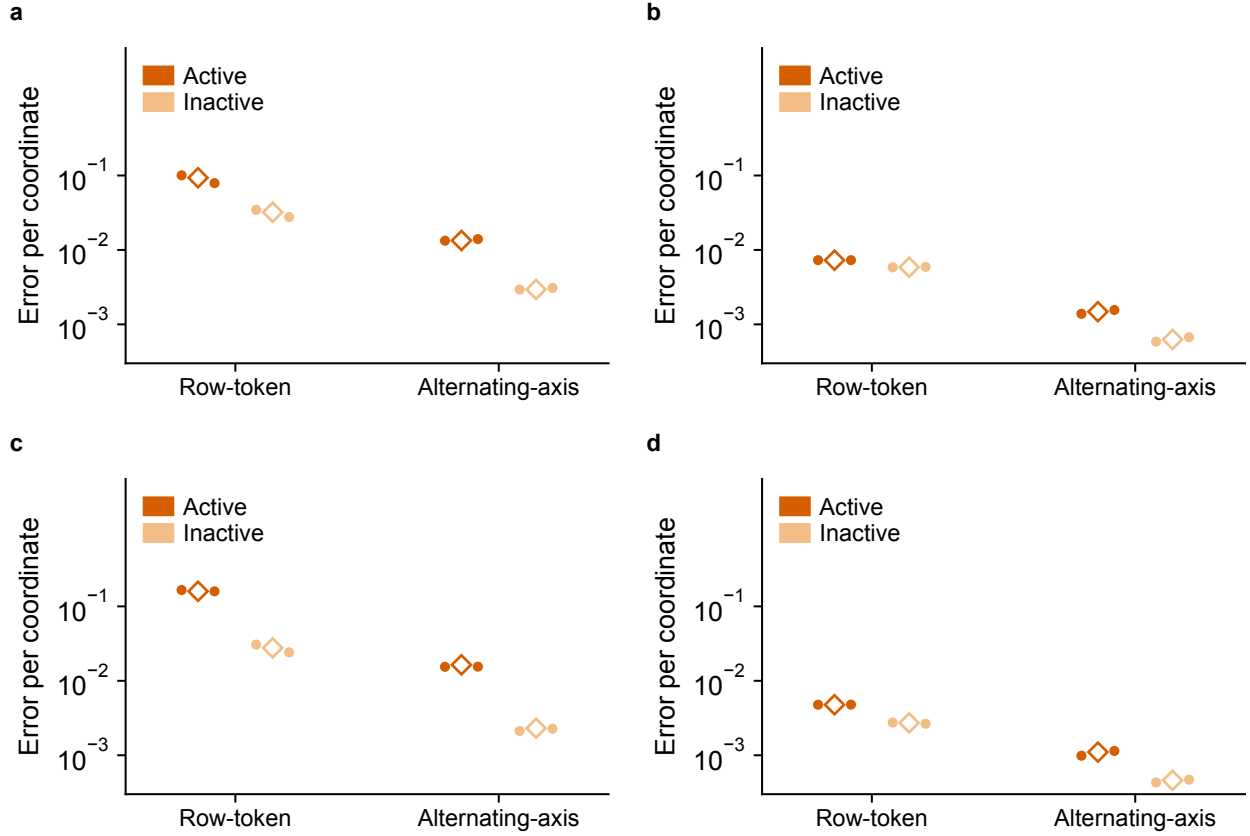}
    \caption{\textbf{Coefficient errors per coordinate under sparse priors.}
    Each model's task-normalized active and inactive errors
    (Eq.~\ref{eq:iv-support-levels}) are divided by $k$ and $d-k$,
    respectively. Rows correspond to $d=10,20$ and columns to $k=1,d/2$.
    Colors match Figure~\ref{fig:app-iv-support-levels}.
    Per-coordinate errors are averaged equally over the same 32,000
    test datasets for each seed. Diamonds show means over three seeds,
    circles show individual seed averages, and error bars show seed SDs
    (ddof=0). All panels share a logarithmic vertical axis.}
    \label{fig:app-iv-support-per-coordinate}
\end{figure}

\subsubsection{Ordinary least squares as a coefficient baseline}
\label{app:iv-ols}

To distinguish recovering the context statistic from computing the Bayes
coefficient map, we add ordinary least squares (OLS) as a baseline.
Under the orthogonal design, its coefficients and prediction are
\begin{equation}
    c_t=(X_t^\top X_t)^{-1}X_t^\top y_t
       =\frac{X_t^\top y_t}{n_t},
    \qquad f_{\mathrm{OLS}}(D_t,x)=x^\top c_t.
    \label{eq:iv-ols-predictor}
\end{equation}
We compute $c_t$ directly from each of the same 32,000 test contexts.
Its affine projection has
coefficient $c_t$, zero intercept, and zero nonlinear residual. Thus its
task-normalized coefficient error is
\begin{equation}
    C_{\mathrm{OLS},k}^{(t)}
    :=\frac{\|c_t-\eta_{k,n_t}^*(c_t)\|_2^2}{s_k^{(t)}},
    \qquad s_k^{(t)}=\|\eta_{k,n_t}^*(c_t)\|_2^2.
    \label{eq:iv-ols-coefficient-error}
\end{equation}
This is also its population relative prediction error under
$x\sim\mathcal N(0,I_d)$. We use the same support partition and equal-task
averaging as in Eq.~\ref{eq:iv-support-levels}.

\begin{figure}[!tbp]
    \centering
    \includegraphics[width=\linewidth]{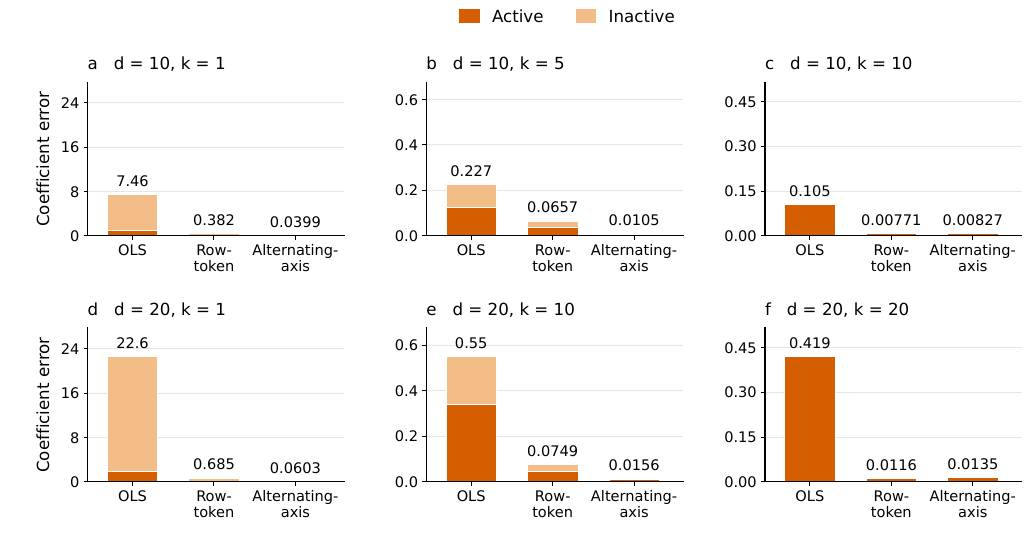}
    \caption{\textbf{OLS and learned-model coefficient errors relative to Bayes.}
    Rows correspond to $d=10,20$ and columns to $k=1,d/2,d$.
    Bars stack task-normalized active and inactive coefficient errors;
    annotations give their sum. Every bar uses the same 32,000 test
    datasets per condition. OLS is deterministic given the context;
    learned-model bars average three training seeds and reproduce
    Figure~\ref{fig:app-iv-support-levels}. At $k=d$, the inactive sum is zero.
    Linear vertical ranges are shared within each column.}
    \label{fig:app-iv-ols-comparison}
\end{figure}

OLS has higher mean coefficient error than either trained architecture
in all six conditions (Figure~\ref{fig:app-iv-ols-comparison}). At $k=1$,
inactive coordinates contribute 86.0\% and 91.9\% of its mean error at
$d=10$ and $d=20$, respectively. Indeed,
$c_t=\beta+X_t^\top\varepsilon/n_t$ retains estimation noise on every
coordinate, with conditional variance $\sigma^2/n_t$ per coordinate.
The Bayes map instead applies both posterior inclusion weights and
shrinkage, $\eta_{k,n_t}^*(c_t)=\lambda_{k,n_t}q_k(c_t)\odot c_t$.
Both predictions are linear in the query; the difference lies in how
their coefficients depend on the context.

\paragraph{Coefficient shrinkage.}
Figure~\ref{fig:app-iv-ols-coefficients} shows the relation between OLS and
Bayes coefficients. Although OLS recovers the sufficient statistic $c_t$,
it does not implement the prior-dependent posterior coefficient map. For
sparse priors,
$\eta_{k,n_t}^*(c_t)=\lambda_{k,n_t}q_k(c_t)\odot c_t$ combines coordinate
selection with shrinkage: coordinates with weak posterior support are mapped
toward zero, while coordinates with stronger evidence retain a larger fraction
of their OLS coefficients. This separation is most pronounced at $k=1$ and
contracts as the support size increases. At $k=d$, posterior inclusion is
identically one and the map reduces to uniform shrinkage,
$\eta_{d,n_t}^*(c_t)=\lambda_{d,n_t}c_t$, yielding the exact identity
\begin{equation}
    C_{\mathrm{OLS},d}^{(t)}
    =\left(\frac{1-\lambda_{d,n_t}}{\lambda_{d,n_t}}\right)^2
    =\left(\frac{d\sigma^2}{n_tv^2}\right)^2.
    \label{eq:iv-ols-dense-check}
\end{equation}
Its task averages are $0.104774$ and $0.419098$ at $d=10,20$.

\begin{figure}[!tbp]
    \centering
    \includegraphics[width=.94\linewidth]{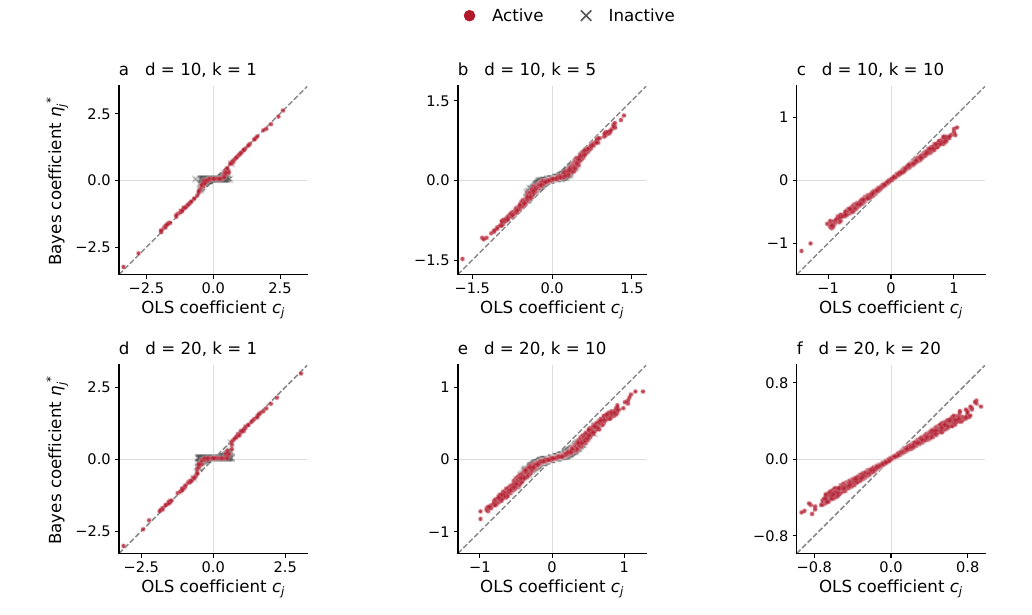}
    \caption{\textbf{OLS coefficients and posterior-mean coefficients.}
    Each panel plots all coordinates from 150 fixed, randomly sampled
    test tasks. Dark red circles denote active coordinates and dark gray crosses
    denote inactive coordinates in the realized support.
    Dashed lines indicate equal coefficients. At $k=d$, variation in
    context length produces different shrinkage slopes across tasks.
    Error summaries use the complete 32,000-task test set.}
    \label{fig:app-iv-ols-coefficients}
\end{figure}

\endgroup

\section{Feature-message interventions: methods and additional results}
\label{app:cross-scale}

This appendix supports Section~\ref{sec:cross-scale-gating}. We specify the
intervention and statistics, assess the coefficient description of prediction
changes, and compare context- and query-message effects. We then examine
response direction, attenuation strength, and support size. The primary
analysis uses $k=1$; Appendix~\ref{app:cross-scale-support-size} extends it to
$k=d/2$ and $k=d$.

\subsection{Intervention protocol and statistical aggregation}
\label{app:cross-scale-protocol}

\paragraph{Tasks and models.}
The results reported here use $k=1$ and $d\in\{10,20\}$ under the setup in
Appendix~\ref{app:iv-configurations}. At each dimension, all interventions use
the same 128 tasks for the alternating-axis model and frozen TabPFN~v2.
The task inputs, targets, and query sets are identical across models and row scopes.
These tasks are the first 128 original contexts in the saved one-sparse cohort;
selection precedes all intervention summaries.

The alternating-axis model uses the final seed-0 checkpoint trained under
$P_1$ at each dimension (Appendix~\ref{app:iv-training}).
Frozen TabPFN v2 uses \texttt{tabpfn-v2-regressor-09gpqh39.ckpt}
with package version 8.0.8, one estimator, original feature order, and no
augmentation. We retain its native encoding, normalization, and
bar-distribution mean, returning predictions in original response units.
Inference uses float32. Layers are numbered from one throughout.

\paragraph{Feature-message intervention.}
Let $I(g)$ contain the raw-feature indices in source token $g$.
For the alternating-axis model, $I(g)=\{g\}$; for TabPFN v2,
$|I(g)|=2$ at both evaluated dimensions. A source is active if its group
contains an active coordinate; this status is used only for analysis.
We intervene on each source separately at every layer, holding the context
and query sets fixed. At layer $\ell$, let $\alpha_{i,p,g}^{(\ell,h)}$
be the attention weight from source $g$ to receiver $p$ in row $i$ and
head $h$, and let $\mathbf v_{i,g}^{(\ell,h)}$ be its value vector.
For a selected row set $\mathcal R$, the projected source message and
modified attention output are
\begin{equation}
\begin{aligned}
 m_{i,p\leftarrow g}^{(\ell)}
 &=W_O^{(\ell)}\operatorname{concat}_h
 [\alpha_{i,p,g}^{(\ell,h)}\mathbf v_{i,g}^{(\ell,h)}],\\
 o_{ip}^{(\ell)}(\delta;g,\mathcal R)
 &=o_{ip}^{(\ell)}(0)-\delta\,\mathbf1\{i\in\mathcal R\}
 m_{i,p\leftarrow g}^{(\ell)}.
\end{aligned}
\label{eq:nano-source-message-attenuation}
\end{equation}
We set $\delta=0.1$, with dose comparisons
at $0.05$ and $0.2$. Context and query interventions use their respective row
sets $\mathcal R$; joint interventions use both sets. The same multiplier
applies to every receiving token and attention head in the selected rows.
Values include their projection bias; the output-projection bias is unchanged.
Attention weights at the modified operation are held fixed without
renormalization, and downstream computation is rerun. The alternating-axis
implementation subtracts the projected source message; the TabPFN adapter
scales the source value before aggregation.

\paragraph{Coefficient sensitivity.}
For each context $D$, we fit the baseline and intervened predictors on the
same queries using the affine model
\[
 \widehat\zeta(D)+\sum_r\widehat\eta_r(D)x_r.
\]
Appendix~\ref{app:iv-projection}, Eq.~\ref{eq:iv-affine-blp}, gives the estimator.
Here $\widehat\eta_r(D)$ is the fitted slope for query feature $x_r$.
Subscripts $0$ and $\delta,\ell g$ denote the baseline and intervened
predictors. Extending Eq.~\ref{eq:nano-route-matrix-estimate} to grouped sources gives
\[
 \widehat J_{rg}^{(\ell)}(D;\delta,\mathcal R)
 =\frac{\widehat\eta_{0,r}(D)
 -\widehat\eta_{\delta,\ell g,\mathcal R,r}(D)}{\delta}.
\]
We suppress the row-set argument $\mathcal R$ below.
Attenuating source $g$ therefore changes coefficient $r$ by
$-\delta\widehat J_{rg}^{(\ell)}$.

For isotropic Gaussian queries, the population slope is $\eta_f(D)=\mathbb E_x[xf(D,x)]$, hence
\[
 J_{rg}^{(\ell)}
 =\mathbb E_x\!\left[
 x_r\,\frac{f_0(D,x)-f_{\delta,\ell g}(D,x)}{\delta}
 \right].
\]
Thus $J$ measures the feature-aligned component of the prediction response; with the same query design, its least-squares estimate is equivalently obtained by regressing the prediction difference divided by $\delta$ on the query features, including an intercept.

To normalize for the magnitude of context statistics, let
$c=X^\top y/n$ and $u_g=\sum_{r\in I(g)}|c_r|$. We define the
\emph{coefficient sensitivity} as the total absolute own-feature
response divided by $u_g$:
\begin{equation}
 w_g^{(\ell)}
 =\frac{\sum_{r\in I(g)}|\widehat J_{rg}^{(\ell)}|}{u_g},
 \qquad u_g>0.
 \label{eq:grouped-coefficient-control}
\end{equation}
For a single-feature source, $w_j^{(\ell)}=|\widehat J_{jj}^{(\ell)}|/|c_j|$
when $c_j\ne0$. Thus $w_g^{(\ell)}\geq0$;
larger values indicate stronger own-feature coefficient effects per unit
context-statistic magnitude. All evaluated sources have $u_g>0$.

\paragraph{Prediction-error change.}
Using a shared set of held-out queries, we compare the baseline predictor $f_0$ and
the intervened predictor $f_{\delta,\ell g}$ against the Bayes prediction
under $P_k$:
\begin{equation}
 \begin{aligned}
 E(f;D)&=\frac{1}{Q_{\mathrm{eval}}}
 \sum_{m=1}^{Q_{\mathrm{eval}}}
 \left[f(D,x_m^{\mathrm{eval}})
       -(x_m^{\mathrm{eval}})^\top\eta_k^*(c)\right]^2,\\
 \Delta E_{\ell g}&=E(f_{\delta,\ell g};D)-E(f_0;D).
 \end{aligned}
 \label{eq:route-prediction-error}
\end{equation}
Positive $\Delta E_{\ell g}$ indicates that attenuation worsens Bayes
approximation; negative values indicate improvement.

\paragraph{Specificity.}
For a fixed context, layer, and dose, first compute each source's own-feature
response fraction,
\begin{equation}
 p_g^{(\ell)}(D;\delta)=
 \frac{\sum_{r\in I(g)}|\widehat J_{rg}^{(\ell)}|}
      {\sum_{r=1}^d|\widehat J_{rg}^{(\ell)}|}.
 \label{eq:route-specificity}
\end{equation}
For a source role $R\in\{\mathrm{active},\mathrm{inactive}\}$, specificity is
the arithmetic mean over its source set, so each source receives equal weight
regardless of its total response magnitude:
\begin{equation}
 P_R^{(\ell)}(D;\delta)=
 \frac{1}{|\mathcal G_R(D)|}
 \sum_{g\in\mathcal G_R(D)}p_g^{(\ell)}(D;\delta).
 \label{eq:route-specificity-role}
\end{equation}
\paragraph{Aggregation and uncertainty.}
The plotted points are arithmetic means of these context-level quantities,
with equal weights across contexts. For example, at $d=10$, $k=1$ in the
alternating-axis model, inactive specificity is the average of nine separate
source fractions within each context. Active and inactive specificities do
not sum to one. Ratios with source denominators at most $10^{-12}$ are
undefined and excluded before taking the source mean; no such ratios occur
in the plotted cohorts. A context with no sources in a role is omitted for
that role. In particular, no inactive curve exists at $k=d$.
Coefficient sensitivity $w$ and error change $\Delta E$ are averaged over
sources within each role and context, then summarized by the median across
contexts. Specificity uses the equal-source and equal-context means above.
Intervals are pointwise 95\% percentile confidence intervals from 2,000
generator-batch resamples at a fixed checkpoint. The one-sparse cohort
contains 122 original batches per dimension. Each draw retains all selected contexts
from a sampled batch and recomputes the context-level mean or median.
Draws are shared across models, row scopes, layers, doses, and source roles
within each dimension and prior.

\subsection{Validity of the coefficient-response interpretation}
\label{app:cross-scale-geometry}
\label{app:cross-scale-reconstruction}
\label{app:cross-scale-fidelity}

We assess how well fitted coefficient changes reconstruct the prediction
response, then use response matrices to locate these changes by feature.

\paragraph{Fidelity to the prediction change.}
Fix a context, layer $\ell$, and source $g$, and suppress $\ell,g$ in the
following definitions. On held-out queries, write
$\Delta f_m=f_{\delta,\ell g}(D,x_m^{\rm eval})
-f_0(D,x_m^{\rm eval})$ and
$\Delta\widehat\eta=\widehat\eta_{\delta,\ell g}-\widehat\eta_0$.
The centered prediction change is
\[
 \widetilde{\Delta f}_m:=\Delta f_m-\overline{\Delta f},
\]
and its linear reconstruction from the fitted coefficient change is
\[
 \widetilde{\Delta f}^{\,\mathrm{lin}}_m
 :=(x_m^{\rm eval}-\bar x^{\rm eval})^\top\Delta\widehat\eta.
\]
Bars denote held-out query means. We measure reconstruction fidelity by
\begin{equation}
 \widehat R^2_{\Delta,\ell g}=1-
 \frac{\sum_m\bigl(\widetilde{\Delta f}_m
       -\widetilde{\Delta f}^{\,\mathrm{lin}}_m\bigr)^2}
      {\sum_m\bigl(\widetilde{\Delta f}_m\bigr)^2}.
 \label{eq:full-intervention-fidelity}
\end{equation}
The numerator is the squared reconstruction error; the denominator is the
total squared centered prediction change. Values near one indicate that
coefficient changes faithfully capture the intervention's effect across
queries, after removing its mean shift. We exclude ratios with
denominators at most $10^{-12}$ and average valid scores within each source
role and context.

For separate-row interventions, final-layer active-source median fidelity
ranges from $0.81$ to $0.97$ under context attenuation and from $0.97$ to
$0.98$ under query attenuation across models and dimensions. Thus, fitted
coefficient changes capture most of the centered active-source response,
with higher fidelity for query interventions.

Figure~\ref{fig:message-fidelity} gives the full depth profiles under joint
attenuation. Final-layer active-source medians exceed $0.96$ in both models
and dimensions. Inactive-source medians are $0.94$ and $0.83$ in the
alternating-axis model at $d=10$ and $d=20$, and $0.71$ and $0.67$ in
TabPFN~v2. The coefficient description is therefore more complete for
active-source responses.

\begin{figure}[!tbp]
 \centering
 \input{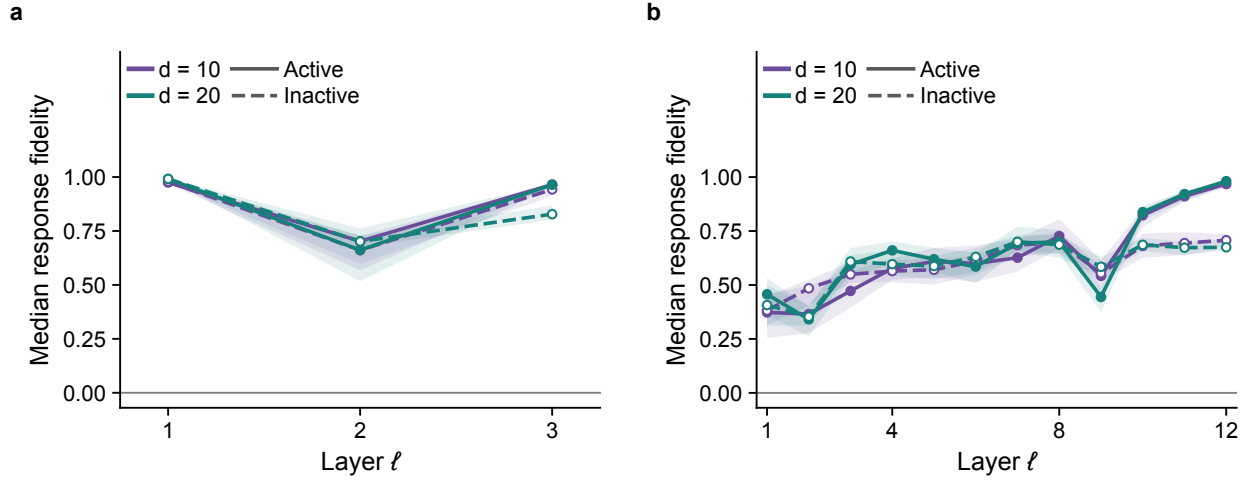}
 \caption{\textbf{Fidelity of coefficient changes to prediction changes.}
 All layers, $k=1$, $\delta=0.1$, joint context-and-query attenuation.
 Panels compare active and inactive sources within (a) the alternating-axis
 model and (b) frozen TabPFN~v2, each using the same 128 tasks per dimension.
 Colors distinguish $d=10$ and $d=20$.
 Solid and dashed lines show active- and inactive-source medians of
 within-context role-averaged $\widehat R^2_\Delta$, respectively, with equal
 context weights. Bands use the pointwise 95\% bootstrap procedure in
 Appendix~\ref{app:cross-scale-protocol}.
 Higher values indicate
 better fidelity ($\uparrow$); the maximum is one, and values can be negative.}
 \label{fig:message-fidelity}
\end{figure}

\paragraph{Illustrative response matrices.}
Figure~\ref{fig:message-response-matrices-context-d10} complements the
query interventions in Figure~\ref{fig:message-response-matrices} with
context interventions on the same $d=10$ task.
Figure~\ref{fig:message-response-matrices-d20} shows both row scopes at $d=20$. At each dimension, the example
is selected by proximity to the median Bayes coefficient norm among the
same 128 tasks, with the smallest task ID breaking ties. Within each
dimension, the same context is shown for both row scopes at all three layers.
In both examples, the largest final-layer response occurs on the active
source's own coefficient under either intervention, with a larger magnitude
for query attenuation. At $d=20$, the final-layer own-coefficient responses
have opposite signs across the two row scopes, illustrating distinct
context and query effects within the same task.

\begin{figure}[!tbp]
 \centering
 \input{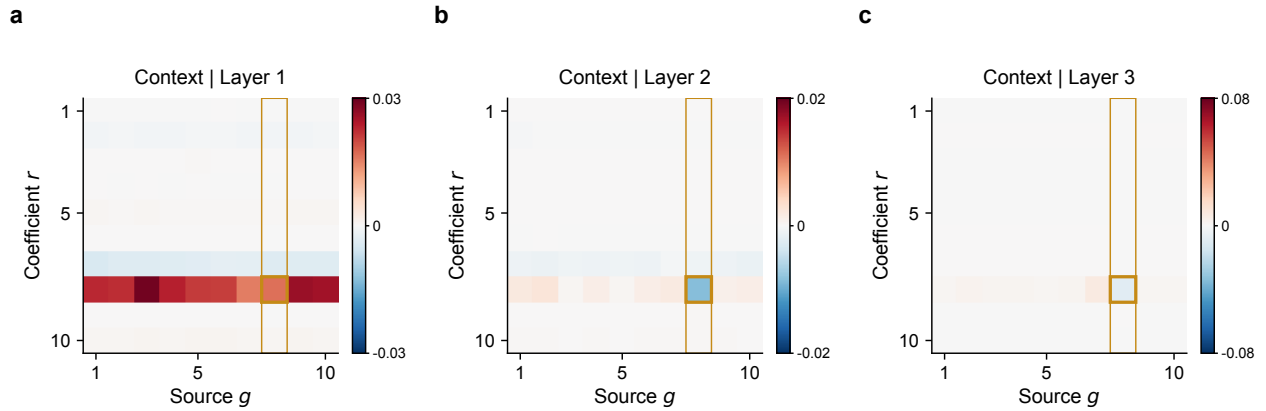}
 \caption{\textbf{Coefficient responses to context-message attenuation.}
 Alternating-axis model, $d=10$, $k=1$, $\delta=0.1$; panels show layers 1--3.
 Each matrix column attenuates one source on context rows.
 The task, matrix axes, per-layer color scales, and gold outlines match
 Figure~\ref{fig:message-response-matrices}, allowing direct comparison with query interventions.}
 \label{fig:message-response-matrices-context-d10}
\end{figure}

\begin{figure}[!tbp]
 \centering
 \input{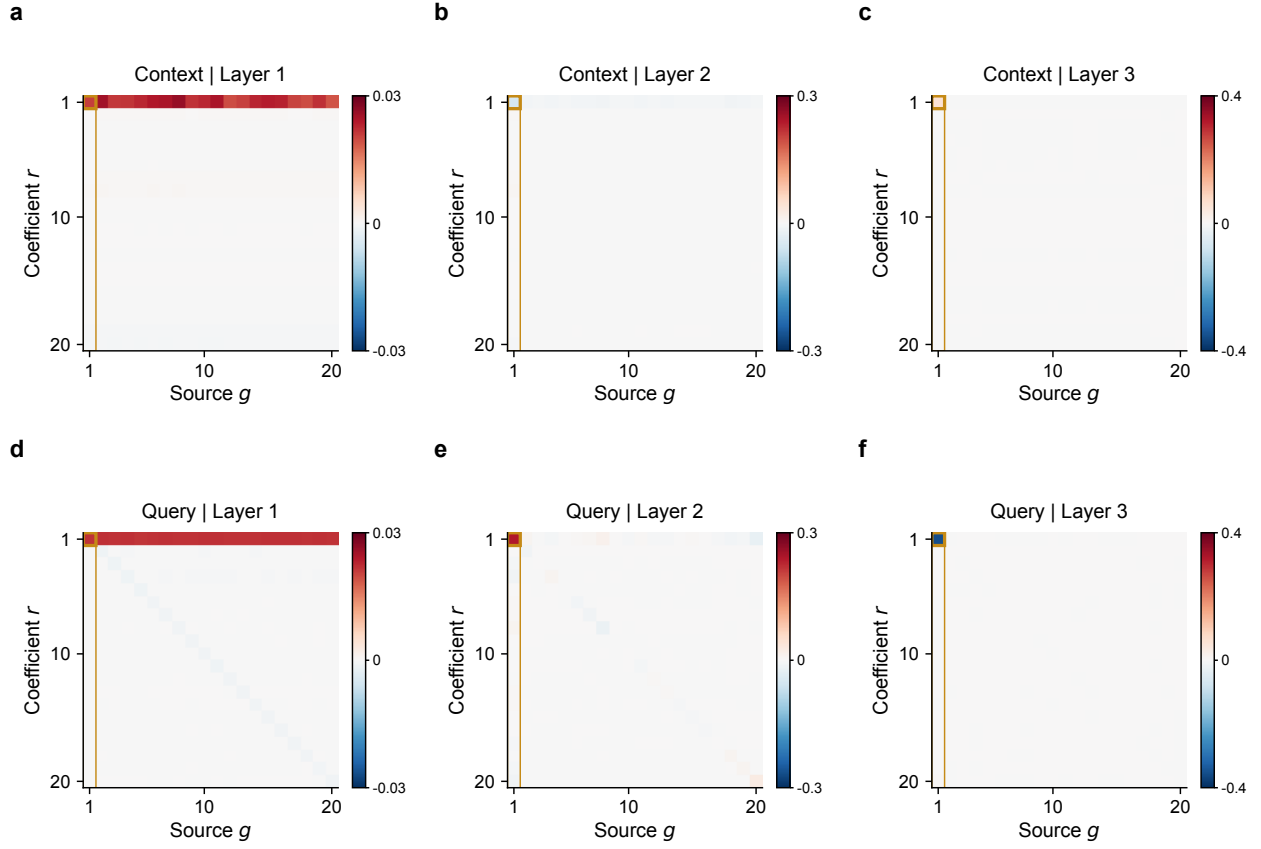}
 \caption{\textbf{Coefficient responses within a $\boldsymbol{d=20}$ task.}
 Alternating-axis model, $k=1$, $\delta=0.1$.
 Top: context interventions; bottom: query interventions. Columns show layers 1--3.
 The same task is used throughout; context and query panels share a symmetric
 color scale within each layer. Matrix axes and gold outlines follow
 Figure~\ref{fig:message-response-matrices}.}
 \label{fig:message-response-matrices-d20}
\end{figure}

\subsection{Context- and query-message effects across dimensions}
\label{app:cross-scale-row-scope}

\paragraph{Predictive consequences at $d=10$.}
Figure~\ref{fig:scope-prediction-effects-d10} complements the main-text
specificity and sensitivity panels with held-out Bayes-target error changes.
In TabPFN~v2, final-layer query attenuation increases error for active groups
and decreases it for inactive groups. In the alternating-axis model, the
active-query median is smaller and its 95\% interval spans zero.
These effects distinguish coefficient control from predictive usefulness.

\begin{figure}[!tbp]
\centering
\begin{subfigure}[t]{0.45\linewidth}
 \setcounter{subfigure}{2}
 \phantomsubcaption
 \label{fig:section4_scope_effects:c}
 \includegraphics[width=\linewidth]{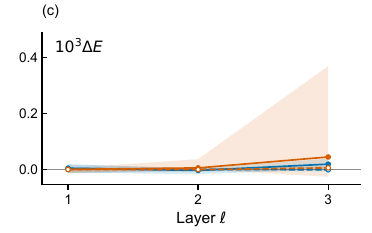}
\end{subfigure}\hfill
\begin{subfigure}[t]{0.45\linewidth}
 \setcounter{subfigure}{5}
 \phantomsubcaption
 \label{fig:section4_scope_effects:f}
 \includegraphics[width=\linewidth]{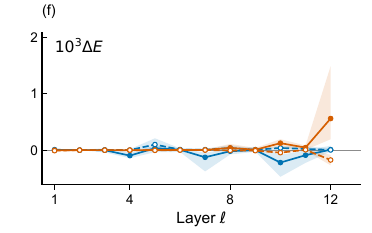}
\end{subfigure}
\caption{\textbf{Predictive consequences of context- and query-message attenuation.}
Panel (c): alternating-axis model; panel (f): frozen TabPFN~v2.
$d=10$, $k=1$, $\delta=0.1$; the same 128 tasks as in
Figure~\ref{fig:cross-scale-signed}.
Curves show median prediction-error changes after averaging sources within
each role and task; bands are pointwise 95\% generator-batch bootstrap intervals.
Panels (c) and (f) complete Figure~\ref{fig:cross-scale-signed}, whose panels (a, b, d, e) appear in the main text.}
\label{fig:scope-prediction-effects-d10}
\end{figure}

\paragraph{Replication at $d=20$.}

Figure~\ref{fig:cross-scale-scope-d20} repeats the separate-row interventions
at $d=20$. Active sources have higher specificity at every layer in both
models and row scopes. Final-layer active-source sensitivity also exceeds
inactive-source sensitivity in all four comparisons. In the alternating-axis
model, attenuating active sources in either row set increases final-layer
prediction error. In TabPFN~v2, the larger prediction-error contrast occurs
under query attenuation, as at $d=10$.

\begin{figure}[!tbp]
 \centering
 \input{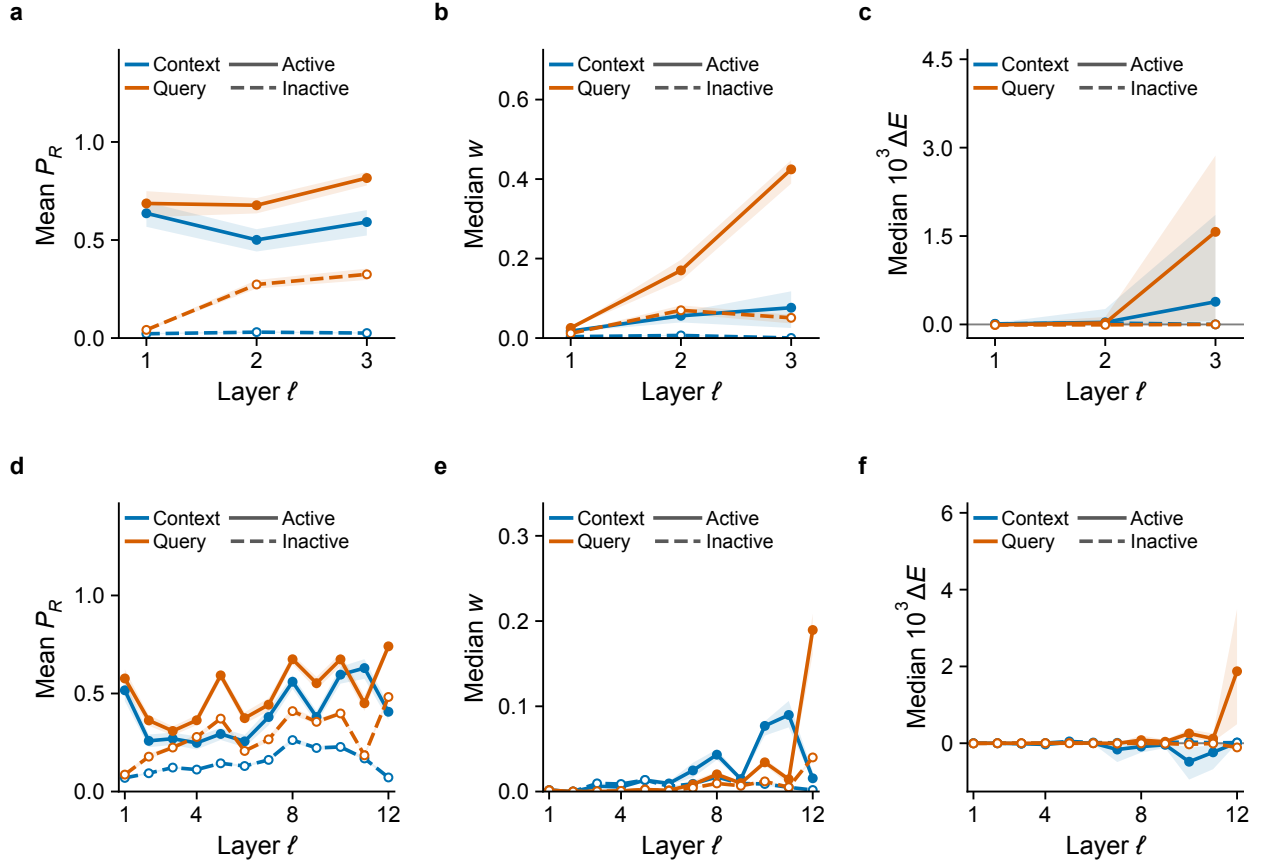}
 \caption{\textbf{Context and query message effects at $\boldsymbol{d=20}$.}
 $k=1$, $\delta=0.1$, 128 shared tasks. Top: alternating-axis model;
 bottom: frozen TabPFN~v2. Layout, source-role summaries, and confidence
 intervals follow Figure~\ref{fig:cross-scale-signed}.}
 \label{fig:cross-scale-scope-d20}
\end{figure}

\subsection{Response direction and predictive usefulness}
\label{app:cross-scale-direction}

Specificity and sensitivity measure the location and magnitude of a response.
We next examine its direction and relation to prediction error.

\paragraph{Signed sensitivity.}
For each learned-model source, we orient its own-feature responses by the
corresponding context statistics and define
\begin{equation}
 s_g^{(\ell)}=
 \frac{\sum_{r\in I(g)}\operatorname{sign}(c_r)\widehat J_{rg}^{(\ell)}}
      {\sum_{r\in I(g)}|c_r|}.
 \label{eq:signed-source-sensitivity}
\end{equation}
For the alternating-axis model, this reduces to
$s_j^{(\ell)}=\widehat J_{jj}^{(\ell)}/c_j$.
Positive values indicate that retaining the message supports the direction
of $c$; negative values indicate an opposing response. For native
TabPFN~v2 groups, the numerator sums the two direction-aligned coordinate
responses. The magnitude satisfies $|s_g|\leq w_g$, with equality for
single-feature sources. We also report the unnormalized own-feature
magnitude $a_g=\sum_{r\in I(g)}|\widehat J_{rg}|$.

Figure~\ref{fig:cross-scale-signed-distribution} shows the distribution of
signed sensitivity for active sources at the final feature-attention layer,
using the same 128 one-sparse tasks at $\delta=0.1$.
Each task has exactly one active source in both
models. Active query responses are negative in 56 tasks at
$d=10$ and seven at $d=20$ in the alternating-axis model, compared with one
and zero in TabPFN~v2. Among these negative-query cases, the alternating-axis
model has median $w_g=0.199$ and $a_g=0.063$ at $d=10$, versus $0.026$
and $0.002$ at $d=20$. Context interventions yield more frequent negative
responses than query interventions in every model--dimension comparison.
The distributions therefore distinguish response direction from the
unsigned selectivity measured in Figure~\ref{fig:cross-scale-signed}.

\begin{figure}[!tbp]
 \centering
 \input{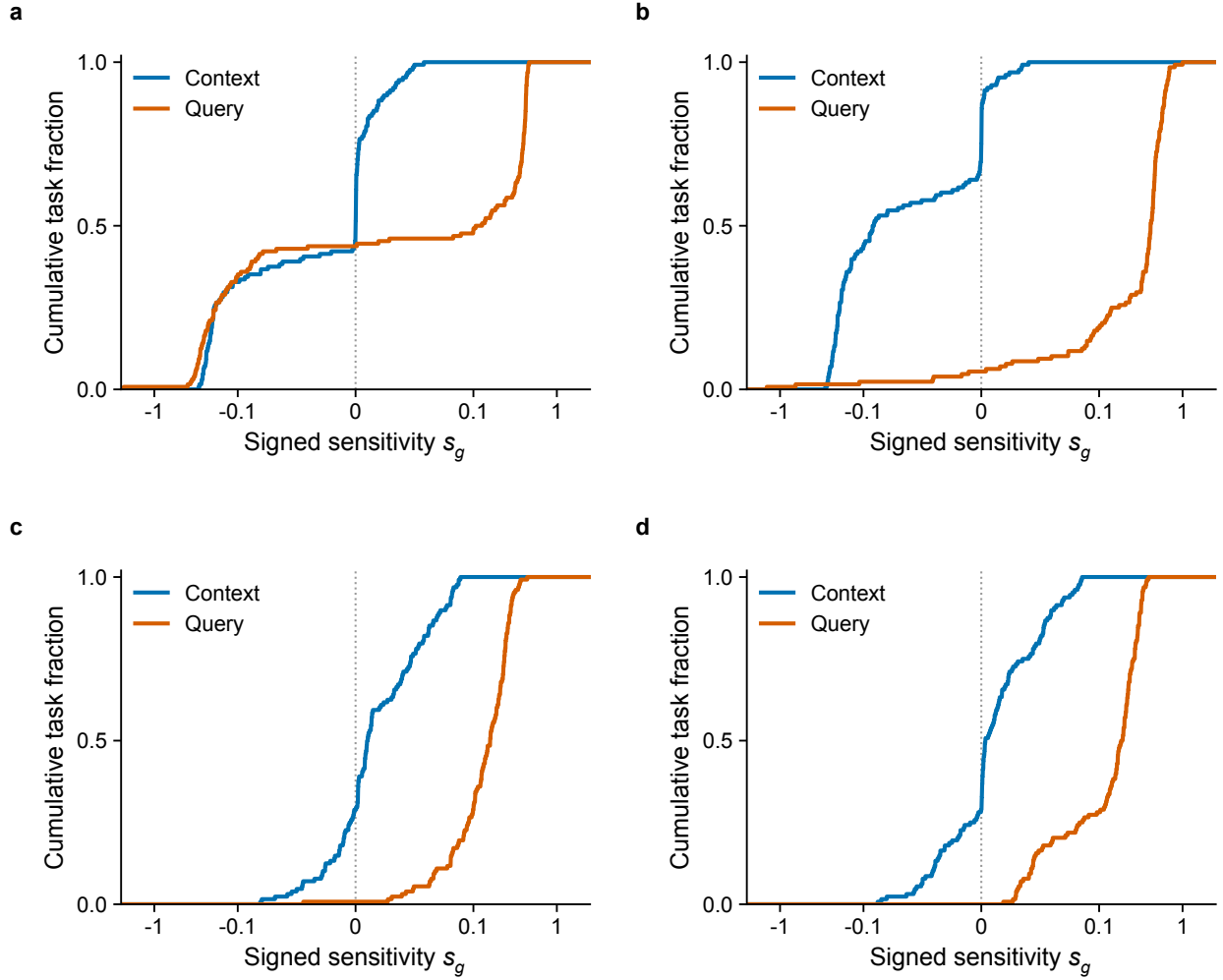}
 \caption{\textbf{Direction and magnitude of active-source control.}
 Final-layer empirical distributions of signed sensitivity $s_g$, with
 $k=1$, $\delta=0.1$, and the same 128 tasks per dimension and model.
 Top: alternating-axis model; bottom: TabPFN~v2. Left: $d=10$; right: $d=20$.
 Colors distinguish context and query attenuation, as in
 Figure~\ref{fig:cross-scale-signed}. All tasks are shown. The horizontal
 axis is linear for $|s_g|\leq0.05$ and logarithmic outside that interval;
 the dotted line marks zero.}
 \label{fig:cross-scale-signed-distribution}
\end{figure}

The following three references relate response direction to the intervened
quantity in a Bayesian computation.

\paragraph{Final-contribution reference.}
In the one-sparse construction of Figure~\ref{fig:arch-and-construction}(c),
the final message from feature $j$ contributes $\eta_{1,j}^*x_j$ to the
prediction. Attenuating this contribution gives
\[
 f_{\delta,j}(D,x)=f_1^*(D,x)-\delta\eta_{1,j}^*x_j,
 \qquad
 J^{\mathrm{ideal}}_{rj}=\mathbf1\{r=j\}\eta_{1,j}^*.
\]
Thus its own-feature response relative to $c_j$ is
$J^{\mathrm{ideal}}_{jj}/c_j=\lambda_1q_{1,j}(c)\geq0$ for $c_j\ne0$.
This final-contribution intervention supplies a reference for both the
localization and direction of the learned coefficient response.
For every nonzero response, specificity is one regardless of realized
support membership. Under independent isotropic queries, the increase in
Bayes approximation error from this exact baseline is
$\delta^2(\eta_{1,j}^*)^2\geq0$. Finite-noise posterior inclusion is positive
even for inactive features, so zero inactive response is not required.

\paragraph{Perturbing evidence before competition.}
Write $q_j=q_{1,j}(c)$, $\eta_j^*=\eta_{1,j}^*(c)$,
$\lambda=\lambda_1$, and
$\theta=\theta_1=v^2/[2\tau^2(v^2+\tau^2)]$, with $\tau^2=\sigma^2/n$.
Suppose an intervention replaces only $c_j$ by $(1-\delta)c_j$ before
evaluating $q_r=\exp(\theta c_r^2)/\sum_s\exp(\theta c_s^2)$ and the
posterior coefficients. Differentiating gives
\[
 \frac{\partial q_r}{\partial c_j}
 =2\theta c_jq_r(\mathbf1\{r=j\}-q_j).
\]
The limiting coefficient response is therefore
\begin{equation}
 \begin{aligned}
 J^{\mathrm{evidence}}_{rj}
 &:=\lim_{\delta\to0}
   \frac{\eta_r^*(c)-\eta_r^*(c-\delta c_j\mathbf e_j)}{\delta}
   =c_j\frac{\partial\eta_r^*}{\partial c_j}\\
 &=\lambda q_r c_j\mathbf1\{r=j\}
   +2\theta\eta_r^*c_j^2(\mathbf1\{r=j\}-q_j),
 \end{aligned}
 \label{eq:app-evidence-response}
\end{equation}
where $\mathbf e_j$ is the $j$th coordinate vector. In particular,
$J^{\mathrm{evidence}}_{jj}=\eta_j^*[1+2\theta c_j^2(1-q_j)]$,
while $J^{\mathrm{evidence}}_{rj}=-2\theta\eta_r^*q_jc_j^2$ for $r\ne j$.
Weakening one evidence coordinate thus increases competitors' coefficient
magnitudes to first order. Own-feature signed sensitivity remains
nonnegative. This reference describes evidence rescaling before posterior competition;
its off-diagonal terms distinguish it from final-contribution attenuation.

\paragraph{An alternative signed realization.}
The same Bayesian coefficient can be decomposed as
\[
 \eta_j^*=\lambda c_j-\lambda(1-q_j)c_j.
\]
Consider two additive prediction contributions implementing these terms.
Attenuating only the second, with $q$ and the first contribution held
fixed, gives $J_{jj}^{\mathrm{correction}}=-\lambda(1-q_j)c_j$ and zero
off-diagonal responses. Its signed sensitivity is nonpositive although
the baseline prediction is exactly Bayesian. The same Bayesian predictor can therefore yield either response sign,
depending on the contribution being attenuated. These additive realizations
provide functional references for interpreting the measured signs.

\paragraph{Relation to predictive effects.}
For $e=\widehat\eta_0-\eta_1^*$ and
$\widehat J=\widehat J_{\cdot g}^{(\ell)}$, the coefficient-error change is
\begin{equation}
 \|\widehat\eta_{\delta,\ell g}-\eta_1^*\|_2^2-\|e\|_2^2
 =-2\delta e^\top\widehat J+\delta^2\|\widehat J\|_2^2.
 \label{eq:message-coefficient-error-change}
\end{equation}
Equation~\ref{eq:message-coefficient-error-change} follows by substituting
$\widehat\eta_{\delta,\ell g}=\widehat\eta_0-\delta\widehat J_{\cdot g}^{(\ell)}$
and expanding the square. The identity holds at the measured dose, since
$\widehat J$ is the corresponding finite difference.
Predictive usefulness therefore depends on the response's alignment with
the baseline error as well as its magnitude. The reported $\Delta E$
evaluates the full prediction change against Bayes on held-out queries
(Eq.~\ref{eq:route-prediction-error}).
For population affine projections, Lemma~\ref{lem:coefficient-accounting}
gives the corresponding full error change as
\[
 \Delta E_{\mathrm{pop}}
 =\Delta C+(\zeta_\delta^2-\zeta_0^2)
   +\mathbb E_x[r_\delta(D,x)^2-r_0(D,x)^2],
\]
where $\Delta C$ is the population coefficient-error change. The full error change combines coefficient, intercept, and residual
components. Empirical estimates additionally include projection error and
finite-query cross terms.

\subsection{Robustness to attenuation strength}
\label{app:cross-scale-signed}
\label{app:cross-scale-nano-results}
\label{app:cross-scale-v2-results}

We vary $\delta\in\{0.05,0.1,0.2\}$ under joint context-and-query attenuation,
using the same 128 tasks per dimension. This analysis tests the dose dependence
of combined message effects; Appendix~\ref{app:cross-scale-row-scope} compares
context and query effects separately at $\delta=0.1$.

\paragraph{Alternating-axis model.}
Final-layer active-source sensitivity exceeds inactive-source sensitivity at
all three doses (Figures~\ref{fig:message-dose-d10} and~\ref{fig:message-dose-d20}).
The median active-source error change increases from $-4.6\times10^{-6}$ to
$1.7\times10^{-4}$ at $d=10$, and from $1.4\times10^{-5}$ to $3.6\times10^{-3}$
at $d=20$, as $\delta$ increases from $0.05$ to $0.2$.
Inactive-source changes remain smaller at the largest dose.

\paragraph{Frozen TabPFN v2.}
Final-layer active groups also have higher sensitivity at all three doses.
Their median error change is positive and increases with $\delta$, whereas
attenuating inactive groups decreases error. Thus, the active-source
sensitivity advantage persists across doses in both models, while the sign
of the prediction effect depends on the model and source role.

\begin{figure}[!tbp]
\centering
\begin{minipage}{.95\linewidth}
\centering
\input{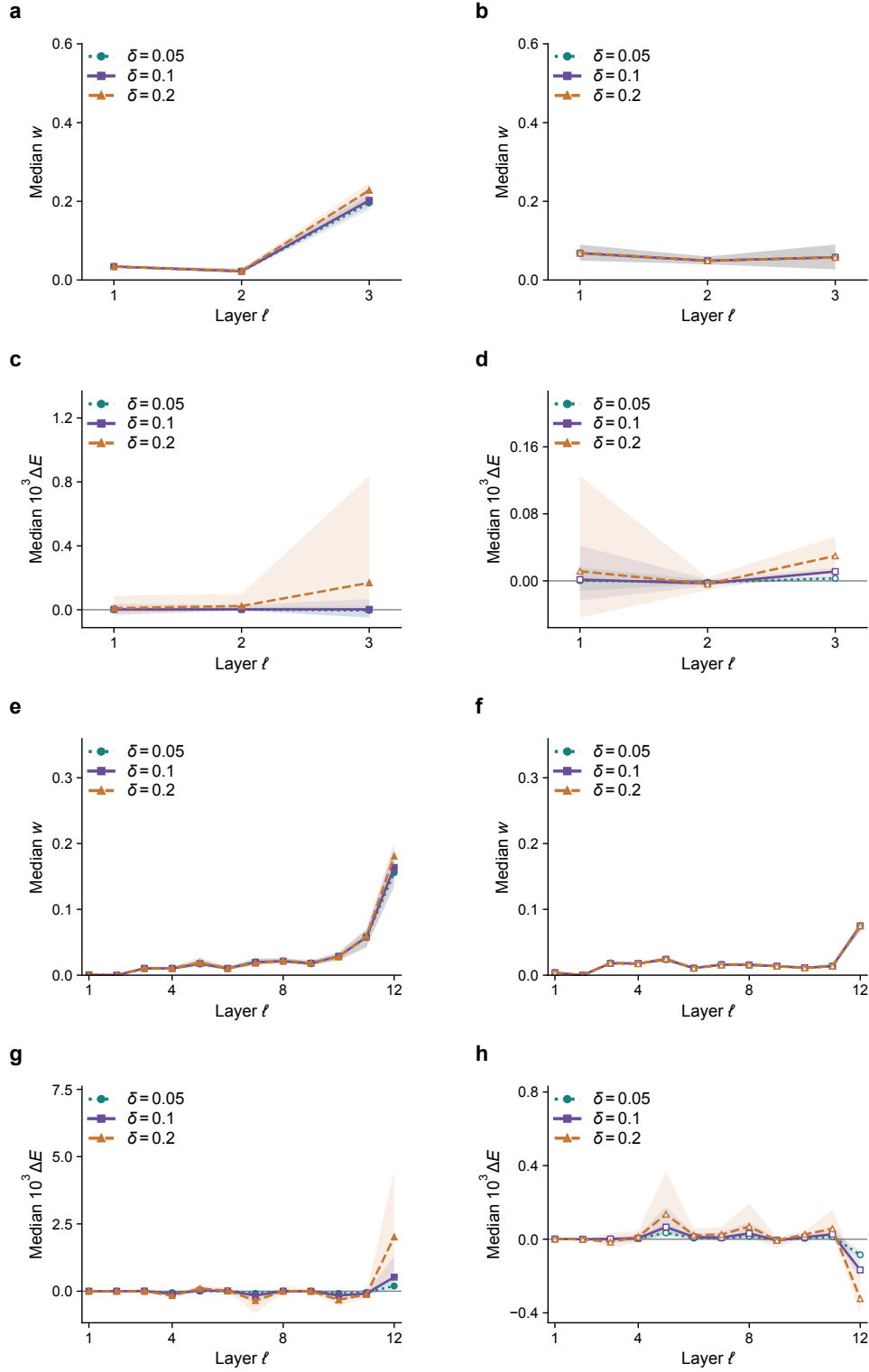}
\end{minipage}
\caption{\textbf{Attenuation-strength comparison at $d=10$.}
$k=1$, 128 tasks, joint context-and-query attenuation.
Panels (a--d): alternating-axis model; (e--h): TabPFN~v2.
For each model, rows show sensitivity and error change; columns show active
and inactive sources. Curves show task medians of role averages; bands are
95\% batch-bootstrap intervals (Appendix~\ref{app:cross-scale-protocol}).
Dose comparisons share bootstrap draws. Sensitivity axes match across
dimensions; error axes are panel-specific.}
\label{fig:message-dose-d10}
\label{fig:message-nano-sensitivity-d10}
\label{fig:message-v2-sensitivity-d10}
\label{fig:message-nano-sensitivity}
\label{fig:message-v2-sensitivity}
\end{figure}

\begin{figure}[!tbp]
\centering
\begin{minipage}{.95\linewidth}
\centering
\input{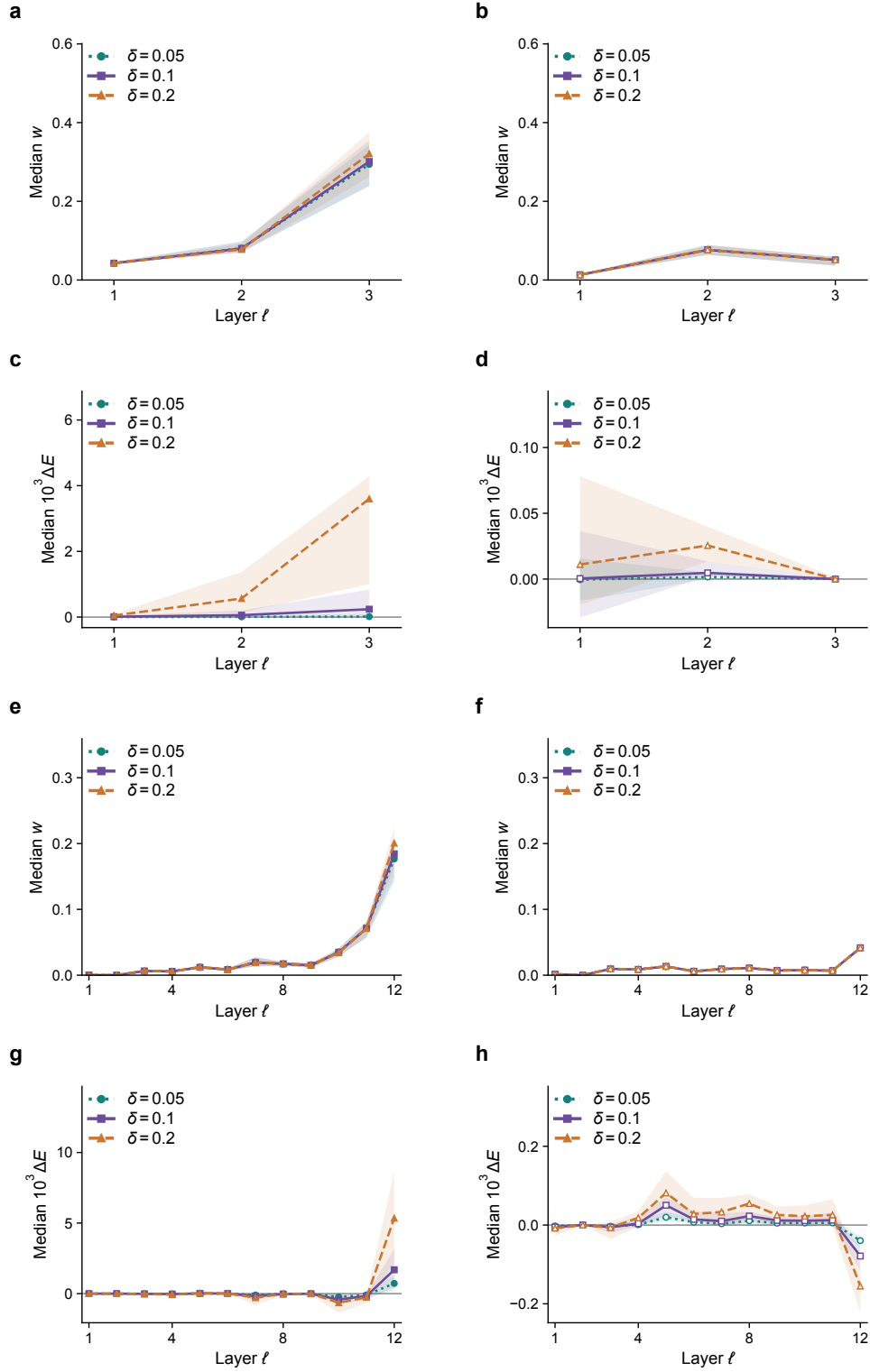}
\end{minipage}
\caption{\textbf{Attenuation-strength comparison at $d=20$.}
$k=1$, 128 tasks, joint context-and-query attenuation.
Panels (a--d): alternating-axis model; (e--h): TabPFN~v2.
For each model, rows show sensitivity and error change; columns show active
and inactive sources. Curves show task medians of role averages; bands are
95\% batch-bootstrap intervals (Appendix~\ref{app:cross-scale-protocol}).
Dose comparisons share bootstrap draws. Sensitivity axes match across
dimensions; error axes are panel-specific.}
\label{fig:message-dose-d20}
\label{fig:message-nano-sensitivity-d20}
\label{fig:message-v2-sensitivity-d20}

\end{figure}

\FloatBarrier

\subsection{Extension across support sizes and scope of conclusions}
\label{app:cross-scale-support-size}

Section~\ref{sec:cross-scale-gating} focuses on the one-sparse prior.
We repeat the same intervention at $k=d/2$ to test whether the final-layer active--inactive distinction persists beyond extreme sparsity, and at $k=d$ to test whether feature messages remain functionally important when every feature is active.

\paragraph{Protocol.}
We use the intervention and measurements in Appendix~\ref{app:cross-scale-protocol}, with $\delta=0.1$, $d\in\{10,20\}$, and joint attenuation of context and query rows.
The alternating-axis model uses the final seed-0 checkpoint trained under each matched prior $P_k$; TabPFN~v2 remains frozen across priors.
At each dimension and prior, both models use the same 128 tasks. These are the first 128 original contexts in the corresponding saved active--active pair cohort; only the original context enters the analysis.
Bayes approximation error uses the corresponding $\eta_k^*(c)$ in Eq.~\ref{eq:route-prediction-error}.
Aggregation and batch-bootstrap intervals follow Appendix~\ref{app:cross-scale-protocol}.

At $k=d/2$, TabPFN~v2 has an inactive native feature group in 111 of 128 contexts at $d=10$ and all 128 at $d=20$; inactive-group summaries use these contexts, while active-group summaries
use all 128. The two role summaries therefore cover different task sets at $d=10$.
At $k=d$, every source is active.

\paragraph{Intermediate support: $k=d/2$.}
Final-layer active sources have higher mean specificity and larger median
prediction-error changes than inactive sources in both models and dimensions
(Figure~\ref{fig:support-size-comparison}, rows 1 and 3).
TabPFN~v2 also retains a larger active-source median sensitivity at both
dimensions, as does the alternating-axis model at $d=20$.
In the alternating-axis model at $d=10$, active and inactive median
sensitivities are close ($0.194$ and $0.200$), while specificity is
$0.499$ versus $0.251$. Response localization and prediction effects retain an active--inactive
contrast, while the sensitivity advantage varies across conditions.

\paragraph{Dense support: $k=d$.}
At the dense endpoint every source is active. Source attenuation continues to change the corresponding coefficients and produces positive final-layer median $\Delta E$ in both models and dimensions (Figure~\ref{fig:support-size-comparison}, rows 2 and 4).
Feature-message control therefore extends to dense prediction, where support
selection is unnecessary.

\paragraph{Fidelity of coefficient changes.}
Across the two additional priors and dimensions, final-layer active-source
medians of $\widehat R^2_\Delta$ range from $0.93$ to $0.98$ in the
alternating-axis model and from $0.80$ to $0.85$ in TabPFN~v2
(Eq.~\ref{eq:full-intervention-fidelity}). Coefficient changes therefore
explain most of the centered prediction change, supporting the same
coefficient-based interpretation across support sizes.

\paragraph{Scope of the evidence.}
Across these experiments, feature messages exert feature-aligned control over
prediction coefficients. Under sparse priors, this control varies with source
relevance; under the dense prior, it remains functionally important.
These results locate coefficient control within the computation. Identifying
how relevance is inferred requires tracing the formation of the intervened
messages.

\begin{figure}[!tbp]
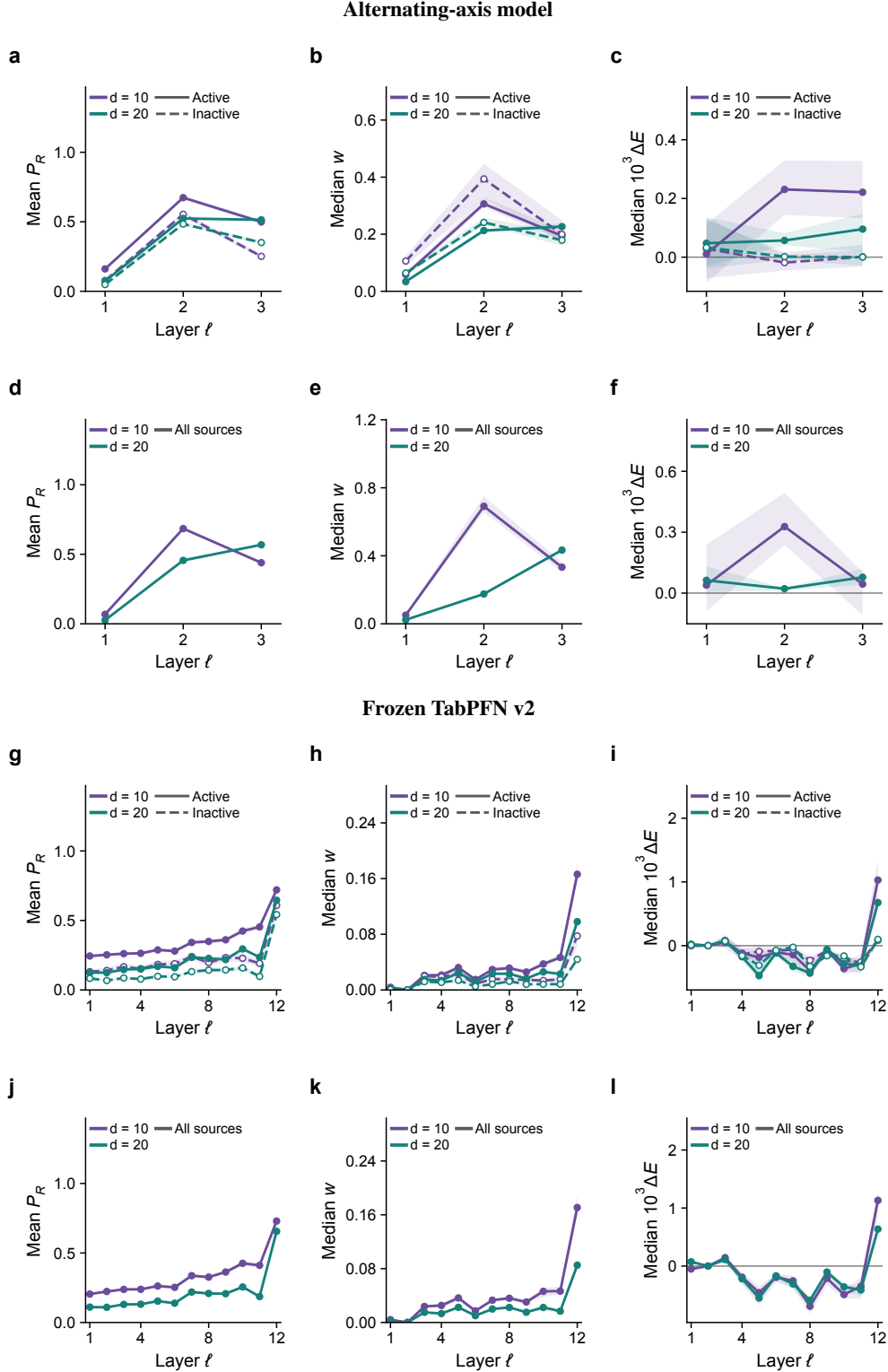

 \centering
 \begin{minipage}{.79\linewidth}
 \centering
 {\small\textbf{Alternating-axis model}}\par
 \input{figures/layouts/appendix_e_support_nano}
 \par\smallskip
 {\small\textbf{Frozen TabPFN v2}}\par
 \input{figures/layouts/appendix_e_support_v2}
 \end{minipage}
 \caption{\textbf{Coefficient control across support sizes.}
 Panels (a--f): alternating-axis model; (g--l): frozen TabPFN~v2.
 Within each model, rows show $k=d/2$ and $k=d$; columns show specificity,
 sensitivity, and prediction-error change. Joint attenuation, $\delta=0.1$,
 128 shared tasks per dimension and prior. Colors distinguish dimensions;
 solid and dashed curves denote active and inactive sources. All sources
 are active at $k=d$. Bands are pointwise 95\% batch-bootstrap intervals.
 TabPFN~v2 sources are native two-feature groups.}
 \label{fig:support-size-comparison}
 \label{fig:support-size-nano}
 \label{fig:support-size-v2}
\end{figure}

\section{Notation}
\label{app:notation}

Population BLP coefficients and error components carry no accents; hats denote
fitted predictors and empirical estimates. Stars identify Bayes-optimal
quantities or population optima, and bars denote sample means. A subscript
identifies the predictor, prior, or coordinate; a parenthesized superscript
$(t)$ identifies a test task. Dependence on fixed parameters is suppressed
only within a stated setting.
Controlled-experiment definitions: dense-prior prediction gap
(Eq.~\ref{eq:iv-dense-reference-error}), Bayes approximation error
(Eq.~\ref{eq:iv-relative-error}; finite-query estimate in
Eq.~\ref{eq:iv-error-estimate}), architecture gap
(Eq.~\ref{eq:iv-architecture-gap}), population error components
(Lemma~\ref{lem:coefficient-accounting}; Appendix~\ref{app:coefficient-accounting}), and their empirical estimates
(Eqs.~\ref{eq:iv-estimated-components}--\ref{eq:iv-component-gap-definitions}).

\begin{table}[!ht]
    \centering
    \small
    \renewcommand{\arraystretch}{1.08}
    \caption{Notation for the controlled model, affine projection, and interventions.}
    \label{tab:notation}
    \begin{tabular}{@{}p{0.32\linewidth}p{0.64\linewidth}@{}}
        \toprule
        Symbol & Meaning \\
        \midrule
        $D=(X,y)$; $n,d$ & Context, context size, and feature dimension;
        $X\in\mathbb R^{n\times d}$, $y\in\mathbb R^n$. \\
        $x\in\mathbb R^d$; $x_i^\top$ & Independent query vector;
        row $i$ of the context design. Within a query, $x_j$ denotes coordinate $j$. \\
        $\beta$; $S$; $k$; $P_k$ & Data-generating regression coefficients,
        their active support, its size, and the corresponding coefficient prior. \\
        $v^2$; $\sigma^2$; $\tau^2$ & Total prior signal energy,
        response-noise variance, and $\sigma^2/n$. \\
        $c=X^\top y/n$ & Sufficient statistic for the coefficient posterior. \\
        $\pi_k(A\mid c)$; $q_{k,j}(c)$ & Posterior probability of support
        $A$; posterior inclusion probability of coordinate $j$. \\
        $\lambda_k$; $\theta_k$ & Conditional Gaussian shrinkage factor;
        scale multiplying squared evidence in the support posterior. \\
        $\eta_k^*(c)$; $f_k^*(D,x)$ & Bayes coefficient map and prediction
        $x^\top\eta_k^*(c)$; an additional index $n$ exposes context-length dependence. \\
        $\zeta_f(D)$; $\eta_f(D)$ & Population affine-BLP intercept and
        slope of predictor $f$ under the query distribution. \\
        $\widehat\zeta_f(D)$; $\widehat\eta_f(D)$ & Finite-query
        least-squares estimates of the BLP intercept and slope. \\
        $r_f(D,x)$ & Residual after subtracting the population affine BLP;
        $\widehat r_f$ uses the fitted BLP. \\
        $C_f$; $\zeta_f^2$; $N_f$ & Coefficient error, squared intercept,
        and nonlinear residual error, respectively. \\
        $\Delta_k^{\mathrm{dense}}(D)$ & Mean squared prediction difference
        between the dense-prior rule and Bayes for dataset $D$. \\
        $\mathcal E_{a,k}$; $\widehat{\mathcal E}_{a,k}$
        & Task-level normalized Bayes approximation error and its finite-query
        estimate; both refer to one dataset and one training seed. \\
        $\widehat C_{a,k}$; $\widehat Z_{a,k}$; $\widehat N_{a,k}$
        & Normalized coefficient, squared-intercept, and residual errors
        for one dataset and one training seed. \\
        $R_k(f)$ & Population risk against noisy query responses
        (Appendix~\ref{app:coefficient-accounting}). \\
        $a$; $\widehat f_{a,k}$ & Architecture or model label; trained
        controlled predictor under $P_k$; $a$ is row-token or alternating-axis. \\
        $T$; $t$ & Test-task count and task index. \\
        $Q_{\mathrm{proj}}$; $Q_{\mathrm{eval}}$ & Numbers of
        projection and evaluation queries; the task index is omitted for
        a fixed context. \\
        $\widehat G_k$; $\widehat G_k^{U}$
        & Per-dataset row-token minus alternating-axis gaps in normalized
        prediction error and component $U\in\{C,\zeta,N\}$, for one training
        seed. The $\zeta$ label denotes squared intercepts. \\
        $\gamma_{\ell g}$; $\delta$; $\widehat J_{rg}^{(\ell)}$ & External
        message multiplier, attenuation fraction, and fitted coefficient
        response per unit attenuation $(\widehat\eta_{0,r}-\widehat\eta_{\delta,\ell g,r})/\delta$. \\
        $I(g)$; $u_g$; $w_g^{(\ell)}$ & Source token's raw-feature indices,
        $\sum_{r\in I(g)}|c_r|$, and coefficient sensitivity. \\
        $\widehat R^2_{\mathrm{aff}}$; $\widehat R^2_\Delta$ & Held-out
        affine-fit fidelity and fidelity to the centered intervention effect. \\
        \bottomrule
    \end{tabular}
\end{table}

\begingroup
\small
In the production experiments of Section~\ref{sec:part-i-production-models},
$M$ denotes the released model, $N$ sample size, and $\rho$ the null-feature
ratio.
$S_j^{\mathrm{perm}}$ and $S_j^{\mathrm{PDP}}$ are sensitivity scores, distinct
from support $S$. $\mathcal N,\mathbb E,I_d$ denote the Gaussian distribution,
expectation, and identity matrix. Held-out $R^2$ values may be negative.
In Section~\ref{sec:cross-scale-gating}, $\ell$ indexes the attenuated layer,
$g$ the source token, and $r$ an affected coefficient;
$m_{i,p\leftarrow g}^{(\ell)}$ is the projected message to receiver $p$ at row $i$,
and $\mathcal R$ specifies the intervened context rows, query rows, or both.
$\widehat J_{rg}^{(\ell)}$ is the coefficient response per unit attenuation.
$w_g^{(\ell)}=\sum_{r\in I(g)}|\widehat J_{rg}^{(\ell)}|/u_g\geq0$
is coefficient sensitivity. Table subscripts
$\mathrm{act}$ and $\mathrm{inact}$ denote averages over active and inactive
sources within a context, respectively.
$P^{(\ell)}$ is the equal-source mean of own-feature coefficient-response fractions; $P_R^{(\ell)}$ restricts that mean to source role $R$. Both are distinct from the prior $P_k$.
Following the main text, $E(f;D)$ and $\Delta E_{\ell g}$ denote held-out
Bayes approximation error and its change under attenuation without hats;
$\Delta\widehat C_{\ell g}$ is the corresponding change in fitted coefficient
error. These intervention errors use original response units and do not
use the task-energy normalization of the architecture comparison.
\endgroup

\end{document}